\documentclass[10pt, letterpaper, oneside, reqno]{amsart}
\usepackage[colorlinks=true, linkcolor=red,citecolor=blue]{hyperref}
\usepackage{multirow}
\usepackage{tcolorbox}
\usepackage{wrapfig}
\usepackage{amssymb, amsthm, amsmath, amsfonts}
\usepackage{mathtools, longfbox, dsfont, mathrsfs, color, bm, enumitem, url, upgreek}
\usepackage{graphicx, transparent, wrapfig}
\usepackage{caption}
\usepackage{subcaption}
\usepackage{multicol}
\usepackage{algorithm, algorithmic}

\usepackage{tcolorbox}
\usepackage{bbm}

\newcommand{\Let}{\triangleq}
\newcommand{\Pnom}{\hat \PP} 

\newcommand{\be}{\begin{equation}}
\newcommand{\ee}{\end{equation}}
\newcommand{\Max}{\max\limits_}

\newcommand{\Sup}{\sup\limits_}
\newcommand{\Inf}{\inf\limits_}
\newcommand{\mbb}{\mathbb}
\newcommand{\Wass}{\mathds W}

\newcommand{\opt}{^\star}

\newcommand{\ds}{\displaystyle}

\newcommand{\R}{\mathbb{R}}
\newcommand{\wh}{\hat}

\newcommand{\mc}{\mathcal}

\newcommand{\EE}{\mathbb{E}}
\newcommand{\QQ}{\mathbb{Q}}
\newcommand{\PP}{\mathbb{P}}

\newcommand{\st}{\mathrm{s.t.}}

\newcommand{\eps}{\varepsilon}
\theoremstyle{definition}
\newtheorem{theorem}{Theorem}[section]
\newtheorem{definition}[theorem]{Definition}

\newtheorem{proposition}[theorem]{Proposition}
\newtheorem{corollary}[theorem]{Corollary}
\newtheorem{lemma}[theorem]{Lemma}
\newtheorem{remark}[theorem]{Remark}

\graphicspath{{figures/}}
\usepackage[utf8]{inputenc}    % allow utf-8 input
\usepackage[T1]{fontenc}       % use 8-bit T1 fonts
\usepackage{tabularx}
\usepackage{lmodern}           % allow pdf compability
\usepackage{xcolor}
\usepackage[margin=1in]{geometry}
\graphicspath{{graph/}}
\usepackage{setspace}

\begin{document}

%\title{Quantifying Fairness via Adversarial Distributions}
\title[Wasserstein Robust Classification with Fairness Constraints]{Wasserstein Robust Classification with Fairness Constraints}
\date{\today}
\author{Yijie Wang, Viet Anh Nguyen, Grani A. Hanasusanto}
\thanks{The authors are with the Graduate Program in Operations Research and Industrial Engineering, University of Texas at Austin (\texttt{yijie-wang, grani.hanasusanto@utexas.edu}) and the Department of Management Science and Engineering, Stanford University (\texttt{viet-anh.nguyen@stanford.edu})}
	\maketitle
	\begin{abstract}
	
We propose a distributionally robust classification model with a fairness constraint that encourages the classifier to be fair in view of the equality of opportunity criterion. We use a type-$\infty$ Wasserstein ambiguity set centered at the empirical distribution to model distributional uncertainty and derive a conservative reformulation for the worst-case equal opportunity unfairness measure. We establish that the model is equivalent to a mixed binary optimization problem, which can be solved by standard off-the-shelf solvers. To improve scalability, we further propose a convex, hinge-loss-based model for large problem instances whose reformulation does not incur any binary variables. Moreover, we also consider the distributionally robust learning problem with a generic ground transportation cost to hedge against the uncertainties in the label and sensitive attribute. Finally, we numerically demonstrate that our proposed approaches improve fairness with negligible loss of predictive accuracy.
	\end{abstract}

    \section{Introduction}
        Machine learning algorithms are increasingly deployed to support consequential decision-making processes, from deciding which applicants will receive the job offers \cite{bigdatahiring,AmazonAI}, loans \cite{bose2001business,shaw1988using}, to university enrollments~\cite{chang2006applying,kabakchieva2013predicting}, or medical interventions \cite{shipp2002diffuse,obermeyer2016predicting}. Even though machine learning algorithms can extract signals from large datasets, they may not be entirely objective and can be susceptible to amplify human biases. For example, it was found that the hiring recommendation system of Amazon AI discriminated against female candidates for technical positions~\cite{AmazonAI}. Similarly, Google’s ad-targeting algorithm had recommended higher-paying executive jobs more often to male than to female candidates \cite{ref:datta2015automated}. It has also been shown that an algorithm used by the US justice system to predict future criminals is significantly biased against African Americans---where it falsely flags black defendants as future criminals at almost twice the rate of white defendants~\cite{machinebias}.
        
        The amplification of human bias caused by algorithms has sparked the emerging field of algorithmic fairness. Strategies to promote fairness in machine learning can be divided into three main categories. The first category includes proposals to \textit{pre-process} the training data before solving a plain-vanilla machine learning problem \cite{ref:calmon2017optimized, ref:del2018obtaining, ref:feldman2015certifying, ref:kamiran2012data, ref:luong2011k, ref:samadi2018price, ref:zemel2013learning}. The second category includes \textit{post-processing} approaches applied to a pre-trained classifier in order to increase its fairness properties while retaining to the largest extent as possible the predictive power of the learned algorithms \cite{ref:corbett2017algorithmic, ref:dwork2018decoupled, ref:hardt2016equality, ref:menon2018cost}. The third category of strategies aims to enforce fairness in the training process by modeling explicitly fairness constraints to the learning problem \cite{ref:donini2018empirical, lawlessfair, ref:menon2018cost,  ref:woodworth2017learning, ref:ye2020fairness,ref:zafar2017fairness, ref:zafar2015fairness}, by penalizing discrimination using fairness-driven regularization terms~\cite{ref:baharlouei2019r, ref:huang2019stable, ref:kamishima2012fairness, ref:kamishima2011fairness} or by (approximately) penalizing any mismatches between the true positive rates and the false negative rates across different groups \cite{ref:bechavod2017penalizing}. Adversarial training to promote algorithmic fairness has also been used and shown to deliver promising results~\cite{ref:edwards2015censoring, ref:garg2019counterfactual, ref:hashimoto2018fairness, ref:kannan2018adversarial, ref:madras2018learning,  ref:rezaei2020fairness, ref:yurochkin2020training, ref:zhang2018mitigating}.

        The method we propose in this paper can be viewed as an adversarial approach pertaining to the third category. 
        More specifically, we consider the training problem of a general linear classifier, which is arguably one of the most popular classification methods in the statistical learning literature \cite{ref:hastie09elements}. The classifier aims to establish a deterministic relationship between a feature vector $X\in \mc X = \R^d$ and a binary response, or label, variable $Y \in \mc Y=\{-1, 1\}$. Without any loss of generality, we associate the positive response with the “advantaged” outcome, such as “being hired” or “receiving a loan approval.” We also assume that there is a single sensitive attribute $A\in\mc A=\{0,1\}$. In a real-world setting, this sensitive attribute can represent information such as the race, gender, or age of a person, and it distinguishes the privileged from the unprivileged individuals. Throughout this paper, we assume that we possess a training data set containing $N$ samples of the form $\{(\hat x_i, \hat a_i, \hat y_i)\}_{i=1}^N$, and these samples are generated independently from a single data-generating probability distribution. In the setting, a classifier $\mc C: \mc X \to \mc Y$ is parameterized by a slope parameter $w \in \R^d$ and an offset $b \in \R$, and the classification output is determined through an indicator function of the form
        \[
        \mc C(x) = \begin{cases}
            1 & \text{if } w^\top x + b \ge 0, \\
            -1 & \text{if } w^\top x + b < 0.
        \end{cases}
        \]
       Throughout, we consider the privileged learning setting in which the sensitive information is only available at the training stage but not at the testing stage~\cite{ref:quadrianto2017recycling, ref:vapnik2009new}. It is therefore reasonable to consider only classifiers $\mc C$ that do not take the sensitive attribute as input. In the context of linear binary classification problems, we need to find a classifier that maximizes the correct classification probability. To this end, we can consider the \textit{correct classification probability} with respect to the distribution $\QQ$ as
\[
\QQ \left(Y(w^\top X + b) > 0\right).
\]
Complementarily, the \textit{misclassification probability} with respect to $\QQ$ is defined as
\[
    \QQ \left(Y(w^\top X + b) \le 0\right).
\]
Notice that by definition, we consider that any $x$ falling exactly on the hyperplane $w^\top X + b = 0$ as misclassified irrespective of the true label of $x$. The linear classifier can be trained by solving the \textit{misclassification probability minimization problem}
\begin{equation}\label{eq:probminimize}
    \min_{w \in \R^d,b \in \R}~\QQ\left( Y(w^\top X + b)\leq 0\right). 
\end{equation}

    To make the linear classifier fair, we can incorporate a measure of fairness into problem~\eqref{eq:probminimize}, either in the form of a constraint or in the form of an objective regularization. There are a plethora of fairness measure that we can utilize to promote fairness in this case, including the Equal Opportunity \cite{hardt2016equality}, Demographic Parity~\cite{calders2009building}, and Equalized Odds \cite{hardt2016equality,zafar2017fairness} among many others. We refer the reader to the references \cite{ref:berk2018fairness, ref:chouldechova2020snapshot, ref:corbett2017algorithmic, ref:mehrabi2019survey} for comprehensive treatments of fairness in machine learning in general and in the classification problem in particular. In this paper, we focus on the Equal Opportunity (EO) criterion, which requires the true positive rate of the classifier is the same across the sensitive groups. The EO unfairness measure is challenging to formulate due to its non-convexity. Moreover, one can verify that the EO unfairness constraint leads to an open feasible set, which prohibits exact mixed binary programming reformulations. To alleviate intractability, simple convex functions such as linear functions have been used to approximate the unfairness measure \cite{ref:agarwal2018fairness,ref:donini2018empirical}. Recently, the paper \cite{ref:ye2020fairness} proposes a mixed binary model that incorporates non-convex approximations of the fairness measures as a regularization term to enhance fairness. In our paper, we consider two different approximation schemes. We first propose a conservative approximation of the EO unfairness measure. The approximation admits a mixed binary reformulation, and it allows users to bound the EO unfairness measure explicitly. As the mixed binary model might not be efficiently solvable for large instances, we also develop a convex approximation based on the hinge loss function.
    The existing notions of fairness proposed in the literature necessitate precise knowledge about the joint probability distribution that governs $(X, A, Y)$. In practice, this distribution is rarely available to the decision makers and is typically estimated using the empirical distribution generated from the imbalanced---and possibly biased---historical observations. While the empirical-based methods may work well on the observed data set, they often fail to yield complete fairness in practice because they do not generalize to out-of-sample data that have not been observed. For example, since there are fewer females in the technical positions at Amazon, relying on the empirical distribution can give rise to severe overfitting that yields an unfair hiring decision. On the other hand, even if the true underlying distribution is available, computing the fairness of the decision is generically intractable (\#P-hard~\cite{dyer:88}) because it involves evaluating a multi-dimensional integration (e.g.,  computing the probability of getting hired conditionally on being an unprivileged person).
    
    Fundamentally, promoting fairness in machine learning algorithms needs to balance among conflicting objectives including predictive accuracy, fairness, computational efficiency, while at the same time having to deal with mismatches between the training and the test data. In this paper, we endeavor to explore the trade-offs between these objectives using the ideas of \emph{distributionally robust optimization (DRO)}. The DRO approach does not impose a single distribution of the features, the attributes and the response label of the entities in the population. Instead, it constructs a set of plausible probability distributions that are locally consistent with the available data set. The DRO approach then optimizes for a safe classifier that performs best in view of the most adverse distribution from within the prescribed distribution set. This approach thus may yield a fair classifier that has provable guarantees on the out-of-sample data.

    Our paper belongs to an emerging class of fairness-aware distributionally robust algorithms. Previously, a repeated loss minimization model with a $\chi^2$-divergence ambiguity set is considered in~\cite{ref:hashimoto2018fairness}. Alternatively,~\cite{ref:rezaei2020fairness} embeds the fairness constraint in the ambiguity set and proposes a robust classification model. However, the paper is under a different setting where information of the sensitive attribute is available in the testing stage. When only the labels are noisy, robust fairness constraints based on a total variation ambiguity set is described in~\cite{ref:wang2020robust}. Wasserstein distributionally robust classification is also proposed to promote individual fairness~\cite{ref:yurochkin2020training}, or to train a log-probabilistic fair logistic classifier~\cite{ref:taskesen2020distributionally}. Our paper is also closely related to the literature on Wasserstein min-max statistical learning, which connects to various forms of regularization (e.g., norm~\cite{ ref:blanchet2019robust,ref:shafieezadeh2019regularization}; shrinkage~\cite{ref:nguyen2018distributionally}). Our formulation considers adversarial perturbations based on the Wasserstein distance~\cite{ ref:blanchet2019quantifying, ref:gao2016distributionally, ref:nam2020adversarial,ref:kuhn2019wasserstein,ref:esfahani2018data}. In particular, the type-$\infty$ Wasserstein distance \cite{ref:givens1984class} is recently applied in distributionally robust formulations \cite{ref:bertsimas2018data-driven, ref:bertsimas2019computational,  ref:nguyen2020distributionally-1,xie2020tractable}. 
    
    In this paper, we consider the worst-case unfairness measure and the worst-case misclassification probability under the most unfavorable distribution from within the type-$\infty$ Wasserstein ambiguity set constructed around the empirical distribution. If the radius of the ambiguity set vanishes to zero, our formulation recovers the unfairness measure evaluated at the empirical distribution. As such, our proposed conservative estimate can be leveraged as a regularization of the empirical-based method.

    \hspace{-4mm}\textbf{Contributions.} The contributions of this paper can be summarized as follows.
    \begin{itemize}[leftmargin = 5mm]
        \item \textbf{Conservative approximation reformulation:} We  robustify a recent unfairness measure motivated by the EO criterion and incorporate this unfairness measure into the distributionally robust misclassification probability minimization problem as a constraint. As the original model does not admit an exact reformulation, we propose a conservative approximation that can be reformulated as a mixed binary optimization program. Compared with existing approximations, ours is the first to guarantee an upper bound on the in-sample EO unfairness measure. Additionally, we illustrate how to generalize the conservative approximation model to handle ambiguity in the marginal distribution and obtain finite-sample guarantees.
        \item \textbf{Hinge-loss-based fairness-aware model:}  To enhance scalability, we propose a \textit{convex} distributionally robust fairness-aware classification model: this model uses the convex hinge loss function to approximate the unfairness measure and the objective function. Experimental results demonstrate that this classifier generates a marked improvement in terms of fairness, with a negligible loss of predictive accuracy. Interestingly, we find that minimizing the expected hinge loss, also known as the Support Vector Machine (SVM), is exactly the Conditional Value at Risk (CVaR) approximation of the misclassification probability minimization problem. %To the best of our knowledge, we are the first to reveal this relationship between these two problems.
        % \item \textbf{Relationship between the Support Vector Machines (SVM) and misclassification probability:} We prove that for any fixed data-generating distribution, SVM is exactly the Conditional Value-at-Risk (CVaR) approximation of the misclassification probability minimization problem. To the best of our knowledge, we are the first to reveal this salient relationship between these two problems.
        \item \textbf{Training with the label and sensitive attribute uncertainties using type $\infty$-Wasserstein ambiguity sets:} We also consider the case where there are uncertainties in the label and sensitive attribute. To reduce the conservativeness, we develop a type-$\infty$ Wasserstein ambiguity set with a side constraint that restricts the proportion of training samples whose sensitive attributes and labels can be `flipped.' We then derive a mixed binary conic program reformulation and a linear program reformulation for training the conservative approximation model and the hinge-loss-based model with this ambiguity set, respectively.
        % \viet{useless --- This ambiguity set can be adopted as a systematic way for training type-$\infty$ Wasserstein ambiguity set with general ground metrics for other classification problems.}
        % \item \textbf{Marginal constraints and finite-sample guarantees:}  For readers interested in the out-of-sample performance, we present a detailed demonstration of how to handle marginal uncertainties and obtain finite-sample guarantees. We show that by replacing the empirical marginal constraints with uncertainty sets, the conservative approximation model can offer attractive finite-sample guarantees. However, we further establish that this generalized model is equivalent to a mixed binary second-order cone program, which might be cumbersome in practice. On the other hand, this result suggests that the original conservative approximation is an excellent model for its simplicity.
    \end{itemize}

    The paper is organized as follows. Section~\ref{sec:prob} describes the distributionally robust fairness-aware classification problem. Section~\ref{sec:refor-prob} proposes a conservative approximation to the original problem, and provides a binary optimization reformulation for training the model. Section~\ref{sec:cvx} further proposes a convex fairness-aware model for large instances, and a convex optimization reformulation is derived for training. Section~\ref{sec:notrust} discusses the situation where there are uncertainties in the sensitive attribute and label. Finally, Section~\ref{sect:numerical} reports on the numerical experiments.

    \textbf{Notations.} For any set $\mc S$, we use $\mc M(\mc S)$ to denote the set of probability measures supported on $\mc S$ and  $| \mc S|$ to denote its cardinality. For any logical expression $\mc E$, the indicator function $\mbb I(\mc E)$ admits value 1 if $\mc E$ is true, and value 0 if $\mc E$ is false. For any norm $\|\cdot \|$ on $\R^d$, we use $\|\cdot\|_*$ to denote the dual norm. We use $\R_{+}$ to denote the set of nonnegative real numbers, and $\R_{++}$ to denote the set of strictly positive real numbers. 
    %%%%%%%%%%%%%%%%%%%%%%%%%%
    
    \section{Distributionally Robust Fairness-aware Linear Classifiers} \label{sec:prob}
    Throughout this section, we focus on promoting fairness of a linear classifier with respect to the criterion of equal opportunity, or also known as equality of opportunity~\cite{ref:hardt2016equality}. This criterion is formally defined as follows.
    
    \begin{definition}[Equal opportunity] \label{def:EO}
    A classifier $\mc C: \mc X \to \mc Y$ satisfies the equal opportunity criterion relative to $\QQ$ if
    \[
        \QQ( \mc C(X) = 1 | A = 1, Y = 1) = \QQ(\mc C(X) = 1 | A = 0, Y = 1) .
    \]
\end{definition}
The definition indicates that the true positive rate is the same across the privileged and unprivileged groups. Based on this definition, we then further define the \emph{equal opportunity unfairness measure} by
\begin{equation}\label{eq:unfmeasure}
    \mathds U(w, b, \QQ) \Let \left|\QQ( \mc C(X) = 1 | A = 1, Y = 1) - \QQ(\mc C(X) = 1 | A = 0, Y = 1)\right|,
\end{equation}
which measures how biased the classification result is across the privileged and unprivileged groups.

We say that a classifier is \textit{trivial} if it is parametrized by $(w, b) = (0, 0) \in \R^{d+1}$. In this case, $\mc C(x) = 1$ for any input $x \in \mc X$. It is easy to verify that the trivial classifier is also fair with respect to any possible distribution $\QQ$. Our goal in this paper is to search for a \textit{non-trivial} classifier that strikes a balance between promoting fairness and achieving superior predictive power.
To this end, suppose that $\PP\opt \in \mc M(\mc X \times \mc A \times \mc Y)$ is the data-generating distribution of the joint random vector $(X, A, Y)$. The fair linear classifier solves the constrained misclassification probability minimization problem

\be \label{eq:exactmodel}
    \begin{array}{cl}
    \min & \PP\opt ( Y(w^\top X + b) \leq 0 ) \\
    \st & w \in \R^d,~b \in \R, \\
    & \mathds U(w, b, \PP\opt) \leq \eta.
    \end{array}
\ee

The objective function of~\eqref{eq:exactmodel} minimizes the misclassification probability, while the constraint of~\eqref{eq:exactmodel} imposes an upper bound $\eta$ on the unfairness measure with respect to $\PP\opt$. 
Unfortunately, the data-generating distribution $\PP\opt$ is elusive to the decision maker. Even if $\PP\opt$ is known, the probabilistic program~\eqref{eq:exactmodel} is computationally intractable.\footnote{Formally, the problem of computing the probability of an event involving multiple random variables belongs to the complexity class \#P-hard \cite{dyer:88}---which is perceived to be `harder' than the class NP-hard.} In a data-driven setting, we assume that we have access to $N$ training samples generated from $\PP\opt$. Let $\Pnom$ be the empirical distribution supported on $\{(\wh x_i, \wh a_i, \wh y_i)\}_{i=1}^N$, we will construct an ambiguity set around $\Pnom$ using the Wasserstein distance.

    	\begin{definition}[Wasserstein distance]
		Let $c$ be a metric on $\Xi$. The type-$p$ $(1 \leq p < +\infty)$ Wasserstein distance between $\QQ_1$ and $\QQ_2$ is defined as
		\[
		\Wass_{p}(\QQ_1, \QQ_2) \Let \inf \left\{ \big(\EE_\pi [c(\xi_1, \xi_2)^p] \big)^{\frac{1}{p}}:
		\pi \in \Pi(\QQ_1, \QQ_2)
		\right\},
		\]
		where $\Pi(\Xi \times \Xi)$ is the set of all probability measures on $\Xi \times \Xi$ with marginals $\QQ_1$ and $\QQ_2$, respectively. 
		The type-$\infty$ Wasserstein distance is defined as the limit of $\Wass_p$ as $p$ tends to $\infty$ and amounts to
		\[
		\Wass_{\infty}(\QQ_1, \QQ_2) \Let \inf \left\{ \mathrm{ess} \Sup{\pi} \big\{ c(\xi_1, \xi_2) : (\xi_1, \xi_2)  \in \Xi \times \Xi \big\} :
		\pi \in \Pi(\QQ_1, \QQ_2)
		\right\}.
		\]
	\end{definition}
		We let $\Xi = \mc X \times \mc A \times \mc Y$ be the joint outcome space of the covariate, the sensitive attribute and the label. The ground metric on $\Xi$ is supposed to be separable, meaning that $c$ can be written as a sum of three components as
	\[
    c\big( (x', a', y'),  (x, a, y) \big) = \| x - x'\| + \kappa_{\mc A} | a - a'| + \kappa_{\mc Y} | y - y'|
    \]
	for some parameters $\kappa_{\mc A} \in [0, +\infty]$ and $\kappa_{\mc Y} \in [0, +\infty]$.
	Moreover, let $\wh p_{ay} = \Pnom(A = a, Y = y)$ denote the empirical marginals constructed from the training samples. We will consider the following marginally-constrained ambiguity set
    \be \label{eq:B-def}
        \mbb B(\Pnom) = \left\{
            \QQ \in \mc M(\mc X \times \mc A \times \mc Y) : \begin{array}{l}
                \Wass_\infty(\QQ, \Pnom) \le \rho, \\
                \QQ(A = a, Y = y) = \wh p_{ay} \quad \forall (a, y) \in \mc A \times \mc Y
            \end{array}
        \right\},
    \ee
    which is a neighborhood around the empirical distribution $\Pnom$. Intuitively, $\mbb B(\Pnom)$ contains all the distributions of $(X, A, Y)$ which is of a type-$\infty$ Wasserstein distance less than or equal to $\rho$ from $\Pnom$, and at the same time has the same marginal distribution on $(A, Y)$ as $\Pnom$. The ambiguity set~$\mbb B(\Pnom)$ is thus parametrized by $\rho$ and the marginals $\wh p$; however, the dependence on these parameters is made implicit. Adding a marginal constraint to the ambiguity set is an expedient practice to achieve tractable reformulation, especially when dealing with conditional expectation constraints that are prevalent in fairness~\cite{ref:taskesen2020distributionally}. Indeed, conditional expectation is typically a non-linear function of the probability measure. However, when confining inside the set $\mbb B(\Pnom)$, we have
    \[
        \QQ( \mc C(X) = 1 | A = a, Y = y) = \wh p_{ay}^{-1} \EE_{\QQ}[\mathbbm{1}_{\{x: \mc C(x) = 1\}}(X) \mathbbm{1}_{(a, y)}(A, Y)] \qquad \forall (a, y) \in \mc A \times \mc Y,
    \]
    which are linear functions of $\QQ$ and conveniently simplifies the problem.

    Equipped with the ambiguity set $\mbb B(\Pnom)$, we can consider the fairness-aware distributionally robust linear classifier
    \be \label{eq:dro-prob}
        \begin{array}{cl}
            \min & \Sup{\QQ \in \mbb B(\Pnom)}~\QQ( Y(w^\top X + b) \leq 0) \\
            \st & w \in \R^d,~b \in \R, \\
            & \Sup{\QQ \in \mbb B(\Pnom)}~\mathds U(w, b, \QQ) \le \eta.
        \end{array}
    \ee
    The constraint of problem~\eqref{eq:dro-prob} depends on a tolerance $\eta \in \R_{+}$: it requires that the difference between the correct positive classification rates in two groups $A = 0$ and $A = 1$ to be smaller than $\eta$, uniformly over all distributions in the ambiguity set. It is easy to verify that the trivial classifier with $(w, b) = (0, 0)$ is feasible for~\eqref{eq:dro-prob} with an objective value of 1. 
    
    Unfortunately, it is challenging to transform problem \eqref{eq:dro-prob} into an exact reformulation for the off-the-shelf solvers. 
    To see this, consider the simplest case where $\rho = 0$, which implies that $\mbb B(\Pnom) = \{ \Pnom\}$, and the constraint of~\eqref{eq:dro-prob} becomes $\mathds U(w, b, \Pnom) \le \eta$. Let us define the index set $\mc I_{a1} = \left\{ i \in [N]: \wh a_i = a, \wh y_i = 1\right\}$ containing indices of the samples with sensitive attribute $a$ and label $1$. Fixing any pair $(a, a') \in \{(0, 1), (1, 0)\}$, we have
    \begin{align*}
        \mathds U(w, b, \Pnom) =&\left|\Pnom( w^\top X + b \ge 0 | A = a, Y = 1) - \Pnom(w^\top X + b \ge 0 | A = a', Y = 1)\right|\\
        =& \left| \EE_{\Pnom}[\wh p_{a1}^{-1} \mathbbm{I}(w^\top X + b \ge 0) \mathbbm{1}_{(a, 1)}(A, Y) - \wh p_{a'1}^{-1} \mbb I(w^\top X + b \ge 0) \mathbbm{1}_{(a', 1)}(A, Y) ] \right| \\
        % =& \frac{1}{N} \left(\wh p_{a1}^{-1} \sum_{i \in \mc I_{a1}} \Sup{x_i: \| x_i - \wh x_i \| \le \rho} \mbb I(w^\top x_i + b > -\eps) - \wh p_{a'1}^{-1} \sum_{i \in \mc I_{a'1}} \Inf{x_i: \| x_i - \wh x_i \| \le \rho} \mbb I(w^\top x_i + b \ge 0) \right) \\
        % =& \frac{1}{N} \left(\wh p_{a1}^{-1} \sum_{i \in \mc I_{a1}} \Sup{x_i: \| x_i - \wh x_i \| \le \rho} \mbb I(w^\top x_i + b > -\eps) - \wh p_{a'1}^{-1} \big(| \mc I_{a'1}| - \sum_{i \in \mc I_{a'1}} \Sup{x_i: \| x_i - \wh x_i \| \le \rho} \mbb I(w^\top x_i + b < 0) \big) \right) \\
        =& \left|\frac{1}{N} \left(\wh p_{a1}^{-1} \sum_{i \in \mc I_{a1}}  \mbb I(w^\top \hat x_i + b \ge 0)  + \wh p_{a'1}^{-1}   \sum_{i \in \mc I_{a'1}}  \mbb I(w^\top \hat x_i + b < 0)  - \wh p_{a'1}^{-1} | \mc I_{a'1}| \right) \right|.
    \end{align*}
    Thus, for any pair $(a, a') \in \{(0, 1), (1, 0)\}$, its corresponding fairness constraint can be written as
    \[\left|\frac{1}{N} \left(\wh p_{a1}^{-1} \sum_{i \in \mc I_{a1}}  \mbb I(w^\top x_i + b \ge 0)  + \wh p_{a'1}^{-1}   \sum_{i \in \mc I_{a'1}}  \mbb I(w^\top x_i + b < 0)  - \wh p_{a'1}^{-1} | \mc I_{a'1}| \right) \right| \leq \eta
    \]
    Consider now the simplest case where $\mc I_{a'1}$ is empty. Then the constraint only involves the first part
    \[\frac{1}{N} \left(\wh p_{a1}^{-1} \sum_{i \in \mc I_{a1}}  \mbb I(w^\top \hat x_i + b \ge 0) \right) \leq \eta.
    \]
    % \viet{use hats for data. $\QQ$ should be $\Pnom$}
    It can be verified that for $\eta \in (0,1)$, the feasible region of $(w,b)$ with such constraints is an open set that cannot be exactly reformulated to a solvable form even with the big-M constraints. For example, if~$\eta<1/(N \wh p_{a1})$, then $(w,b)$ must satisfy $w^\top x_i + b < 0~\forall i \in \mc I_{a1}$ to be feasible to the constraint. In the following sections, we will develop approximations to problem~\eqref{eq:dro-prob} that are amenable to solutions using off-the-shelf solvers.
    
    % As a consequence, we will approximate problem~\eqref{eq:dro-prob} in subsequent sections.
    
    %To address this issue, we propose two approximation schemes for different sizes of problems. In Section~\ref{sec:refor-prob}, we propose a tight conservative approximation by replacing the weak inequality inside the probability measure~$\QQ$ with a strict inequality, and derive a mixed binary conic program reformulation. Then, we delineate some theoretical guarantees of this conservative approximation model. In Section \ref{sec:cvx}, we propose a convex approximation of~\eqref{eq:dro-prob} for solving large-scale problems.

%%%%%%%%%%%%%%%%%%%%%%%%%%%%%%%%%%%%%%%%%%%%%%%%%%%%%%%%%%
\section{$\eps$-Distributionally Robust Fairness-aware Classifier}
\label{sec:refor-prob}
 
    In this section, we propose a conservative approximation of the original problem~\eqref{eq:dro-prob}. Notice that the openness of the feasible set as previously described is because the function $\mbb I(w^\top \wh x_i + b \ge 0)$ is an upper-semicontinuous function in the variable $(w, b)$. In order to generate a closed approximation of the feasible set, it requires to change the inequality sign to a \textit{strict} inequality. Moreover, to guarantee robustness, we also require an inner approximation by modifying the right-hand side to a negative quantity $-\eps$.
    To proceed, given any probability measure $\QQ \in \mc M(\mc X \times \mc A \times \mc Y)$, we can leverage the finite cardinality of $\mc A$ and $\mc Y$ to decompose $\QQ$ using its conditional measures $\QQ_{ay}(X \in \cdot) = \QQ( X \in \cdot | A = a, Y = y)$. 
    We now define the $\eps$-\textit{un}fairness measure $\mathds U_\eps$ as
\begin{align} \label{eq:U-def}
    \mathds U_\eps(w, b, \QQ) \Let \max \left\{ \begin{array}{l}
    \QQ_{01}( w^\top X + b > -\eps) - \QQ_{11}(w^\top X + b \ge 0), \\
    \QQ_{11}( w^\top X + b > -\eps) - \QQ_{01}(w^\top X + b \ge 0) 
    \end{array}
    \right\},
\end{align}
which is parametrized by a strictly positive value $\eps \in \R_{++}$. Similarly, as the objective function of~\eqref{eq:dro-prob} does not admit an exact reformulation, we replace it with $\QQ( Y(w^\top X + b) < \eps)$, which is a conservative approximation of the misclassification probability for any $\eps \in \R_{++}$. The next proposition demonstrates that these approximations are tight in the limit as $\eps$ tends to zero.
\begin{proposition}[Convergence]\label{prop:eps_convergence}
    Fix a measure $\QQ$, the $\eps$-unfairness measure $\mathds U_\eps$ converges to the EO unfairenss measure $\mathds U$ as $\eps \rightarrow 0$, i.e.,
    \[
        \lim_{\eps \rightarrow 0} \mathds U_\eps(w, b, \QQ) = \mathds U(w, b, \QQ).
    \]
    And similarly, we have
    \[ \lim_{\eps \rightarrow 0} \QQ( Y(w^\top X + b) < \eps) = \QQ( Y(w^\top X + b) \ge 0).
    \]
\end{proposition}

\begin{proof}[Proof of Proposition~\ref{prop:eps_convergence}]
See Appendix~\ref{subsec:proof3}.
\end{proof}

Now, consider the following $\eps$-distributionally robust fairness-aware classification ($\eps$-DRFC) problem
\be \label{eq:dro-prob2}
    \begin{array}{cl}
        \min & \Sup{\QQ \in \mbb B(\Pnom)}~\QQ( Y(w^\top X + b) < \eps) \\
        \st & w \in \R^d,~b \in \R, \\
        & \Sup{\QQ \in \mbb B(\Pnom)}~\mathds U_\eps(w, b, \QQ) \le \eta.
    \end{array}
\ee
We remark that when defining the $\eps$-DRFC model~\eqref{eq:dro-prob2}, we can also use two separate parameters of $\eps$: one for the objective function and one for $\mathds U_\eps$. Nevertheless, we opt for a single parameter $\eps$ to simplify the notation, and also to alleviate the burden for parameter tuning. 
The next result shows that problem~\eqref{eq:dro-prob2} is well-defined in the sense that its feasible set contains a non-trivial classifier.

\begin{proposition}[Feasibility] \label{lemma:feasibility}
For any $\eta \in \R_+$, there exists a non-trivial classifier that is feasible for problem~\eqref{eq:dro-prob2}. 
\end{proposition}

\begin{proof}[Proof of Proposition~\ref{lemma:feasibility}]
It suffices to show that problem~\eqref{eq:dro-prob2} is feasible for $\eta = 0$. Let us consider a hyperplane parameterized by $(w_s,b_s) \in \R^{d+1}$ such that $w_s^\top x + b_s > 0 $ for all $x \in \mbb X$, where $\mbb X$ is defined as in Lemma~\ref{lemma:compact}. Because $\mbb X$ is compact and convex, the existence of the hyperplane $(w_s, b_s)$ is a direct result of the separating hyperplane theorem~\cite[\S2.5.1]{ref:boyd2004convex}. In this case, one can verify that for any $\QQ \in \mbb B(\Pnom)$, we have
\begin{align*}
    &\QQ_{11}( w_s^\top X + b_s > -\eps)-\QQ_{01}( w_s^\top X + b_s \geq 0)=1-1=0\\
    &\QQ_{01}( w_s^\top X + b_s > -\eps)-\QQ_{11}( w_s^\top X + b_s \geq 0)=1-1=0,
\end{align*}
where $\QQ_{ay}$ are the conditional distributions of $X $ given $(A= a, Y=y)$. 
This implies that
\[
\Sup{\QQ \in \mbb B(\Pnom)}~\mathds U(w_s, b_s, \QQ) \le 0,
\]
and thus $(w_s, b_s)$ is feasible for problem~\eqref{eq:dro-prob2} at $\eta = 0$. This completes the proof.
\end{proof}

    We now show that the $\eps$-DRFC problem~\eqref{eq:dro-prob2} is a conservative approximation of~\eqref{eq:dro-prob}.
    \begin{proposition}[Conservative approximation]\label{prop:mpm_conservative}
        For any $\eta \in [0,1]$, let~$(w\opt, b\opt)$ be the optimal solution to problem~\eqref{eq:dro-prob2}. Then $(w\opt, b\opt)$ is feasible for problem~\eqref{eq:dro-prob}. Moreover, let~$v\opt$ be the corresponding optimal value of~\eqref{eq:dro-prob2}, then 
        \[
            \QQ(Y(({w\opt})^\top X + b\opt) \leq 0) \le v\opt \qquad \forall \QQ \in \mbb B(\Pnom).
        \]
    \end{proposition}
    
    \begin{proof}[Proof of Proposition~\ref{prop:mpm_conservative}]
    By definition, we find
    \begin{align*}
        \mathds U(w, b, \QQ)=&|\QQ(w^\top X + b \geq 0 | A = 1, Y = 1) - \QQ(w^\top X + b \ge 0 | A = 0, Y = 1)|  \\
         =&\max \left\{ \begin{array}{l}
        \QQ_{01}( w^\top X + b \ge 0) - \QQ_{11}(w^\top X + b \ge 0), \\
        \QQ_{11}( w^\top X + b \ge 0) - \QQ_{01}(w^\top X + b \ge 0) 
        \end{array}
        \right\}\\
        \leq&  \max \left\{ \begin{array}{l}
        \QQ_{01}( w^\top X + b > -\eps) - \QQ_{11}(w^\top X + b \ge 0), \\
        \QQ_{11}( w^\top X + b > -\eps) - \QQ_{01}(w^\top X + b \ge 0) 
        \end{array}
        \right\}\\
         = & \mathds U_\eps(w, b, \QQ),
    \end{align*}
    for every possible value of the classifier parameter $(w, b)$ and any distribution $\QQ$. As a consequence, the feasible region of problem~\eqref{eq:dro-prob2} is an \textit{inner} approximation of the feasible region of problem~\eqref{eq:dro-prob} for any $\eps \in \R_{++}$. This implies that the optimal solution $(w\opt, b\opt)$ of problem~\eqref{eq:dro-prob2} is also feasible for problem~\eqref{eq:dro-prob}.
    
    Furthermore, one can verify that for any $(w, b) \in \R^{d+1}$ and any distribution $\QQ \in \mbb B(\Pnom)$,
\begin{align*}
    \QQ(Y(w^\top X + b) \leq 0) \leq \Sup{\QQ' \in \mbb B(\Pnom)}~\QQ'(Y(w^\top X + b) \leq 0)  \leq \Sup{\QQ' \in \mbb B(\Pnom)}~\QQ'(Y(w^\top X + b) < \eps).
\end{align*}
Thus, by plugging in the optimal solution $(w\opt, b\opt)$ we have 
\begin{align*}
   \QQ(Y((w\opt)^\top X + b\opt) \leq 0) &\leq v\opt \qquad \forall \QQ \in \mbb B(\Pnom),
\end{align*}
which completes the proof.
    \end{proof}
    
% \viet{This is not clear. Does problem~\eqref{eq:dro-prob} have in-sample guarantee? What does in-sample guarantee mean by the way?
% --- Proposition~\ref{prop:mpm_conservative} indicates that the $\eps$-DRFC model \eqref{eq:dro-prob2} is a conservative approximation to the original problem \eqref{eq:dro-prob}, which implies that the model can provide in-sample performance guarantees in terms of accuracy and fairness.} More specifically, the objective value constitutes an upper bound on the misclassification rate, and the EO unfairness measure of the classification result is less than $\eta$.

\noindent The $\eps$-DRFC model \eqref{eq:dro-prob2} enables decision makers to bound the unfairness measure in the training set explicitly using $\eta$. Moreover, as shown in Proposition \ref{prop:mpm_conservative}, the optimal value of problem \eqref{eq:dro-prob2} constitutes an upper bound on the misclassification probability. 
    
\begin{remark}[Out-of-sample guarantee]
We also investigate the out-of-sample performance of the model~\eqref{eq:dro-prob2}. Note that the ambiguity set \eqref{eq:B-def} contains marginal constraints that require probability measures in the ambiguity sets to have the same marginal distribution as the empirical distribution. This constraint invalidates the finite sample guarantees unless the true distribution shares the same marginal distribution with the empirical distribution. In Appendix \ref{sec:marginal}, we illustrate that relaxing the marginal constraints does admit a solvable model with attractive theoretical results; however, the model is more computationally intensive as we add an extra layer of robustness. 
\end{remark}

For any $\eps \in \R_{++}$, problem~\eqref{eq:dro-prob2} admits a mixed binary reformulation, which is our next focus. We first consider the case where we have absolute trust in sensitive attributes and labels, i.e., we use the ground metric
\be
\label{eq:cost}
    c\big( (x', a', y'),  (x, a, y) \big) = \| x - x'\| + \infty | a - a'| + \infty | y - y'|,
\ee
where $\| \cdot \|$ is an arbitrary norm on $\R^d$. Notice that in this setting, we have set $\kappa_{\mc A} = \kappa_{\mc Y} = \infty$, which indicates that we have absolute trust in the value of the sensitive attribute $A$ and the label $Y$. When $c$ is chosen as in~\eqref{eq:cost}, a simple modification of the proof of~\cite[Theorem~3.2]{ref:taskesen2020statistical} shows that any distribution $\QQ$ with $\Wass_\infty(\QQ, \Pnom) < \infty$ should satisfy $\QQ(A = a, Y = y) = \wh p_{ay}$ for all $(a, y) \in \mc A \times \mc Y$. As a consequence, the marginal constraint in the definition of the set $\mbb B(\Pnom)$ becomes redundant and can be omitted. This simplification with absolute trust in the sensitive attribute and label has been previously exploited to derive hypothesis tests for fair classifiers~\cite{ref:taskesen2020statistical} and to train fair logistic classifier~\cite{ref:taskesen2020distributionally}. 

\begin{theorem}[$\eps$-DRFC reformulation] \label{thm:refor-probtrust}
    Suppose that the ground metric is prescribed using~\eqref{eq:cost}, then the $\eps$-DRFC model~\eqref{eq:dro-prob2} is equivalent to the conic mixed binary optimization problem
    \be \label{eq:refor-probtrust}
    \begin{array}{cll}
            \min & \ds\frac{1}{N} \sum_{i=1}^N t_i\\
            \st & w \in \R^d,\; b\in \R,\;t \in \{0, 1\}^N, \; \lambda^{0} \in \{0, 1\}^N, \; \lambda^{1} \in \{0, 1\}^N,\\
            &-\wh y_i (w^\top \wh x_i + b)  + \rho \| w\|_*  \le M t_i - 1 &\forall i \in [N], \\
            &
            \hspace{-2mm}\left.
            \begin{array}{l}
            \ds \frac{1}{| \mc I_{a1}|} \sum_{i \in \mc I_{a1}} \lambda_i^a + \frac{1}{| \mc I_{a'1}|} \sum_{i \in \mc I_{a'1}} \lambda_i^a - 1 \le \eta,  \\
            w^\top \wh x_i + \rho \| w \|_* + b + \eps \le M \lambda_i^a \quad \forall i \in \mc I_{a1}, \\
            -w^\top \wh x_i + \rho \| w \|_* - b  \le M \lambda_i^a \qquad \forall i \in \mc I_{a'1}
            \end{array} 
            \right\}  & \forall (a, a') \in \{(0, 1), (1, 0)\},
        \end{array}
    \ee
    where $M$ is the big-M parameter.
\end{theorem}
\noindent For notational simplicity, we present the reformulation~\eqref{eq:refor-probtrust} with $3N$ binary variables. A closer investigation into problem~\eqref{eq:refor-probtrust} reveals that it suffices to use $N+2 |\mc I_1|$ binary variables, where $\mc I_1 = \{i \in [N]: \wh y_i = 1\}$ is the index set of training samples with positive labels. If $\| \cdot \|$ is either the 1-norm or the $\infty$-norm on~$\R^d$, problem~\eqref{eq:refor-probtrust} is a linear mixed binary optimization problem. If $\| \cdot \|$ is the Euclidean norm, problem~\eqref{eq:refor-probtrust} becomes a mixed binary second-order cone optimization problem. Both types of problems can be solved using off-the-shelf solvers such as MOSEK~\cite{mosek}.

For the remainder of this section, we will provide the proof for Theorem~\ref{thm:refor-probtrust}. This proof relies on the following auxiliary result.
\begin{lemma}[Indicator function reformulation] \label{lemma:indicator}
    Fix any index set $\mc K \subseteq \{1, \ldots, N\}$, a radius $\rho \in \R_+$, a classifier $(w,b) \in \R^{d+1}$ and a collection of samples $\{ \wh x_k\}_{k \in \mc K}$. For any $\eps \in \R$, we have
    \[
        \sum_{k \in \mc K} \Sup{x_k: \| x_k - \wh x_k \| \le \rho} \mathbb I (w^\top x_k + b > \eps) = 
        \left\{
            \begin{array}{cl}
                \min & \ds \sum_{k \in \mc K} \lambda_k \\
                \st & \lambda \in \{0, 1\}^N,\\
                & w^\top \wh x_k + \rho \| w \|_* + b - \eps \le M \lambda_k \quad \forall k \in \mc K,
            \end{array}
        \right.
    \]
    where $M$ is the big-M parameter.
\end{lemma}

\begin{proof}[Proof of Lemma~\ref{lemma:indicator}]
    Using an epigraphical formulation of each supremum term, we find
    \begin{align*}
        \sum_{k \in \mc K} \Sup{x_k: \| x_k - \wh x_k \| \le \rho} \mathbb I (w^\top x_k + b > \eps) &= \left\{
            \begin{array}{cl}
                \min & \ds \sum_{k \in \mc K} \lambda_k \\
                \st & \lambda \in \{0, 1\}^N, \\
                & \Sup{x_k: \| x_k - \wh x_k \| \le \rho} \mathbb I (w^\top x_k + b > \eps) \le \lambda_k \quad \forall k \in \mc K
            \end{array}
        \right. \\
        &= \left\{
            \begin{array}{cl}
                \min & \ds \sum_{k \in \mc K} \lambda_k \\
                \st & \lambda \in \{0, 1\}^N, \\
                & \Max{x_k: \| x_k - \wh x_k \| \le \rho} w^\top x_k + b - \eps \le M\lambda_k \quad \forall k \in \mc K,
          \end{array}
        \right. 
    \end{align*}
    where $M$ is the big-M constant. The dual norm definition implies that
    \[
        \Max{x_k: \| x_k - \wh x_k \| \le \rho} w^\top x_k = w^\top \wh x_k + \rho \| w \|_*,
    \]
    where $\|\cdot\|_*$ is the dual norm of $\| \cdot \|$ on $\R^d$.  This completes the proof.
\end{proof}

We are now ready to prove Theorem~\ref{thm:refor-probtrust}.

\begin{proof}[Proof of Theorem~\ref{thm:refor-probtrust}]
    By exploiting the choice of $c$ with an infinite unit cost on $\mc A$ and $\mc Y$, the ambiguity set $\mbb B(\Pnom)$ can be re-expressed as
    \[
        \mbb B(\Pnom) = \left\{ \QQ \in \mc M(\mc X \times \mc A \times \mc Y):
        \begin{array}{l}
             \exists \pi_i \in \mc M(\mc X) \quad \forall i \in [N], \\
            \QQ(\mathrm{d} x \times \mathrm{d}a \times \mathrm{d}y) = N^{-1} \sum_{i=1}^N \pi_i (\mathrm{d}x) \delta_{(\wh a_i, \wh y_i)}(\mathrm{d}a \times \mathrm{d}y), \\
            \| x_i - \wh x_i \| \le \rho\quad \forall x_i \in \mathrm{supp}(\pi_i)
        \end{array}
        \right\},
    \]
    where $\mathrm{supp}(\pi_i)$ denotes the support of the probability measure $\pi_i$ \cite[Page~441]{ref:aliprantis06hitchhiker}.
    We first provide the reformulation for the objective function of~\eqref{eq:dro-prob2}. For any $(w, b) \in \R^{d+1}$, we have
    \begin{align*}
        \Sup{\QQ \in \mbb B(\Pnom)}~\QQ( Y(w^\top X + b) < \eps) &= \frac{1}{N} \sum_{i=1}^N \Sup{x_i: \| x_i - \wh x_i\| \le \rho}~\mbb I(\wh y_i (w^\top x_i + b) < \eps ) \\
        % &= \frac{1}{N} \sum_{i=1}^N \Sup{x_i: \| x_i - \wh x_i\| \le \rho}~\mbb I(-\wh y_i (w^\top x_i + b) >- \eps ) \\
        &= \left\{
            \begin{array}{cll}
                \min & \ds\frac{1}{N} \sum_{i=1}^N t_i \\
                \st & t \in  \{0, 1\}^N,\\
                & -\wh y_i (w^\top \wh x_i + b)  + \rho \| w\|_*  \le M t_i - \eps &\forall i \in [N],
            \end{array}
        \right.
    \end{align*}
    where the last equality follows from an epigraphical reformulation and from the result of Lemma \ref{lemma:indicator}.
    
    Next, we provide the reformulation for the constraints of~\eqref{eq:dro-prob2}. For any $(w, b) \in \R^{d+1}$, we can rewrite the worst-case unfairness value as
    \begin{align*}
    \Sup{\QQ \in \mbb B(\Pnom)}~\mathds U_\eps(w, b, \QQ) &= \Sup{\QQ \in \mbb B(\Pnom)} \max \left\{ \begin{array}{l}
    \QQ_{01}( w^\top X + b > - \eps) - \QQ_{11}(w^\top X + b \ge 0), \\
    \QQ_{11}( w^\top X + b > -\eps) - \QQ_{01}(w^\top X + b \ge 0) 
    \end{array}
    \right\} \\
    &= \max \left\{ \begin{array}{l}
    \Sup{\QQ \in \mbb B(\Pnom)}~\QQ_{01}( w^\top X + b > - \eps) - \QQ_{11}(w^\top X + b \ge 0), \\
    \Sup{\QQ \in \mbb B(\Pnom)}~\QQ_{11}( w^\top X + b > -\eps) - \QQ_{01}(w^\top X + b \ge 0) 
    \end{array}
    \right\}.
    \end{align*}
    Define the following index sets
    $\mc I_{a1} = \left\{ i \in [N]: \wh a_i = a, \wh y_i = 1\right\}~ \forall a \in \mc A$. Fixing any pair $(a, a') \in \{(0, 1), (1, 0)\}$, we have
    \begin{align*}
        &\Sup{\QQ \in \mbb B(\Pnom)}~\QQ( w^\top X + b > -\eps | A = a, Y = 1) - \QQ(w^\top X + b \ge 0 | A = a', Y = 1) \\
        =& \Sup{\QQ \in \mbb B(\Pnom)}~ \EE_{\QQ}[\wh p_{a1}^{-1} \mathbbm{I}(w^\top X + b > -\eps) \mathbbm{1}_{(a, 1)}(A, Y) - \wh p_{a'1}^{-1} \mbb I(w^\top X + b \ge 0) \mathbbm{1}_{(a', 1)}(A, Y) ] \\
        =& \frac{1}{N} \left(\wh p_{a1}^{-1} \sum_{i \in \mc I_{a1}} \Sup{x_i: \| x_i - \wh x_i \| \le \rho} \mbb I(w^\top x_i + b > -\eps) - \wh p_{a'1}^{-1} \sum_{i \in \mc I_{a'1}} \Inf{x_i: \| x_i - \wh x_i \| \le \rho} \mbb I(w^\top x_i + b \ge 0) \right) \\
        =& \frac{1}{N} \left(\wh p_{a1}^{-1} \sum_{i \in \mc I_{a1}} \Sup{x_i: \| x_i - \wh x_i \| \le \rho} \mbb I(w^\top x_i + b > -\eps) - \wh p_{a'1}^{-1} \big(| \mc I_{a'1}| - \sum_{i \in \mc I_{a'1}} \Sup{x_i: \| x_i - \wh x_i \| \le \rho} \mbb I(w^\top x_i + b < 0) \big) \right) \\
        % =& \frac{1}{N} \left(\wh p_{a1}^{-1} \sum_{i \in \mc I_{a1}} \Sup{x_i: \| x_i - \wh x_i \| \le \rho} \mbb I(w^\top x_i + b > -\eps)  + \wh p_{a'1}^{-1}   \sum_{i \in \mc I_{a'1}} \Sup{x_i: \| x_i - \wh x_i \| \le \rho} \mbb I(w^\top x_i + b < 0)  - \wh p_{a'1}^{-1} | \mc I_{a'1}| \right) \\
        =& \frac{1}{N} \left(\frac{N}{| \mc I_{a1}|} \sum_{i \in \mc I_{a1}} \Sup{x_i: \| x_i - \wh x_i \| \le \rho} \mbb I(w^\top x_i + b > -\eps)  + \frac{N}{| \mc I_{a'1}|}   \sum_{i \in \mc I_{a'1}} \Sup{x_i: \| x_i - \wh x_i \| \le \rho} \mbb I(w^\top x_i + b < 0)  - \frac{N}{| \mc I_{a'1}|} | \mc I_{a'1}| \right) \\
        =& \left\{
            \begin{array}{cll}
                \min & \ds \frac{1}{| \mc I_{a1}|} \sum_{i \in \mc I_{a1}} \lambda_i^a + \frac{1}{| \mc I_{a'1}|} \sum_{i \in \mc I_{a'1}} \lambda_i^a - 1\\
                \st & \lambda^a \in \{0, 1\}^N, \\
                & w^\top \wh x_i + \rho \| w \|_* + b + \eps \le M \lambda_i^a & \forall i \in \mc I_{a1}, \\
                & -w^\top \wh x_i + \rho \| w \|_* - b  \le M \lambda_i^a & \forall i \in \mc I_{a'1},
            \end{array}
        \right.
    \end{align*}
    where the last equality follows by applying Lemma~\ref{lemma:indicator} twice and by noticing that $\mc I_{a1} \cap \mc I_{a'1} = \emptyset$. Setting the optimal value of the above minimization problem to be less than $\eta$ completes the proof.
\end{proof}

The deterministic reformulation \eqref{eq:refor-probtrust} may encounter computational difficulties as the sample size $N$ grows large because it involves $\mc O(N)$ binary variables. Thus, there is merit in studying tractable approximations with better scalability properties. In the next section, we will propose a convex model that is efficiently solvable for large problem instances.

% An avid reader may also be interested in the out-of-sample performance of the model, e.g., whether it is possible to obtain finite sample guarantees. In the next section, we will show that with a minor modification to the ambiguity set, the model will admit some appealing probabilistic guarantees in the out-of-sample test.

\section{Hinge Distributionally Robust Fairness-aware Classifier} \label{sec:cvx}

    % \viet{This sentence is too long --- The conic mixed binary programming reformulation of the fairness-aware distributionally robust misclassification probability minimization problem \eqref{eq:dro-prob2} developed in Section \ref{sec:refor-prob} may encounter computational difficulties in the face of large sample sizes as it involves $3N$ binary variables.} 
    Throughout this section, we propose a convex approximation of \eqref{eq:dro-prob} which requires no binary variables in the reformulation. Towards this end, we replace the probability quantity $\QQ(Y(w^\top X + b) \le 0)=\EE_{\QQ}\left[\mathbbm{I}( Y(w^\top X + b) \le 0)\right]$ in the objective function of \eqref{eq:dro-prob} by the expected hinge loss
    %is also known as the Support Vector Machine (SVM).        SVM is a linear classifier obtained by determining the parameters $(w, b) \in \R^{d+1}$ that solves 
    \begin{equation*}
    \EE_{\QQ}\left[\max \left\{0, 1-Y(w^\top X + b)\right\}\right],
    \end{equation*}
    which is a convex approximation of the misclassification probability. To have a convex approximation of the EO unfairness measure~\eqref{eq:unfmeasure}, we can rewrite this quantity as
    \begin{align*}
    \mathds U(w, b, \QQ) =&\max  \left\{ \begin{array}{l}
    \QQ_{01}( w^\top X + b \geq 0) + \QQ_{11}(w^\top X + b < 0) -1 , \\
    \QQ_{11}( w^\top X + b \geq 0) + \QQ_{01}(w^\top X + b < 0) -1 
    \end{array}
    \right\}\\
    =& \max  \left\{ \begin{array}{l}
    \EE_{\QQ_{01}}\left[\mathbbm{I}( w^\top X + b \geq 0)\right] + \EE_{\QQ_{11}}\left[\mathbbm{I}(w^\top X + b < 0)\right] -1 , \\
    \EE_{\QQ_{11}}\left[\mathbbm{I}( w^\top X + b \geq 0)\right] + \EE_{\QQ_{01}}\left[\mathbbm{I}(w^\top X + b < 0)\right] -1
    \end{array}
    \right\}.
    \end{align*}
    Then, similar to the objective function, we propose a hinge-loss-based unfairness measure to approximate the EO criteria. The \emph{hinge unfairness measure} is defined as
    % \yijie{change notations to $\EE_{\QQ_{01}}$. }
    \begin{equation*}%\label{eq:hinge-loss-unf}
                \mathds H(w, b, \QQ)\Let \max  \left\{ \begin{array}{l}
    \EE_{\QQ_{01}}\left[\max\{0, 1+w^\top X + b \}\right] + \EE_{\QQ_{11}}\left[\max\{0,1-w^\top X - b \}\right] -1 , \\
    \EE_{\QQ_{11}}\left[\max\{0, 1+w^\top X + b \}\right] + \EE_{\QQ_{01}}\left[\max\{0,1-w^\top X - b \}\right] -1
    \end{array}
    \right\}.
    \end{equation*}
    Combining the hinge loss objective and the hinge unfairness measure, we arrive at the following hinge distributionally robust fairness-aware classification (HDRFC) problem:
    \be \label{eq:dro-svm}
        \begin{array}{cl}
            \min & \Sup{\QQ \in \mbb B(\Pnom)}~\EE_{\QQ}[\max\{0, 1 - Y(w^\top X + b)\}] \\
            \st & w \in \R^d,~b \in \R, \\
            & \Sup{\QQ \in \mbb B(\Pnom)}~\mathds H(w, b, \QQ) \le  \zeta.
        \end{array}
    \ee
    The constraint of problem~\eqref{eq:dro-prob2} depends on a tolerance $\zeta \in \R_{+}$: it requires that the hinge unfairness measure to be smaller than $\eta$, uniformly over all distributions in the ambiguity set. It can be easily verified that the upper bound of $\mathds H(w, b, \QQ)$ is $\infty$. The following proposition shows that the lower bound of $\mathds H(w, b, \QQ)$ is 1. Moreover, when it achieves the minimum, the expected distance to the classification hyperplane is uncorrelated with respect to the sensitive attribute conditioned on the label being positive.
    % \viet{connecting sentence}
    \begin{proposition}[Lower bound]\label{prop:U_h_lowerbound}
        For any $(w,b) \in \R^{d+1}$ and any distribution $\QQ$, we have 
        $\mathds H(w, b, \QQ) \ge 1.$ And when $\mathds H(w, b, \QQ) = 1$, we have $\EE_{\QQ_{01}}[w^\top X + b]=\EE_{\QQ_{11}}[w^\top X + b]$.
    \end{proposition}

    \begin{proof}[Proof of Proposition~\ref{prop:U_h_lowerbound}]
    See Appendix~\ref{subsec:proof4}.
    \end{proof}
    
    The hinge unfairness measure $\mathds H$ is a convex approximation of the EO unfairness measure. Compared with the $\eps$-unfairness measure $\mathds U_\eps$ defined in Section \ref{sec:refor-prob}, the hinge unfairness measure does not provide a tight upper bound for the EO unfairness measure. Nevertheless, the hinge unfairness measure is an attractive formulation because it is amenable to a convex reformulation, which is essential for solving large-scale problems.
%     \viet{wrong result. This only holds under some conditions --- The next result asserts that the optimal value of \eqref{eq:dro-svm} constitutes an upper bound on the optimal value of \eqref{eq:dro-prob}.

%     \begin{proposition}[Conservative approximation]\label{prop:cvx_conservative}
%         Let $(w\opt, b\opt)$ be the optimal solution to problem~\eqref{eq:dro-svm} and~$v\opt$ be the corresponding optimal value of~\eqref{eq:dro-svm}, then 
%         \[
%             \QQ(Y((w\opt)^\top X + b\opt) \leq 0) \le v\opt \qquad \forall \QQ \in \mbb B(\Pnom).
%         \]
%     \end{proposition}
% }

    % However, as the data generating distribution is generally unknown to decision makers, the empirical risk minimization is often solved by replacing $\QQ$ with the empirical distribution as a surrogate, which leads to the following tractable convex optimization problem.
    % \be \label{eq:SVM-standard}
    % \Min{w, b}~\frac{1}{N}\sum\limits_{i=1}^N  \ell_{(\wh x_i, \wh y_i)}(w, b),\quad \ell_{(\wh x_i, \wh y_i)}(w, b)= \max\{0, 1 - \wh y_i(w^\top \wh x_i + b) \}.
    % \ee
    
% \viet{write a paragraph here to connect... bs about hingeloss, svm, etc.... tell the reader that we are now interested in reformulation problem (10)}  

\begin{remark}[SVM and CVaR]
To obtain a conservative approximation of the misclassification probability minimization problem~\eqref{eq:probminimize}, one could also employ the popular Conditional Value at Risk (CVaR)~\cite{RU00:cvar}, which is the best-known convex approximation of probabilistic constraints~\cite{nemirovski2007convex}. Interestingly, we find that the Support Vector Machine (SVM) model, which minimizes the expected hinge loss, is exactly the CVaR approximation of problem~\eqref{eq:probminimize}. A detailed discussion of this result is provided in Appendix \ref{subsec:cvar}.
\end{remark}

We now present the main result of this section which asserts that the HDRFC problem~\eqref{eq:dro-svm} can be reformulated as a conic optimization problem.
    
\begin{theorem}[HDRFC reformulation] \label{thm:refor-cvx}
    Suppose that the ground metric is prescribed using~\eqref{eq:cost}, the HDRFC model~\eqref{eq:dro-svm} is equivalent to the conic optimization problem
    \be \label{eq:refor-cvx}
        \begin{array}{cll}
            \min & \ds\frac{1}{N} \sum_{i=1}^N t_i\\
            \st & w \in \R^d,\; b\in \R,\;t \in \R_+^N, \; \lambda^{0} \in \R^N_+, \; \lambda^{1} \in \R^N_+, \\
            &-\wh y_i (w^\top \wh x_i + b)  + \rho \| w\|_*  \le t_i - 1 &\forall i \in [N], \\
            &
            \hspace{-2mm}\left.
            \begin{array}{l}
            \ds \frac{1}{| \mc I_{a1}|}  \sum_{i \in \mc I_{a1}} \lambda_i^a + \frac{1}{| \mc I_{a'1}|} \sum_{i \in \mc I_{a'1}} \lambda_i^a - 1 \le \zeta,  \\
            1+w^\top \wh x_i + \rho \| w \|_* + b  \le  \lambda_i^a \qquad \forall i \in \mc I_{a1}, \\
            1-w^\top \wh x_i + \rho \| w \|_* - b  \le  \lambda_i^a \qquad \forall i \in \mc I_{a'1}
            \end{array} 
            \right\}  & \forall (a, a') \in \{(0, 1), (1, 0)\}.
        \end{array}
    \ee
\end{theorem}
\begin{proof}[Proof of Theorem~\ref{thm:refor-cvx}]

    By exploiting the choice of $c$ with an infinite unit cost on $\mc A$ and $\mc Y$, the ambiguity set $\mbb B(\Pnom)$ can be re-expressed as
    \[
        \mbb B(\Pnom) = \left\{ \QQ \in \mc M(\mc X \times \mc A \times \mc Y):
        \begin{array}{l}
             \exists \pi_i \in \mc M(\mc X) \quad \forall i \in [N], \\
            \QQ(\mathrm{d} x \times \mathrm{d}a \times \mathrm{d}y) = N^{-1} \sum_{i=1}^N \pi_i (\mathrm{d}x) \delta_{(\wh a_i, \wh y_i)}(\mathrm{d}a \times \mathrm{d}y), \\
            \| x_i - \wh x_i \| \le \rho\quad \forall x_i \in \mathrm{supp}(\pi_i)
        \end{array}
        \right\},
    \]
    where $\mathrm{supp}(\pi_i)$ denotes the support of the probability measure $\pi_i$ \cite[Page~441]{ref:aliprantis06hitchhiker}.
    We first provide the reformulation for the objective function of~\eqref{eq:dro-svm}. For any $(w, b) \in \R^{d+1}$, we have
    \begin{align*}
        \Sup{\QQ \in \mbb B(\Pnom)}~\EE_{\QQ}[\max\{0, 1 - Y(w^\top X + b)\}] &= \frac{1}{N} \sum_{i=1}^N \Sup{x_i: \| x_i - \wh x_i\| \le \rho}~\max\{0, 1 - \wh y_i (w^\top x_i + b) \} \\
        &= \frac{1}{N} \sum_{i=1}^N \max\left\{0, 1 - \Inf{x_i: \| x_i - \wh x_i\| \le \rho}~\wh y_i (w^\top x_i + b) \right\} \\
        &= \left\{
            \begin{array}{cll}
                \min & \ds\frac{1}{N} \sum_{i=1}^N t_i \\
                \st & t \in \R_+^N, \\
                & -\wh y_i (w^\top \wh x_i + b)  + \rho \| w\|_*  \le t_i - 1 &\forall i \in [N],
            \end{array}
        \right.
    \end{align*}
    where the last equality follows from an epigraphical reformulation and from the properties of the dual norm. 
    
    Next, we provide the reformulation for the constraints of~\eqref{eq:dro-svm}. For any $(w, b) \in \R^{d+1}$, we can rewrite the worst-case hinge-loss unfairness value as
    \begin{align*}
    \Sup{\QQ \in \mbb B(\Pnom)}~\mathds H(w, b, \QQ) &= \Sup{\QQ \in \mbb B(\Pnom)} \max \left\{ \begin{array}{l}
\EE_{\QQ_{01}}\left[\max\{0, 1+w^\top X + b \}\right] + \EE_{\QQ_{11}}\left[\max\{0,1-w^\top X - b \}\right] -1 , \\
\EE_{\QQ_{11}}\left[\max\{0, 1+w^\top X + b \}\right] + \EE_{\QQ_{01}}\left[\max\{0,1-w^\top X - b \}\right] -1
\end{array}
\right\} \\
    &= \max \left\{ \begin{array}{l}
    \Sup{\QQ \in \mbb B(\Pnom)}~\EE_{\QQ_{01}}\left[\max\{0, 1+w^\top X + b \}\right] + \EE_{\QQ_{11}}\left[\max\{0,1-w^\top X - b \}\right] -1 , \\
    \Sup{\QQ \in \mbb B(\Pnom)}~\EE_{\QQ_{11}}\left[\max\{0, 1+w^\top X + b \}\right] + \EE_{\QQ_{01}}\left[\max\{0,1-w^\top X - b \}\right] -1
    \end{array}
    \right\}.
    \end{align*}
    Fixing any pair $(a, a') \in \{(0, 1), (1, 0)\}$, we have

\begin{align*}
    &\Sup{\QQ \in \mbb B(\Pnom)}~\mathds \EE_{\QQ_{a1}}\left[\max\{0, 1+w^\top X + b \}\right] + \EE_{\QQ_{a'1}}\left[\max\{0,1-w^\top X - b \}\right] -1 \\
    =& \Sup{\QQ \in \mbb B(\Pnom)}~ \EE_{\QQ}\left[\wh p_{a1}^{-1} \max\{0, 1+w^\top X + b \} \mathbbm{1}_{(a, 1)}(A, Y) + \wh p_{a'1}^{-1} \max\{0,1-w^\top X - b \} \mathbbm{1}_{(a', 1)}(A, Y) \right] - 1 \\
    = & \frac{1}{N} \left(\wh p_{a1}^{-1} \sum_{i \in \mc I_{a1}} \Sup{x_i: \| x_i - \wh x_i \| \le \rho} \max\{0,1+ w^\top x_i + b \}  + \wh p_{a'1}^{-1}   \sum_{i \in \mc I_{a'1}} \Sup{x_i: \| x_i - \wh x_i \| \le \rho} \max\{0,1- w^\top x_i - b\}   - \wh p_{a'1}^{-1} | \mc I_{a'1}| \right)\\
    =& \left\{
        \begin{array}{cll}
            \min & \ds \frac{1}{| \mc I_{a1}|}  \sum_{i \in \mc I_{a1}} \lambda_i^a + \frac{1}{| \mc I_{a'1}|} \sum_{i \in \mc I_{a'1}} \lambda_i^a - 1 \\
            \st & \lambda^a \in \R^N_+, \\
            & 1+w^\top \wh x_i + \rho \| w \|_* + b  \le  \lambda_i^a & \forall i \in \mc I_{a1}, \\
            & 1-w^\top \wh x_i + \rho \| w \|_* - b  \le \lambda_i^a & \forall i \in \mc I_{a'1},
        \end{array}
    \right.
\end{align*}
where the last equality follows by applying Lemma~\ref{lemma:indicator} twice and by noticing that $\mc I_{a1} \cap \mc I_{a'1} = \emptyset$. Setting the optimal value of the above minimization problem to be less than $\eta$ completes the proof.
\end{proof}

\noindent If $\| \cdot \|$ is either a 1-norm or an $\infty$-norm on $\R^d$, problem~\eqref{eq:refor-cvx} is a linear optimization problem. If $\| \cdot \|$ is an Euclidean norm, problem~\eqref{eq:refor-cvx} becomes a second-order cone optimization problem. Both types of problems can be solved using off-the-shelf solvers such as MOSEK~\cite{mosek}. 

We now benchmark the $\eps$-DRFC model with the HDRFC model. The reformulation of the $\eps$-DRFC problem \eqref{eq:refor-probtrust} involves $3N$ binary variables and $3N$ big-M constraints, while the reformulation of the HDRFC problem \eqref{eq:refor-cvx} only contains $3N$ continuous variables and $3N$ convex constraints. As solving conic mixed binary programs with big-M constraints is challenging in the face of large sample sizes, HDRFC is more suitable for large instances. In practice, we find the hinge unfairness measure performs quite well, and we will further demonstrate its performance in Section \ref{sect:numerical}.

So far, all the reformulations are derived with $\kappa_{\mc A} = \kappa_{\mc Y} = \infty$, which means we have absolute trust in the sensitive attribute and label. We now extend our study to the case where there is uncertainty in the sensitive attribute and label.

%%%%%%%%%%%%%%%%%%%%%%%%%%
    
    % Consider now the sum of infimum terms. We have
    % \begin{align*}
    %     \sum_{k=1}^K \Inf{x_k: \| x_k - \wh x_k \| \le \rho} \mathbb I (w^\top x_k + b \le 0) &= \left\{
    %         \begin{array}{cl}
    %             \max & \ds \sum_{k=1}^K \lambda_k \\
    %             \st & \lambda \in \{0, 1\}^K \\
    %             & \Inf{x_k: \| x_k - \wh x_k \| \le \rho} \mathbb I (w^\top x_k + b \le 0) \ge \lambda_k \quad \forall k = 1, \ldots, K
    %         \end{array}
    %     \right. \\
    %     &= \left\{
    %         \begin{array}{cl}
    %             \max & \ds \sum_{k=1}^K \lambda_k \\
    %             \st & \lambda \in \{0, 1\}^K \\
    %             & \Inf{x_k: \| x_k - \wh x_k \| \le \rho} w^\top x_k + b \le M (1-\lambda_k)
    %         \end{array}
    %     \right. \\
    %     &= \left\{
    %         \begin{array}{cl}
    %             \max & \ds \sum_{k=1}^K \lambda_k \\
    %             \st & \lambda \in \{0, 1\}^K \\
    %             & w^\top \wh x_k - \rho \| w \|_* + b \le M (1-\lambda_k)
    %         \end{array}
    %     \right. \\
    % \end{align*}

%%%%%%%%%%%%%%%%%%%%%
\section{Training with General Ground Metric}
\label{sec:notrust}
Previous sections have consider the absolute trust case of the ground cost~\eqref{eq:cost} in which $\kappa_{\mc A} = \kappa_{\mc Y} = \infty$. Here, we consider a general ground metric
\be
\label{eq:gencost}
    c\big( (x', a', y'),  (x, a, y) \big) = \| x - x'\| + \kappa_{\mc A} | a - a'| + \kappa_{\mc Y} | y - y'|
\ee
for some finite values of $\kappa_{\mc A}$ and $\kappa_{\mc Y}$. The case for finite $\kappa_{\mc A}$ and $\kappa_{\mc Y}$ is particularly relevant when we have noisy observations of the sensitive attributes and class labels~\cite{ref:shafieezadeh2019regularization}. Without any loss of generality, we will illustrate how to incorporate this general ground metric using the HDRFC model~\eqref{eq:dro-svm}. 
%Notice that this general ground metric can also be applied to the $\eps$-DRFC model \eqref{eq:dro-prob2}, and we provide the corresponding results in Appendix \ref{subsec:mpm_general}.
For the $\eps$-DRFC model \eqref{eq:dro-prob2}, we will provide the corresponding results in Appendix \ref{subsec:mpm_general}. At the same time, we will consider in this section a more general definition of the ambiguity set $\mbb B(\Pnom)$. To this end, we first observe that the ambiguity set $\mbb B(\Pnom)$ can be re-expressed as\footnote{A formal proof can be found in Lemma~\ref{lemma:B-refor}.}
\[
    \mbb B(\Pnom) = \left\{
            \QQ \in \mc M(\mc X \times \mc A \times \mc Y) : \begin{array}{ll}
                \exists \pi_i \in \mc M(\mc X \times \mc A \times \mc Y) \quad \forall i \in [N]:\\
                \QQ = N^{-1} \sum_{i\in[N]}\pi_i, \\
                \Wass_\infty(\pi_i, \delta_{(\wh x_i, \wh a_i, \wh y_i)}) \le \rho & \forall i \in [N], \\
                \QQ(A = a, Y = y) = \wh p_{ay} & \forall (a, y) \in \mc A \times \mc Y
            \end{array}
        \right\}.
\]
Let $\gamma \in [0, 1]$ and consider the ambiguity set $\mc B_\gamma(\Pnom)$ parametrized by $\gamma$ as
    \be \label{eq:modified_ambi}
        \mc B_\gamma(\Pnom) \Let \left\{
            \QQ \in \mc M(\mc X \times \mc A \times \mc Y) : \begin{array}{ll}
                \exists \pi_i \in \mc M(\mc X \times \mc A \times \mc Y) \quad \forall i \in [N]:\\
                \QQ = N^{-1} \sum_{i\in[N]}\pi_i, \\
                \Wass_\infty(\pi_i, \delta_{(\wh x_i, \wh a_i, \wh y_i)}) \le \rho &\forall i \in [N], \\
                \QQ(A = a, Y = y) = \wh p_{ay} & \forall (a, y) \in \mc A \times \mc Y, \\
                \sum_{i\in[N]} \pi_i(A = \wh a_i, Y = \wh y_i) \ge (1-\gamma) N
            \end{array}
        \right\}.
    \ee
    Notice that $\mc B_\gamma(\Pnom)$ differs from $\mbb B(\Pnom)$ solely on the basis of the last constraint defining $\mc B_\gamma(\Pnom)$. Intuitively, the parameter $\gamma$ indicates the maximum proportion of the training sample points that can be flipped in the $(A, Y)$ dimension. When $\gamma = 1$, then the last constraint defining~$\mc B_\gamma(\Pnom)$ collapses into
    \[
        \sum_{i\in[N]} \pi_i(A = \wh a_i, Y = \wh y_i) \ge 0,
    \]
    which holds true trivially. Thus, we can deduce that $\mc B_1(\Pnom) = \mbb B(\Pnom)$. At the other extreme when $\gamma = 0$, then we arrive at the constraint
    \[
        \sum_{i\in[N]} \pi_i(A = \wh a_i, Y = \wh y_i) \ge N \quad \implies \quad \pi_i(A = \wh a_i, Y = \wh y_i) = 1 \qquad \forall i \in [N].
    \]
    The latter constraint resembles the case considered in Section~\ref{sec:prob} and Section~\ref{sec:cvx} with absolute trust in the sensitive attribute and the label. Any value $\gamma \in (0, 1)$ thus can be thought of as an interpolation of the robustness condition between these two above-mentioned extreme cases.
  
    We consider in this section the following modification of problem~\eqref{eq:dro-svm} in which the ambiguity set is replaced by $\mc B_\gamma(\Pnom)$:
\be \label{eq:dro-modified}
    \begin{array}{cl}
        \min & \Sup{\QQ \in \mc B_\gamma(\Pnom)}~\EE_{\QQ}[\max\{0, 1 - Y(w^\top X + b)\}] \\
        \st & w \in \R^d,~b \in \R, \\
        & \Sup{\QQ \in \mc B_\gamma(\Pnom)}~\mathds H(w, b, \QQ) \le \zeta.
    \end{array}
\ee
It is easy to show, by modifying Proposition~\ref{prop:U_h_lowerbound} and leveraging Corollary~\ref{corollary:compact2}, the lower bound of $\zeta$ is still one. We now present the main result of this section, which provides the reformulation for problem~\eqref{eq:dro-modified}.
\begin{theorem}(HDRFC reformulation) \label{thm:refor}
    Suppose that the ground metric is prescribed using~\eqref{eq:gencost}, for any $\gamma \in (0, 1)$, problem~\eqref{eq:dro-modified} is equivalent to the mixed binary conic program
\begin{equation*} %\label{eq:general-refor}
\begin{array}{cll}
\inf&\ds \frac{1}{N} \sum_{i \in [N]} \nu_i + \sum_{(\bar a, \bar y) \in \mc A \times \mc Y} \hat p_{\bar a \bar y} \mu_{\bar a \bar y}-\theta(1-\gamma) & \\
 \st& \nu \in \R^N ,\; \theta \in \R_+,\; \mu \in \R^{2\times 2},  \\
  &\nu^a \in \R^N ,\; \theta^a \in \R_+,\; \mu^a \in \R^{2\times 2} \qquad \forall (a, a') \in \{(0, 1), (1, 0)\}, \\
&\hspace{-2mm}\left.\begin{array}{l}
   \textup{If} \ \kappa_{\mc A} |a - \wh a_i| + \kappa_{\mc Y}| y - \wh y_i| \leq \rho:\\
    \quad 0 \leq \mu_{ay} - \theta \mathbbm{1}_{(\hat a_i, \hat y_i)} (a, y) + \nu_i,  \\
    \quad 1 - yb -y w^\top \wh x_i + (\rho - \kappa_{\mc A}|a - \wh a_i| - \kappa_{\mc Y}|y - \wh y_i|) \| w \|_*\\
    \hspace{4cm} \leq \mu_{ay} - \theta \mathbbm{1}_{(\hat a_i, \hat y_i)} (a, y) + \nu_i 
  \end{array}
  \right\}\quad \forall i \in [N] \quad \forall (a, y) \in \mc A \times \mc Y,\\
 &\hspace{-2mm}\left.
\begin{array}{l}
   \textup{If} \ \kappa_{\mc A} |a - \wh a_i| + \kappa_{\mc Y}| 1 - \wh y_i| \leq \rho:\\
    \quad \wh 0 \leq \wh p_{a1}( \mu_{a,1}^a - \theta^a \mathbbm{1}_{(\hat a_i, \hat y_i)} (a, 1) + \nu_i^a ), \\
    \quad 1 + w^\top \wh x_i + (\rho- \kappa_{\mc A} |a - \wh a_i| - \kappa_{\mc Y}| 1 - \wh y_i|) \|w\|_* + b  \\  \hspace{4cm}  \leq  \wh p_{a1} (\mu_{a,1}^a - \theta^a \mathbbm{1}_{(\hat a_i, \hat y_i)} (a, 1) + \nu_i^a )\\
    \textup{If} \ \kappa_{\mc A} |a' - \wh a_i| + \kappa_{\mc Y}| 1 - \wh y_i| \leq \rho:\\
    \quad \wh 0 \leq  \wh p_{a'1}(\mu_{a'1}^a - \theta^a \mathbbm{1}_{(\hat a_i, \hat y_i)} (a', 1) + \nu_i^a),  \\
    \quad 1 -w^\top \wh x_i + (\rho- \kappa_{\mc A} |a' - \wh a_i| - \kappa_{\mc Y}| 1 - \wh y_i|) \|w\|_* - b  \\ \hspace{4cm} \leq  \wh p_{a'1}( \mu_{a'1}^a - \theta^a \mathbbm{1}_{(\hat a_i, \hat y_i)} (a', 1) + \nu_i^a ) \\
    \textup{If} \ \kappa_{\mc A} |a - \wh a_i| + \kappa_{\mc Y}| -1 - \wh y_i| \leq \rho:\\
    \quad 0 \leq \mu_{a,-1}^a - \theta^a \mathbbm{1}_{(\hat a_i, \hat y_i)} (a, -1) + \nu_i^a\\
    \textup{If} \ \kappa_{\mc A} |a' - \wh a_i| + \kappa_{\mc Y}| -1 - \wh y_i| \leq \rho:\\
    \quad 0 \leq \mu_{a',-1}^a - \theta^a \mathbbm{1}_{(\hat a_i, \hat y_i)} (a', -1) + \nu_i^a\\
\end{array}
\right\} \quad \forall (a, a') \in \{(0, 1), (1, 0)\} \quad \forall i \in [N], \\
& \ds\frac{1}{N} \sum_{i \in [N]} \nu_i^a + \sum_{(\bar a, \bar y) \in \mc A \times \mc Y} \hat p_{ay} \mu_{\bar a \bar y}^a-\theta^a(1-\gamma) -1 \leq \zeta
\quad \forall a \in \mc A.\\
\end{array}
\end{equation*}
\end{theorem}

 In the remainder of this section, we will provide the proof of Theorem~\ref{thm:refor}. This proof leverages the following duality result.
\begin{lemma}(Strong duality)\label{lem:strong_dual}
    Let $\phi:\mc X \times \mc A \times \mc Y \rightarrow \R$ be a Borel measurable loss function. Then for any~$\gamma \in (0, 1)$, the semi-infinite program
    \begin{subequations}
    \begin{equation*}%\label{eq:label_primal}
        \Sup{\QQ \in \mc B_\gamma(\Pnom)}    \EE_{\QQ} [\phi(X,A,Y)]
    \end{equation*}
    % \begin{equation*}
    % \begin{array}{cll}
    %      \Sup{\QQ \in \mc M} &   \EE_{\QQ} [\phi(X,A,Y)]& \\
    %      \st & \Wass_\infty(\QQ, \Pnom) \le \rho \\
    %      & \QQ(A = a, Y = y) = \wh p_{ay} & \forall (a, y) \in \mc A \times \mc Y \\
    %     &\sum_{i\in[N]} \pi_i(A = \wh a_i, Y = \wh y_i) \ge (1-\kappa) N&
    % \end{array}
    % \end{equation*}
    admits the following dual form
    \begin{equation*}%\label{eq:strong_dual}
    \begin{array}{cll}
         \inf&\ds \frac{1}{N} \sum_{i \in [N]} \nu_i + \sum_{a \in \mc A, y \in \mc Y} \hat p_{a,y} \mu_{a,y}-\theta(1-\gamma) & \\
         \st& \nu \in \R^N ,~\theta \in \R_+,~\mu \in \R^{2\times 2},\\
         & \Sup{x: \| x - \wh x_i\| \le \rho - \kappa_{\mc A}|a - \wh a_i| - \kappa_{\mc Y}|y - \wh y_i|}\phi(x, a, y) \leq \mu_{ay} - \theta \mathbbm{1}_{(\hat a_i, \hat y_i)} (a,y) + \nu_i & \forall i \in [N] \quad \forall (a, y) \in \mc A \times \mc Y,
    \end{array}
    \end{equation*}
    \end{subequations}
    where the supremum value is considered to be $-\infty$ if the corresponding feasible set is empty.
\end{lemma}

\begin{proof}[Proof of Lemma~\ref{lem:strong_dual}]
Using the definition of the type-$\infty$ Wasserstein distance, we can re-express the ambiguity set $\mc B_\gamma(\Pnom)$ as
	\begin{align*}
		\mc B_\gamma(\Pnom) 
		=&\left\{ \QQ \in \mc M(\mc X \times \mc A \times \mc Y):
		\begin{array}{l}
		\exists \pi_i \in \mc M(\mc X \times \mc A \times \mc Y) ~\forall i \in [N] \text{ such that } \QQ = \frac{1}{N} \sum_{i\in[N]} \pi_i,\\
		N^{-1} \sum_{i=1}^N \pi_i (A = a, Y = y) = \wh p_{ay} \quad \forall (a, y) \in \mc A \times \mc Y, \\
		\| x - \wh x_i\| + \kappa_{\mc A} |a - \wh a_i| + \kappa_{\mc Y}| y - \wh y_i| \leq \rho \quad \forall (x, a, y) \in \mathrm{supp}(\pi_i) \quad \forall i \in [N], \\
		\sum_{i\in[N]} \pi_i(A = \wh a_i, Y = \wh y_i) \ge (1-\gamma) N\\
		\end{array}
		\right\}.
    \end{align*}
The worst-case expected loss can now be written as
\begin{align*}
    \Sup{\QQ \in \mc B_\gamma(\Pnom)}~\EE_{\QQ}[\phi(X,A,Y)] = \left\{
        \begin{array}{cll}
            \sup &\ds \frac{1}{N} \sum_{i \in [N]}~\EE_{\pi_i}[\phi(X,A,Y)] \\
            \st & \pi_i \in \mc M(\mc X \times \mc A \times \mc Y) & \forall i \in [N], \\
            & \ds \sum_{i\in[N]} \pi_i(A = a, Y = y) = N \wh p_{a y} & \forall (a, y) \in \mc A \times \mc Y, \\
            & \ds \sum_{i\in[N]} \pi_i(A = \wh a_i, Y = \wh y_i) \ge (1-\gamma) N,\\
            & \| x - \wh x_i\| + \kappa_{\mc A} | a- \wh a_i| + \kappa_{\mc Y}| y - \wh y_i| \leq \rho & \forall (x, a, y) \in \mathrm{supp}(\pi_i) \quad \forall i \in [N].
        \end{array}
    \right.
    \end{align*} 
    Any $\pi_i \in \mc M(\mc X \times \mc A \times \mc Y)$ can be decomposed as
    \[
        \pi_i(\mathrm{d} x \times \mathrm{d} a' \times \mathrm{d} y') = \sum_{(a, y) \in \mc A \times \mc Y} \tau_{iay} \pi_{iay}(\mathrm{d}x)\delta_{(a, y)}(\mathrm{d} a' \times \mathrm{d} y'),
    \]
    where $\pi_{iay}$ is the conditional distribution of $X$ given that $(A, Y) = (a, y)$ and the nonnegative weights $\tau \in \R_+^{N \times | \mc A| \times |\mc Y|}$ satisfy
    \[
      \sum_{(a, y) \in \mc A \times \mc Y} \tau_{iay} = 1 \qquad \forall i \in [N].
    \]
    Moreover, define the following optimal values
    \[
        v_{iay} = \sup \{ \phi(x, a, y) : x \in \mc X,~ \| x - \wh x_i\| \le \rho - \kappa_{\mc A}|a - \wh a_i| - \kappa_{\mc Y}|y - \wh y_i|\}
    \]
    % \viet{the above equation is not cited anywhere}
    for each $i \in [N]$ and $(a, y) \in \mc A \times \mc Y$. Denote momentarily the feasible set of the above optimization problem as $\mc X_{iay}$. Notice that $\mc X_{iay} =\emptyset$ if $\rho - \kappa_{\mc A}|a - \wh a_i| - \kappa_{\mc Y}|y - \wh y_i| < 0$ and in this case we set $v_{iay} = -\infty$. By definition, we also have
    \[
        v_{iay} = \Sup{\pi_{iay} \in \mc M(\mc X)} \int_{\mc X_{iay}} \phi(x, a, y) \pi_{iay}(\mathrm d x)
    \]
    whenever $\mc X_{iay}$ is non-empty. 
    Using this definition of $v$ and by the above decomposition of $\pi$, we obtain
    \begin{align*}
    \Sup{\QQ \in \mc B_\gamma(\Pnom)}~\EE_{\QQ}[\phi(X,A,Y)] =  \left\{
        \begin{array}{cll}
            \sup & \ds \frac{1}{N} \sum_{i \in [N]}~\sum_{(a,y) \in \mc A \times \mc Y} \tau_{iay} v_{iay} \\
            \st & \tau_{iay}  \in \R_+  &\forall i \in [N] \quad \forall (a, y) \in \mc A \times \mc Y, \\
            &\ds \sum_{(a, y) \in \mc A \times \mc Y} \tau_{iay} = 1 & \forall i \in [N], \\
            &\ds \sum_{i\in[N]} \tau_{iay} = N \wh p_{a y} & \forall (a, y) \in \mc A \times \mc Y, \\
            &\ds \sum_{i\in[N]} \tau_{i \wh a_i \wh y_i} \ge (1-\gamma) N,
        \end{array}
    \right.
    \end{align*} 
    which is a finite-dimensional linear program. Strong duality result from linear programming asserts that
    \begin{align*}
    \Sup{\QQ \in \mc B_\gamma(\Pnom)}~\EE_{\QQ}[\phi(X,A,Y)] =  \left\{
        \begin{array}{cll}
            \inf & \ds \frac{1}{N} \sum_{i \in [N]} \nu_i + \sum_{(a,y) \in \mc A \times \mc Y} \hat p_{ay} \mu_{ay}-\theta(1-\gamma) \\
            \st & \nu \in \R^N ,~\mu \in \R^{|\mc A|\times |\mc Y|},~\theta \in \R_+,\\
            & v_{iay} \leq \nu_i +\mu_{ay} - \theta \mathbbm{1}_{(\hat a_i, \hat y_i)} (a, y)    &\forall i \in [N] \quad \forall (a, y) \in \mc A \times \mc Y.
        \end{array}
    \right.
    \end{align*} 
    Substituting the definition of $v$ into the above optimization problem completes the proof.
\end{proof}

Equipped with the duality result of Lemma~\ref{lem:strong_dual}, we now present the proof of Theorem~\ref{thm:refor}.
\begin{proof}[Proof of Theorem~\ref{thm:refor}]
Notice that the objective function can be written in the form of
\[ 
        \Sup{\QQ \in \mc B_\gamma(\Pnom)}    \EE_{\QQ} [\phi(X,A,Y)],
\]
where $\phi(X,A,Y)=\max\{0, 1 - Y(w^\top X + b)\}$. Thus, by Lemma \ref{lem:strong_dual}, we have
\begin{align}\label{eq:obj_dual}
&\nonumber\Sup{\QQ \in \mc B_\gamma(\Pnom)}    \EE_{\QQ} [\phi(X,A,Y)]\\
\nonumber=&\left\{
\begin{array}{cll}
 \inf& \ds \frac{1}{N} \sum_{i \in [N]} \nu_i + \sum_{(\bar a, \bar y) \in \mc A \times \mc Y} \hat p_{\bar a \bar y} \mu_{\bar a \bar y}-\theta(1-\gamma) & \\
 \st& \nu \in \R^N,~\theta \in \R_+,~\mu \in \R^{2\times 2},\\
 &\Sup{x\in \mc X: \| x - \wh x_i\| \le \rho - \kappa_{\mc A}|a - \wh a_i| - \kappa_{\mc Y}|y - \wh y_i|} \max\{0,1 - y (w^\top x + b) \} \leq \mu_{ay} - \theta \mathbbm{1}_{(\hat a_i, \hat y_i)} (a,  y) + \nu_i \\
 &\hspace{10cm} \forall i \in [N] \quad \forall (a, y) \in \mc A \times \mc Y.
\end{array}
 \right.
\end{align}
The constraint in the above infimum problem is equivalent to
\[
    \left.\begin{array}{l}
   \textup{If} \ \kappa_{\mc A} |a - \wh a_i| + \kappa_{\mc Y}| y - \wh y_i| \leq \rho:\\
    \quad 0 \leq \mu_{ay} - \theta \mathbbm{1}_{(\hat a_i, \hat y_i)} (a, y) + \nu_i,  \\
    \quad 1 - yb + \Sup{x: \| x - \wh x_i\| \le \rho - \kappa_{\mc A}|a - \wh a_i| - \kappa_{\mc Y}|y - \wh y_i|} \{- yw^\top x\} \leq \mu_{ay} - \theta \mathbbm{1}_{(\hat a_i, \hat y_i)} (a, y) + \nu_i 
  \end{array}
  \right\} \forall i \in [N] \quad \forall (a, y) \in \mc A \times \mc Y.
\]
Recall that $\mc Y = \{-1, +1\}$. Thus, the dual norm relationship implies that
\[
    \Sup{x \in \mc X: \| x - \wh x_i\| \le \rho - \kappa_{\mc A}|a - \wh a_i| - \kappa_{\mc Y}|y - \wh y_i|} \{- yw^\top x\} = -y w^\top \wh x_i + (\rho - \kappa_{\mc A}|a - \wh a_i| - \kappa_{\mc Y}|y - \wh y_i|) \| w \|_*,
\]
which leads to first set of constraints in the reformulation
% \begin{align}\label{eq:obj_dual}
% &\nonumber\Sup{\QQ \in \mc B_\gamma(\Pnom)}    \EE_{\QQ} [\phi(X,A,Y)]\\
% \nonumber=&\left\{
% \begin{array}{cll}
%  \inf& \frac{1}{N} \sum_{i \in [N]} \nu_i + \sum_{(\bar a, \bar y) \in \mc A \times \mc Y} \hat p_{\bar a \bar y} \mu_{\bar a \bar y}-\theta(1-\gamma) & \\
%  \st& \nu \in \R^N,\; \theta \in \R_+,\; \mu \in \R^{2\times 2}\\
%  &\max\{0,1 - \bar y_i (w^\top x_i + b) \} \leq \mu_{\bar a_i \bar y_i} - \theta \mathbbm{1}_{(\hat a_i, \hat y_i)} (\bar a_i, \bar y_i) + \nu_i & \forall (x_i,\bar a_i,\bar y_i) \in \Xi_i, \quad \forall i \in [N]
% \end{array}
%  \right.\\
%  =&\left\{
% \begin{array}{cll}
% \inf& \frac{1}{N} \sum_{i \in [N]} \nu_i + \sum_{(\bar a, \bar y) \in \mc A \times \mc Y} \hat p_{\bar a \bar y} \mu_{\bar a \bar y}-\theta(1-\gamma) & \\
%  \st& \nu \in \R^N , \theta \in \R_+, \mu \in \R^{2\times 2}\\
%  &0  \leq \mu_{\bar a_i \bar y_i} - \theta \mathbbm{1}_{(\hat a_i, \hat y_i)} (\bar a_i, \bar y_i) + \nu_i & \forall (\bar a_i, \bar y_i):  \kappa_{\mc A} |\bar a_i - \wh a_i| + \kappa_{\mc Y}| \bar y_i - \wh y_i| \le \rho, \quad \forall i \in [N]\\
%  &1 - \bar y_i(w^\top x_i + b)  \leq \mu_{\bar a_i \bar y_i} - \theta \mathbbm{1}_{(\hat a_i, \hat y_i)} (\bar a_i, \bar y_i) + \nu_i & \forall (x_i,\bar a_i, \bar y_i) \in \Xi_i, \quad \forall i \in [N].
% \end{array}
%  \right.
% \end{align}

Next, we show the derivation for constraints. Recall that the worst-case hinge-loss unfairness measure can be written as
\begin{align*}
\Sup{\QQ \in \mc B_\gamma(\Pnom)}~\mathds H(w, b, \QQ) 
    &= \max \left\{ \begin{array}{l}
    \Sup{\QQ \in \mc B_\gamma(\Pnom)}~\EE_{\QQ_{01}}\left[\max\{0, 1+w^\top X + b \}\right] + \EE_{\QQ_{11}}\left[\max\{0,1-w^\top X - b \}\right] -1 , \\
    \Sup{\QQ \in \mc B_\gamma(\Pnom)}~\EE_{\QQ_{11}}\left[\max\{0, 1+w^\top X + b \}\right] + \EE_{\QQ_{01}}\left[\max\{0,1-w^\top X - b \}\right] -1
    \end{array}
    \right\}.
\end{align*}
Consider a fixed pair of $(a, a') \in \{ (0, 1), (1, 0)\}$. Employing the result of Lemma \ref{lem:strong_dual} yields
\begin{align*}
&\Sup{\QQ \in \mc B_\gamma(\Pnom)}~\EE_{\QQ_{a1}}\left[\max\{0, 1+w^\top X + b \}\right] + \EE_{\QQ_{a'1}}\left[\max\{0,1-w^\top X - b \}\right] -1 \\
    =& \Sup{\QQ \in \mc B_\gamma(\Pnom)}~ \EE_{\QQ}[\wh p_{a1}^{-1} \max\{0, 1+w^\top X + b \} \mathbbm{1}_{(a,1)}(A, Y) - \wh p_{a'1}^{-1} \max\{0,1-w^\top X - b \} \mathbbm{1}_{(a',1)}(A, Y)]-1\\
    =&\left\{
\begin{array}{cl}
 \inf& \ds\frac{1}{N} \sum_{i \in [N]} \nu_i^a + \sum_{(\bar a, \bar y) \in \mc A \times \mc Y} \hat p_{ay} \mu_{\bar a \bar y}^a-\theta^a(1-\gamma) - 1  \\
 \st& \nu^a \in \R^N ,\; \theta^a \in \R_+,\; \mu^a \in \R^{2\times 2},\\
 &\Sup{x_i \in \mc X: \| x_i - \wh x_i\| \le \rho - \kappa_{\mc A}| \bar a_i - \wh a_i| - \kappa_{\mc Y}| \bar y_i - \wh y_i|}\phi_a(x_i, \bar a_i, \bar y_i) \leq \mu_{\bar a_i \bar y_i}^a - \theta^a \mathbbm{1}_{(\hat a_i, \hat y_i)} (\bar a_i, \bar y_i) + \nu_i^a \\
 &\hspace{10cm} \forall i \in [N] \quad \forall (\bar a_i, \bar y_i) \in \mc A \times \mc Y,
\end{array}
 \right.
\end{align*}
where the second equality relies on the result of Lemma \ref{lem:strong_dual} by defining 
\[
\phi_a(X,A,Y)=\wh p_{a1}^{-1} \max\{0, 1+w^\top X + b \} \mathbbm{1}_{(a,1)}(A, Y) - \wh p_{a'1}^{-1} \max\{0,1-w^\top X - b \} \mathbbm{1}_{(a',1)}(A, Y).\] 
Fix any $i \in [N]$, we now iterate over different values of $(\bar a_i, \bar y_i)$.
\begin{enumerate}
    \item Case 1: $(\bar a_i, \bar y_i) = (a, 1)$. There is an active constraint if $\kappa_{\mc A} |a - \wh a_i| + \kappa_{\mc Y}| 1 - \wh y_i| \leq \rho$, and the constraint is equivalent to
\begin{align*}
    &\Sup{\forall x_i \in \mc X : \| x_i - \wh x_i\| \leq \rho- \kappa_{\mc A} |a - \wh a_i| - \kappa_{\mc Y}| 1 - \wh y_i|}~\wh p_{a1}^{-1} \max\{0, 1+w^\top x_i + b \} \mathbbm{1}_{(a,1)}(a, 1)   \\
    &\hspace{4cm}+ \wh p_{a'1}^{-1} \max\{0, 1-w^\top x_i - b \} \mathbbm{1}_{(a',1)}(a, 1) \leq \mu_{a1}^a - \theta^a \mathbbm{1}_{(\hat a_i, \hat y_i)} (a, 1) + \nu_i^a  \\
    \Longleftrightarrow & \Sup{\forall x_i \in \mc X: \| x_i - \wh x_i\| \leq \rho- \kappa_{\mc A} |a_ - \wh a_i| - \kappa_{\mc Y}| 1 - \wh y_i| }  \wh p_{a1}^{-1} \max\{0, 1+w^\top x_i + b \}  \leq \mu_{a1}^a - \theta^a \mathbbm{1}_{(\hat a_i, \hat y_i)} (a, 1) + \nu_i^a \\
    \Longleftrightarrow& \left\{\begin{array}{l}
         0 \leq \wh p_{a1}( \mu_{a1}^a - \theta^a \mathbbm{1}_{(\hat a_i, \hat y_i)} (a, 1) + \nu_i^a ), \\
         1+w^\top \wh x_i + (\rho- \kappa_{\mc A} |a - \wh a_i| - \kappa_{\mc Y}| 1 - \wh y_i|) \|w\|_* + b \leq \wh p_{a1}( \mu_{a1}^a - \theta^a \mathbbm{1}_{(\hat a_i, \hat y_i)} (a, 1) + \nu_i^a ),
    \end{array}
    \right.
\end{align*}
where the last equivalence follows from the definition of dual norm. 
    \item Case 2: $(\bar a_i,\bar  y_i)= (a', 1)$. There is an active constraint if $\kappa_{\mc A} |a' - \wh a_i| + \kappa_{\mc Y}| 1 - \wh y_i| \leq \rho$, and the constraint is equivalent to
\begin{align*}
    &\Sup{\forall x_i \in \mc X: \| x_i - \wh x_i\| \leq \rho- \kappa_{\mc A} |a' - \wh a_i| - \kappa_{\mc Y}| 1 - \wh y_i|}~\wh p_{a1}^{-1} \max\{0, 1+w^\top x_i + b \} \mathbbm{1}_{(a,1)}(a', 1) \\
    &\hspace{4cm}+ \wh p_{a'1}^{-1} \max\{0, 1-w^\top x_i - b \} \mathbbm{1}_{(a',1)}(a', 1)     \leq \mu_{a'1}^a - \theta^a \mathbbm{1}_{(\hat a_i, \hat y_i)} (a', 1) + \nu_i^a \\
    \Longleftrightarrow& \Sup{\forall x_i \in \mc X: \| x_i - \wh x_i\| \leq \rho- \kappa_{\mc A} |a' - \wh a_i| - \kappa_{\mc Y}| y_i - \wh y_i| } \wh p_{a'1}^{-1} \max\{0, 1-w^\top x_i - b \}  \leq \mu_{a'1}^a - \theta^a \mathbbm{1}_{(\hat a_i, \hat y_i)} (a', 1) + \nu_i^a\\
    \Longleftrightarrow& \left\{\begin{array}{l}
     0 \leq  \wh p_{a'1}( \mu_{a'1}^a - \theta^a \mathbbm{1}_{(\hat a_i, \hat y_i)} (a', 1) + \nu_i^a ), \\
     1 -w^\top \wh x_i + (\rho- \kappa_{\mc A} |a' - \wh a_i| - \kappa_{\mc Y}| 1 - \wh y_i|) \|w\|_* - b  \leq  \wh p_{a'1} (\mu_{a'1}^a - \theta^a \mathbbm{1}_{(\hat a_i, \hat y_i)} (a', 1) + \nu_i^a ).
    \end{array}
    \right.
\end{align*}

    \item Case 3: $(\bar a_i, \bar y_i)= (a, -1)$. There is an active constraint if $\kappa_{\mc A} |a - \wh a_i| + \kappa_{\mc Y}| -1 - \wh y_i| \leq \rho$, the constraint is equivalent to
\begin{align*}
        &\Sup{\forall x_i \in \mc X : \| x_i - \wh x_i\| \leq \rho- \kappa_{\mc A} |a - \wh a_i| - \kappa_{\mc Y}| -1 - \wh y_i|}~\wh p_{a1}^{-1} \max\{0, 1+w^\top x_i + b \} \mathbbm{1}_{(a,1)}(a, -1) \\
    &\hspace{4cm}+ \wh p_{a'1}^{-1} \max\{0, 1-w^\top x_i - b \} \mathbbm{1}_{(a',1)}(a, -1)\leq \mu_{a,-1}^a - \theta^a \mathbbm{1}_{(\hat a_i, \hat y_i)} (a, -1) + \nu_i^a \\
    \Longleftrightarrow& 0 \leq \mu_{a,-1}^a - \theta^a \mathbbm{1}_{(\hat a_i, \hat y_i)} (a, -1) + \nu_i^a.\\
\end{align*}

    \item Case 4: $(\bar a_i, \bar y_i)= (a', -1)$. There is an active constraint if $\kappa_{\mc A} |a' - \wh a_i| + \kappa_{\mc Y}| -1 - \wh y_i| \leq \rho$, the constraint is equivalent to
\begin{align*}
    &\Sup{\forall x_i \in \mc X : \| x_i - \wh x_i\| \leq \rho- \kappa_{\mc A} |a' - \wh a_i| - \kappa_{\mc Y}| -1 - \wh y_i| }~\wh p_{a1}^{-1} \max\{0, 1+w^\top x_i + b \} \mathbbm{1}_{(a,1)}(a', -1)  \\
    &\hspace{4cm}+ \wh p_{a'1}^{-1} \max\{0, 1-w^\top x_i - b \} \mathbbm{1}_{(a',1)}(a', -1)  \leq \mu_{a',-1}^a - \theta \mathbbm{1}_{(\hat a_i, \hat y_i)} (a', -1) + \nu_i^a\\
    \Longleftrightarrow& 0 \leq \mu_{a',-1}^a - \theta^a \mathbbm{1}_{(\hat a_i, \hat y_i)} (a', -1) + \nu_i^a.\\
\end{align*}   

\end{enumerate}

Notice that at least one of the above four conditions will be satisfied, because when $\bar a_i=\wh a_i$ and $\bar y_i=\wh y_i$, we have 
\[  \kappa_{\mc A} |a_i - \wh a_i| + \kappa_{\mc Y}| y_i - \wh y_i|=0\leq \rho
\] for any $\rho \geq 0$. Combining all four cases leads to the second set of constraints.

The last constraint in the reformulation is obtained by setting the optimal value of the dual problem to be less than $\zeta$ for each value of $a \in \mc A$. This completes the proof.
\end{proof}

\section{Numerical Experiments}
\label{sect:numerical}

In this section we present the numerical experiments and examine the performance of different distributionally robust fair classifiers. Except for the DOB+ method~~\cite{ref:donini2018empirical} which is solved by an \textit{sklearn} built-in solver, all other optimization problems are implemented in Python 3.7 with package CVXPY 1.1.0 and solved by MOSEK 9.2. The experiments were run on a 2.2GHz Intel Core i7 CPU laptop with 8GB RAM.

\subsection{Synthetic Experiments}
We visualize the classification hyperplanes determined by the $\eps$-DRFC model~\eqref{eq:dro-prob2} and the HDRFC model~\eqref{eq:dro-svm} on a toy dataset with 200 samples (50 for training, 150 for testing) and $d=2$ features. For each of the classifiers, we will plot three variants. The $\eps$-C and HC classification hyperplanes are obtained by dropping the fairness constraints and setting the Wasserstein radius to zero. The $\eps$-FC and the HFC classifiers include the fairness constraint, but still without robustness consideration. The $\eps$-DRFC and HDRFC models include both the fairness constraints and robustification with $\rho > 0$.
We choose the ground cost of the form~\eqref{eq:cost} with $\|\cdot\|$ being the $l_{\infty}$-norm and $\kappa_{\mc A} = \kappa_{\mc Y} = \infty$. 
%The Wasserstein radius is set to $\rho = 0.05$ and the upper bound $\eta$ is chosen to be $0.1$ and $1.1$ for the exact models and SVM models respectively.

% \begin{figure}[h]
% \includegraphics[width=10cm]{figA1_4.pdf}
% \caption{Classification hyperplanes (dashed) obtained by different approaches. Color decodes the labels (blue for +1 labels and red for -1 labels)} 
% \label{fig:synthetic_exper}
% \end{figure}

\begin{figure}[!h]
    \centering
    \begin{subfigure}[b]{0.45\textwidth}
    \centering
    \includegraphics[width=\textwidth]{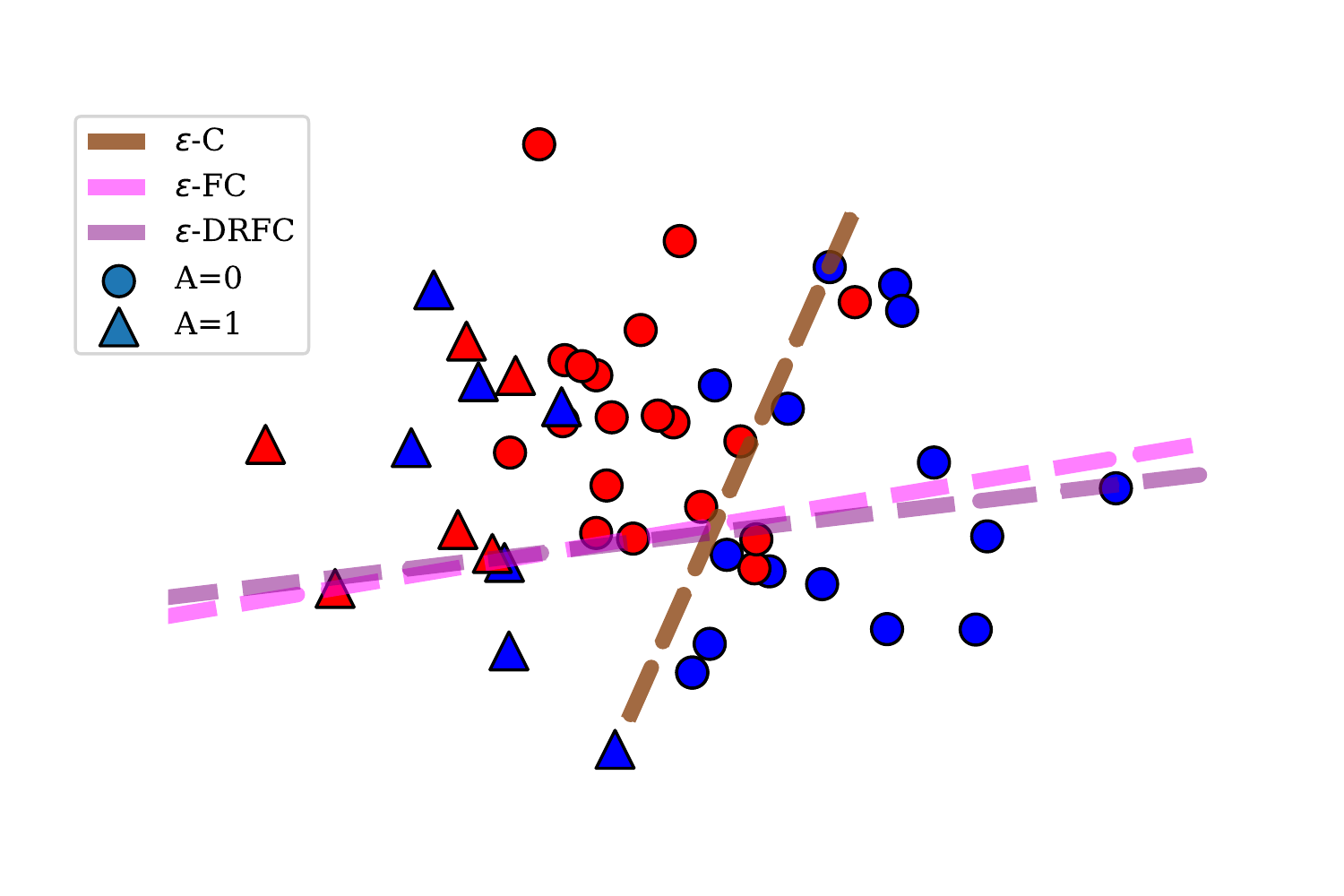}
    \caption{$\eps$-DRFC}
    \end{subfigure}
    \begin{subfigure}[b]{0.45\textwidth}
    \centering
    \includegraphics[width=\textwidth]{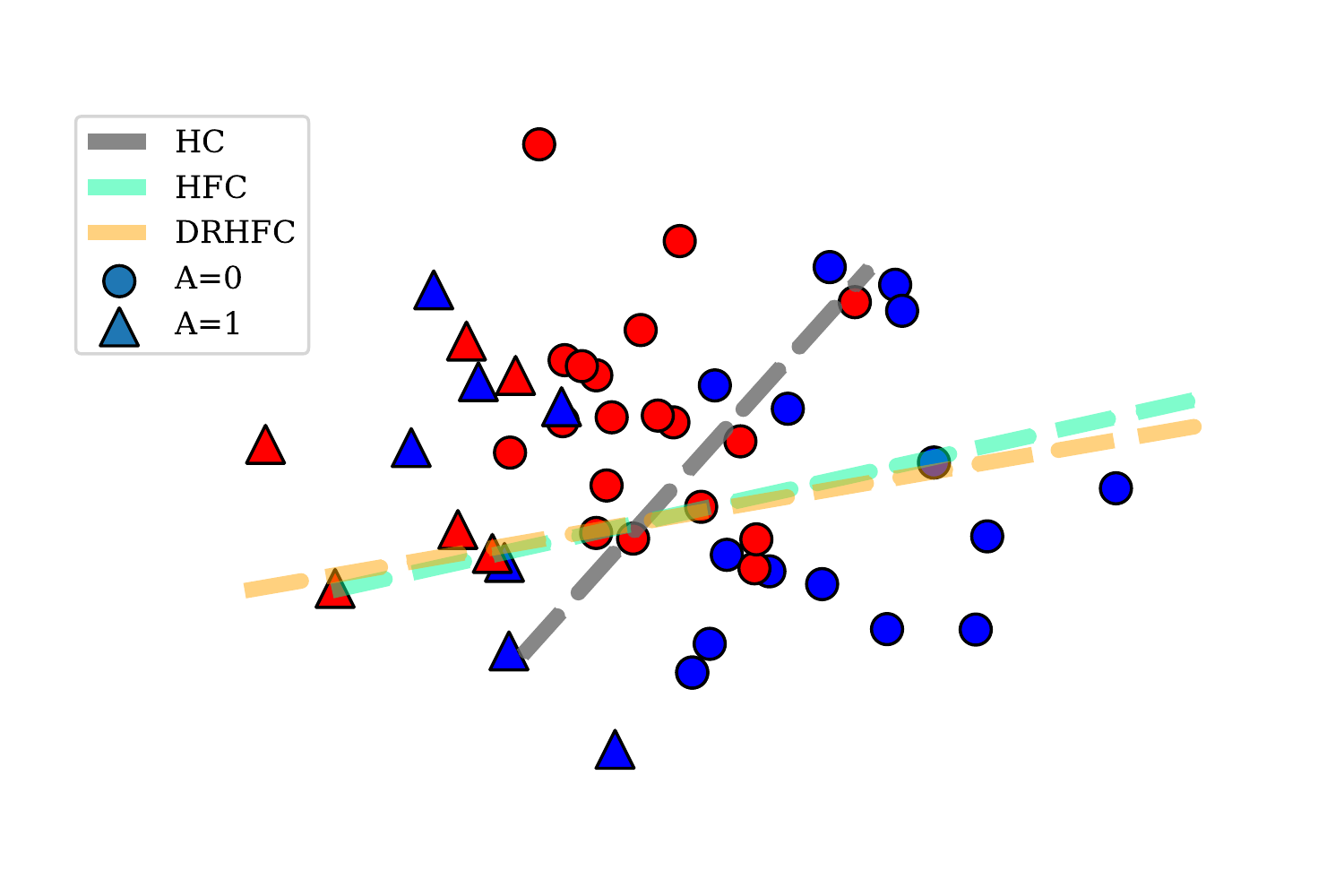}
    \caption{HDRFC}
    \end{subfigure}
    \caption{Classification hyperplanes (dashed) obtained by different approaches. Color encodes the labels (blue for +1 labels and red for -1 labels).}
    \label{fig:synthetic_exper}
\end{figure}

We first demonstrate how the unfairness constraints and the robustification influence the classifiers. Notice that the sensitive attribute $A$ (represented by circles and triangles) is correlated with the feature $X_1$ on the horizontal axis. In the two graphs, all of the four fairness-aware classifiers ($\eps$-FC, $\eps$-DRFC, HFC, HDRFC) assign lower absolute value for the weight $w_1$ corresponding to feature $X_1$. Visually, this shift is reflected by the hyperplane determined by them becoming more horizontal compared to that of $\eps$-C and HC. Moreover, the $\eps$-DRFC and HDRFC models, by being robust, shift their hyperplanes even more horizontal to reduce the dependence of the classifiers on $X_1$ compared to the $\eps$-FC and HFC models.

Next, we benchmark the two objective functions using the $\eps$-C model and the HC model. Recall that the $\eps$-C model minimizes the $\eps$-approximation of the in-sample misclassification rate, i.e., $\hat \PP( Y(w^\top X + b) < \eps)$. In contrast, the HC model minimizes the empirical hinge loss. Note that the HC model is actually equivalent to the well known Support Vector Machine model \cite{cortes1995support}, while $\eps$-C and HC are both vanilla classification models without any fairness constraint. In this graph, we can see a total of $10$ misclassified points in the $\eps$-C result, and a total of $15$ misclassified points in the HC result. Thus, minimizing the $\eps$-approximation of the misclassification probability achieves a higher accuracy in this training dataset.

We then assess the unfairness and accuracy scores on the test set. Compared with the other vanilla classification model HC, the $\eps$-C model achieves higher accuracy in the test set. In addition, by including the fairness constraint, all the fairness-aware classifiers can reduce the unfairness score with a moderate cost of accuracy. Moreover, with the distributionally robust setting, the classification hyperplanes achieve even lower test unfairness scores (evaluated using the EO unfairness measure $\mathds U$ and the empirical distribution supported on the test data).

%And compared to the vanilla model SVM, the FSVM and DRFSVM lower the unfairness from $0.71$ to $0.44$ and $0.37$ at the cost of reducing the accuracy from $65.68\%$ to $63.27\%$ and $62.32\%$, respectively. Details are exhibited in Table~\ref{table:synthetic}.
\begin{table}[h]
\centering
\begin{tabular}{ |p{3cm}||p{3cm}|p{3cm}|  }
 \hline
 Classifier & Test accuracy & Test unfairness ($\mathds U$)\\
 \hline
 $\eps$-C  & 71.03\%   &0.8076\\
 $\eps$-FC  & 57.08\%   &0.0973\\
 $\eps$-DRFC  & 55.92\%   &0.0554\\
 \hline\hline
 HC  & 65.68\%   &0.7092\\
 HFC  & 58.44\%   &0.1500\\
 HDRFC & 57.72\%&  0.1117\\
%  FMPM  & 57.08\%   &0.0973\\
%  DRFMPM & 55.92\%&  0.0554\\
 \hline
\end{tabular}
\caption{Predictive accuracy and unfairness on test data for the synthetic experiment.}
\label{table:synthetic}
\end{table}
% \viet{maybe good to specify what we mean by unfairness. The reason is that we now have multiple version of $\mathds U$, $\mathds U_\eps$ $\mathds H$, etc.}

In the second set of synthetic experiments, we compare the performance of our models against the DOB+~\cite{ref:donini2018empirical} and DRFLR~\cite{ref:taskesen2020distributionally}. The DOB+ model is the state-of-the-art method in deterministic linear classification. It minimizes the empirical hinge loss in the objective function, and adopts a \textit{linear-loss-based} unfairness measure to approximate the EO unfairness measure in the constraint. The DRFLR model is a distributionally robust logistic regression model. It minimizes the empirical \textit{log-loss} together with a fairness-driven regularization term in the objective function. Specifically, the paper proposes a \textit{log-probabilistic equalized opportunities} unfairness measure, which is a convex approximation of the EO unfairness measure, as the fairness-driven regularization term. The DRFLR model is also considered as the state-of-the-art method in distributionally robust fair logistic regression.

We plot the Pareto frontiers of the $\eps$-FC, $\eps$-DRFC, HFC, and HDRFC against those of~DOB+ and DRFLR in Figure~\ref{fig:frontier}. The setup for this experiment follows from the synthetic experiment in \cite{ref:zafar2015fairness}. The data-generating probability distribution~$\PP\opt$ satisfies $\mathbb P\opt(Y=1)= \mathbb P\opt(Y=-1)=0.5$, while the conditional distribution of the 2-dimensional feature vectors are set as the following Gaussian distributions
\[
X|Y=1 \sim \mathcal N([2;2],[5,1;1,5]), \qquad  X|Y=-1 \sim \mathcal N([-2;-2],[10,1;1,3]).\] Next, we generate the sensitive feature for each sample $x$ from a Bernoulli distribution 
\[
\mathbb P\opt(A=1|X=x') = \frac{pdf(x'|Y=1)}{pdf(x'|Y=1)+ pdf(x'|Y=-1)},
\] 
where $x'=[\cos(\pi/4),\sin(\pi/4);\sin(\pi/4),\cos(\pi/4)]x$ is a rotated value of the feature vector $x$ and $pdf(\cdot|Y=y)$ is the Gaussian probability density function of $X|Y=y$.

We draw 200 samples from the data generating distribution $\mathbb P\opt$, and then separate them into a group of~$50$ samples used for the training, while the remaining 150 samples are used as the test set. For the $\eps$-FC and $\eps$-DRFC models, we examine the models with different values of the unfairness controlling parameter $\eta$ on~$[0.05,0.25]$ with 5 equidistant points. Similarly, we examine the HFC and HDRFC models with $\zeta$ on $[1.2,1.4]$ with 5 equidistant points, and the DRFLR model with $\eta_f$ on $[0.1, \min\{\wh p_{01}, \wh p_{11}\}]$.\footnote{ The DRFLR model admits tractable reformulations only if $\eta_f \leq \min\{\wh p_{01}, \wh p_{11}\}$} We fix the Wasserstein radius of the $\eps$-DRFC and HDRFC models to $0.25$, and the radius of the DRFLR model to $0.015$. 
Since the authors of the DOB+ method argue that 0 is a reasonable selection for the unfairness controlling parameter in their model, and their code is implemented under this prerequisite, to be consistent with their paper, we fix this parameter for the DOB+ method in our experiment. The hyperparameter $C$ of the DOB+ method is chosen from $[10^{-1},10^1]$ by cross-validation using the authors' code.\footnote{\url{https://github.com/jmikko/fair_ERM}} The described procedure is repeated~$50$ times independently, and the results are averaged over 50 trials.

\begin{figure}[h]
\includegraphics[width=10cm]{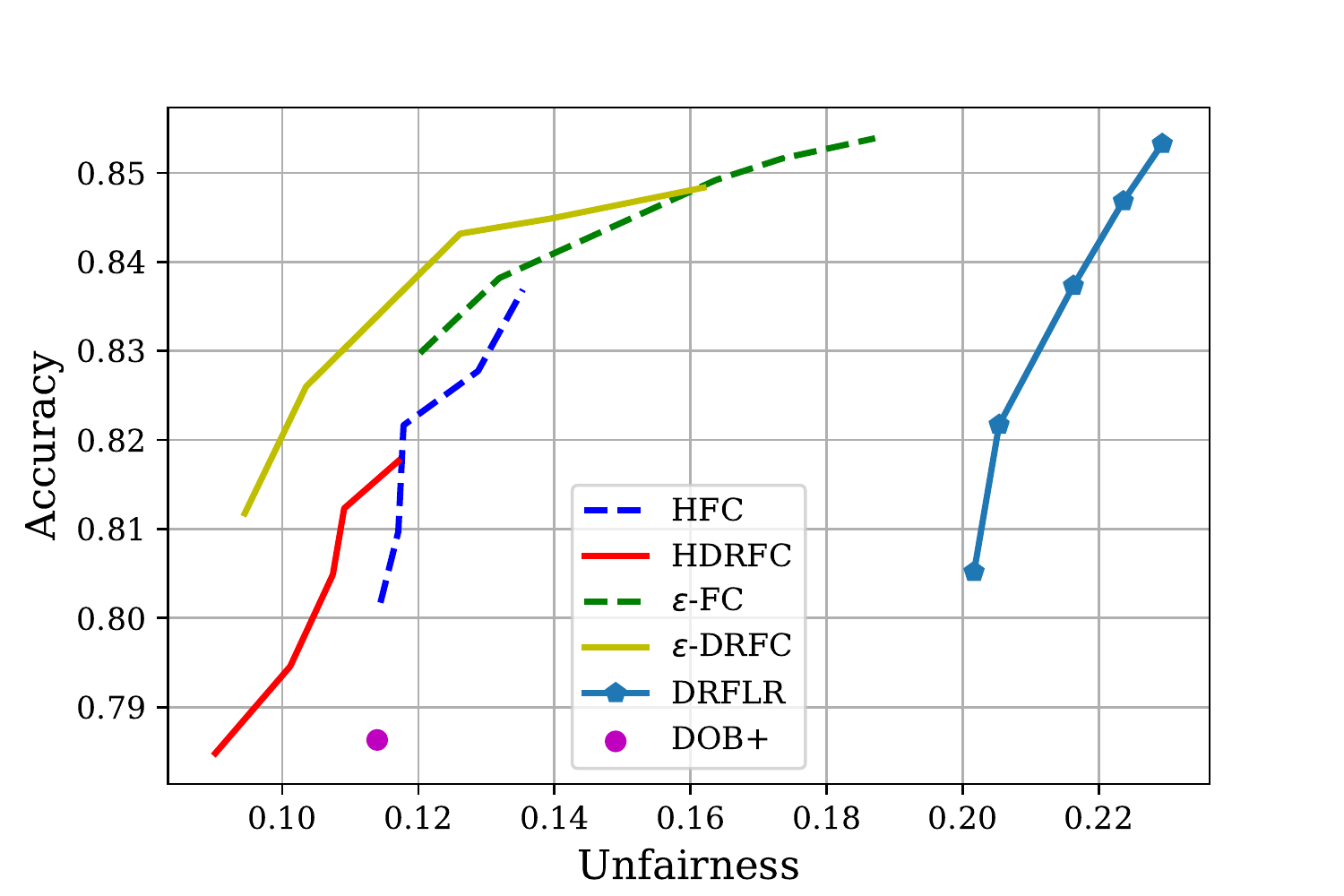}
\caption{Unfairness-accuracy Pareto frontiers for different approaches.}
\label{fig:frontier}
\end{figure}

Figure \ref{fig:frontier} visualizes the Pareto frontiers of six fairness-aware models in the out-of-sample test, where the dashed lines represent the non-robust models ($\eps$-FC and HFC), the solid lines represent the distributionally robust models ($\eps$-DRFC and HDRFC), and the dotted-solid line represents the DRFLR model. Compared to the DOB+ solution (purple dot), our four methods all achieve higher classification accuracy at the same unfairness score. And compared with the DRFLR method, our methods also obtain lower unfairness scores at the same accuracy level. The $\eps$-FC and $\eps$-DRFC models, benefited from their conservative approximation reformulation, dominate the HFC and HDRFC models across all unfairness scores. Nevertheless, the HFC and HDRFC model still perform better than the DOB+ method, and because of the excellent scalability, they are more suitable for practical problems.

% At the extreme, DRFSVM achieves the lowest out-of-sample unfairness measure of $0.09$, while the DOB+ and FSVM model only reach the unfairness value at $0.11$ and $0.14$, respectively. 

\subsection{Experiments with Real Data.} 
We then assess the performance of the HDRFC model and demonstrate its superior performance on four publicly available datasets (Adult, Drug, COMPAS, Arrhythmia). The reason why we select this model is that the $\eps$-DRFC model may encounter computational difficulties in the face of large instances, which we illustrate  in Section \ref{subsec:soltime}. A brief summary of these four datasets is presented in Table \ref{tab:statistics}. While the Adult dataset has already been divided into the training and testing sets, we randomly select $2/3$ samples for training and keep the rest of the data for testing in all other three datasets. 

\begin{table}[ht]
    \centering
    \begin{tabular}{||p{2.5cm}|c|c|c||}
    \hline
         Dataset& Features $d$ & Sensitive Attribute $A$ & Number of samples  \\
         \hline\hline
         Adult  & 12 & Gender & 32561, 12661 \\
         Drug &  11 &  Ethnicity & 1885\\ 
         COMPAS & 10 & Ethnicity & 6172 \\
         Arrhythmia & 279(15) & Gender & 452 \\
         \hline
     \hline
    \end{tabular}
    \caption{Datasets statistics and their sensitive feature. Gender considers the two groups as male and female; ethnicity considers the ethnic groups white and other ethnic groups. The adult dataset has pre-assigned training and test sets.}
    \label{tab:statistics}
\end{table}

We now formally benchmark the models following a cross-validation, training, and testing procedure. The hyperparameter of the HDRFC and DRFLR models, i.e., the radius of the Wasserstein ball, is determined in a cross-validation procedure similar to \cite{ref:donini2018empirical}. We first split the training set into a sub-training set with $200$ samples and keep the remaining samples as a sub-validation set. 
Then we collect statistics (i.e., accuracy score, EO unfairness measure) of $\rho \in [5 \cdot 10^{-3},5]$ on a logarithm searching grid with 40 discretization points based on the sub-training and sub-validation sets. %Notice that the maximal value in the grid search for $\rho$ equals to $2 \kappa_{\mc Y}$, which is sufficient to induce the perturbation on both the label $Y$ and the sensitive attribute $A$. 
This process is repeated $K_1=5$ times, and the average accuracy and unfairness are recorded for each candidate value. Finally, we select the value with the highest (Accuracy~$-$~$0.5 \times$ Unfairness) score from the list. Similarly, the tuning parameter $C$ of the DOB+ method is also determined by cross-validation using the author's code.

With the hyperparameters obtained from cross-validation, we now retrain the four classifiers using another random draw of $N=300$ samples from the training set. We set $\zeta=1.1$ for HFC and HDRFC, $\eta_f=\min\{\wh p_{01}, \wh p_{11}\}/2$ for DRFLR, and assume absolute trust in labels and sensitive features. The DOB+ method is computed using the authors’ code. The accuracy and unfairness measures of all classifiers are then evaluated on the test set. We repeat this process for $K_2=100$ times and report the average accuracy scores and unfairness measures on Table \ref{tab:performance}.

Table \ref{tab:performance} suggests that our HDRFC model performs favorably relative to its competitors: it yields the lowest unfairness score across three datasets with only a moderate loss in accuracy. DRFLR is also an attractive method, as its distributionally robust setting enables the model to achieve good out-of-sample performance.  Meanwhile, even without robustness setting, the HFC model still outperforms the DOB+ method, implying that the hinge unfairness measure can better recover the EO unfairness measure compared with linear-loss-based unfairness measure. Since the statistics are averaged over $100$ replications, and in each replication, testing the $\eps$-DRFC model requires solving five mixed binary programs that take a while, we only report the result of the aforementioned convex models. In the next experiment, we will further demonstrate the efficiency of these methods.

\begin{table}[h]
\begin{center}
 \begin{tabular}{||p{2.5cm}|p{2.5cm} | c| c| c| c| c||} 
 \hline
 Dataset & Metric & HC(SVM) & HFC & DOB+ & DRFLR & HDRFC\\ [1ex] 
 \hline\hline
 \multirow{2}{4em}{Adult} & Accuracy & $\pmb{0.80 \pm 0.01}$ & ${0.79 \pm 0.01}$ & $0.79 \pm 0.02$ & ${0.79 \pm 0.02}$ & ${0.79 \pm 0.02}$ \\ 
 &Unfairness ($\mathds U$)& $0.11 \pm 0.09$ & $0.03 \pm 0.03$ & $0.07 \pm 0.06$ &${0.06 \pm 0.03}$  &$\pmb{0.03 \pm 0.01}$  \\
 \hline
 \multirow{2}{4em}{Drug} & Accuracy & $\pmb{0.80 \pm 0.01}$ & $\pmb{0.80 \pm 0.01}$ & $\pmb{0.80 \pm 0.02}$ & $0.78 \pm 0.02$ & $0.79 \pm 0.02$ \\
 &Unfairness ($\mathds U$)& $0.15 \pm 0.07$ & $0.12 \pm 0.07$ & $0.12 \pm 0.08$ &$\pmb{0.03 \pm 0.02}$&${0.06 \pm 0.07}$ \\
 \hline
 \multirow{2}{4em}{COMPAS} & Accuracy & $\pmb{0.64 \pm 0.01}$ & ${0.57 \pm 0.04}$ & $0.56 \pm 0.04$ & $0.57 \pm 0.03$ & $0.56 \pm 0.02$ \\
 &Unfairness ($\mathds U$)& $0.23 \pm 0.03$ & $0.11 \pm 0.05$ & $0.12 \pm 0.06$ &${0.09 \pm 0.05}$&$\pmb{0.06 \pm 0.03}$ \\
 \hline
 \multirow{2}{4em}{Arrhythmia} & Accuracy & $\pmb{0.65 \pm 0.03}$ & $\pmb{0.65 \pm 0.03}$ & $0.64 \pm 0.03$ & $0.64 \pm 0.02$& $0.64 \pm 0.02$ \\
 &Unfairness ($\mathds U$)& $0.24 \pm 0.09$ & $0.09 \pm 0.07$ & $0.10 \pm 0.08$ & $0.08 \pm 0.06$ &$\pmb{0.04 \pm 0.04}$ \\
 \hline
\end{tabular}
\end{center}
\caption{Test accuracy and unfairness (average $\pm$ standard deviation) for $N = 300$. The best results for each dataset is highlighted in bold.}
\label{tab:performance}
\end{table}

\begin{table}[]
    \centering
    \begin{tabular}{p{2.5cm}|p{2cm}|r|r|r|r|r|r|}
         \multicolumn{1}{c}{}& \multicolumn{1}{c}{}& \multicolumn{6}{c}{Sample size $N$ } \\
         \hline
         Dataset& Classifier& \multicolumn{1}{r}{50}&\multicolumn{1}{r}{100} &\multicolumn{1}{r}{250} &\multicolumn{1}{r}{500} &\multicolumn{1}{r}{750} &\multicolumn{1}{r}{1000}\\
         \hline
         \multirow{3}{4em}{Adult}&$\eps$-FC&\multicolumn{1}{r}{2.38} &\multicolumn{1}{r}{13.37}&\multicolumn{1}{r}{35.69}& \multicolumn{1}{r}{93.28}&\multicolumn{1}{r}{346.84}&\multicolumn{1}{r}{769.432}\\
         & $\eps$-DRFC &\multicolumn{1}{r}{2.15} &\multicolumn{1}{r}{19.58}&\multicolumn{1}{r}{42.31}& \multicolumn{1}{r}{108.80}&\multicolumn{1}{r}{431.36}&\multicolumn{1}{r}{809.68}\\
         & HFC &\multicolumn{1}{r}{0.02} &\multicolumn{1}{r}{0.02}&\multicolumn{1}{r}{0.04}& \multicolumn{1}{r}{0.08}&\multicolumn{1}{r}{0.12}&\multicolumn{1}{r}{0.16}\\
        & HDRFC &\multicolumn{1}{r}{0.03} &\multicolumn{1}{r}{0.03}&\multicolumn{1}{r}{0.04}& \multicolumn{1}{r}{0.10}&\multicolumn{1}{r}{0.13}&\multicolumn{1}{r}{0.15}\\
         & DOB+ & \multicolumn{1}{r}{0.02} & \multicolumn{1}{r}{0.03} & \multicolumn{1}{r}{0.07}& \multicolumn{1}{r}{0.16} & \multicolumn{1}{r}{0.35} & \multicolumn{1}{r}{0.43}\\
          & DRFLR & \multicolumn{1}{r}{4.03} & \multicolumn{1}{r}{8.37} & \multicolumn{1}{r}{21.00}& \multicolumn{1}{r}{44.23} & \multicolumn{1}{r}{67.37} & \multicolumn{1}{r}{91.79}\\
         \hline
          \multirow{3}{4em}{Drug}&$\eps$-FC&\multicolumn{1}{r}{1.61} &\multicolumn{1}{r}{18.93}&\multicolumn{1}{r}{42.14}& \multicolumn{1}{r}{127.93}&\multicolumn{1}{r}{347.53}&\multicolumn{1}{r}{812.45}\\
         & $\eps$-DRFC &\multicolumn{1}{r}{1.70} &\multicolumn{1}{r}{21.83}&\multicolumn{1}{r}{66.59}& \multicolumn{1}{r}{145.37}&\multicolumn{1}{r}{357.74}&\multicolumn{1}{r}{865.12}\\
         & HFC &\multicolumn{1}{r}{0.01} &\multicolumn{1}{r}{0.02}&\multicolumn{1}{r}{0.02}& \multicolumn{1}{r}{0.07}&\multicolumn{1}{r}{0.10}&\multicolumn{1}{r}{0.14}\\
        & HDRFC &\multicolumn{1}{r}{0.02} &\multicolumn{1}{r}{0.02}&\multicolumn{1}{r}{0.03}& \multicolumn{1}{r}{0.08}&\multicolumn{1}{r}{0.11}&\multicolumn{1}{r}{0.15}\\
         & DOB+ & \multicolumn{1}{r}{0.02} & \multicolumn{1}{r}{0.03} & \multicolumn{1}{r}{0.07}& \multicolumn{1}{r}{0.15} & \multicolumn{1}{r}{0.19} & \multicolumn{1}{r}{0.30}\\
        & DRFLR & \multicolumn{1}{r}{3.75} & \multicolumn{1}{r}{7.07} & \multicolumn{1}{r}{21.43}& \multicolumn{1}{r}{45.81} & \multicolumn{1}{r}{68.48} & \multicolumn{1}{r}{90.56}\\
         \hline
         \multirow{3}{4em}{COMPAS}&$\eps$-FC&\multicolumn{1}{r}{1.31} &\multicolumn{1}{r}{22.04}&\multicolumn{1}{r}{73.98}& \multicolumn{1}{r}{237.48}&\multicolumn{1}{r}{548.08}&\multicolumn{1}{r}{1075.67}\\
         & $\eps$-DRFC &\multicolumn{1}{r}{2.52} &\multicolumn{1}{r}{21.61}&\multicolumn{1}{r}{84.76}& \multicolumn{1}{r}{215.43}&\multicolumn{1}{r}{447.87}&\multicolumn{1}{r}{1174.07}\\
      & HFC &\multicolumn{1}{r}{0.02} &\multicolumn{1}{r}{0.03}&\multicolumn{1}{r}{0.05}& \multicolumn{1}{r}{0.09}&\multicolumn{1}{r}{0.12}&\multicolumn{1}{r}{0.16}\\
        & HDRFC &\multicolumn{1}{r}{0.02} &\multicolumn{1}{r}{0.04}&\multicolumn{1}{r}{0.05}& \multicolumn{1}{r}{0.11}&\multicolumn{1}{r}{0.14}&\multicolumn{1}{r}{0.17}\\
         & DOB+ & \multicolumn{1}{r}{0.02} & \multicolumn{1}{r}{0.03} & \multicolumn{1}{r}{0.02}& \multicolumn{1}{r}{0.15} & \multicolumn{1}{r}{0.18} & \multicolumn{1}{r}{0.16}\\
         & DRFLR & \multicolumn{1}{r}{3.89} & \multicolumn{1}{r}{7.10} & \multicolumn{1}{r}{20.09}& \multicolumn{1}{r}{42.95} & \multicolumn{1}{r}{66.09} & \multicolumn{1}{r}{90.29}\\
          \hline
         \multirow{3}{4em}{Arrhythmia}&$\eps$-FC&\multicolumn{1}{r}{3.06} &\multicolumn{1}{r}{88.81}&\multicolumn{1}{r}{378.23}& \multicolumn{1}{r}{-}&\multicolumn{1}{r}{-}&\multicolumn{1}{r}{-}\\
         & $\eps$-DRFC &\multicolumn{1}{r}{4.55} &\multicolumn{1}{r}{107.46}&\multicolumn{1}{r}{419.58}& \multicolumn{1}{r}{-}&\multicolumn{1}{r}{-}&\multicolumn{1}{r}{-}\\
          & HFC &\multicolumn{1}{r}{0.14} &\multicolumn{1}{r}{0.77}&\multicolumn{1}{r}{1.38}& \multicolumn{1}{r}{-}&\multicolumn{1}{r}{-}&\multicolumn{1}{r}{-}\\
        & HDRFC &\multicolumn{1}{r}{0.16} &\multicolumn{1}{r}{0.68}&\multicolumn{1}{r}{1.73}& \multicolumn{1}{r}{-}&\multicolumn{1}{r}{-}&\multicolumn{1}{r}{-}\\
         & DOB+ & \multicolumn{1}{r}{0.11} & \multicolumn{1}{r}{0.66} & \multicolumn{1}{r}{1.20}& \multicolumn{1}{r}{-} & \multicolumn{1}{r}{-} & \multicolumn{1}{r}{-}\\
          & DRFLR & \multicolumn{1}{r}{4.64} & \multicolumn{1}{r}{10.40} & \multicolumn{1}{r}{24.26}& \multicolumn{1}{r}{-} & \multicolumn{1}{r}{-} & \multicolumn{1}{r}{-}\\
          \hline
         \multirow{3}{4em}{Synthetic}&$\eps$-FC&\multicolumn{1}{r}{1.02} &\multicolumn{1}{r}{2.67}&\multicolumn{1}{r}{15.31}& \multicolumn{1}{r}{72.58}&\multicolumn{1}{r}{299.32}&\multicolumn{1}{r}{572.49}\\
         & $\eps$-DRFC &\multicolumn{1}{r}{1.47} &\multicolumn{1}{r}{3.54}&\multicolumn{1}{r}{21.52}& \multicolumn{1}{r}{70.33}&\multicolumn{1}{r}{309.34}&\multicolumn{1}{r}{593.18}\\
      & HFC &\multicolumn{1}{r}{0.01} &\multicolumn{1}{r}{0.01}&\multicolumn{1}{r}{0.01}& \multicolumn{1}{r}{0.02}&\multicolumn{1}{r}{0.05}&\multicolumn{1}{r}{0.07}\\
        & HDRFC &\multicolumn{1}{r}{0.12} &\multicolumn{1}{r}{0.01}&\multicolumn{1}{r}{0.02}& \multicolumn{1}{r}{0.03}&\multicolumn{1}{r}{0.05}&\multicolumn{1}{r}{0.08}\\
         & DOB+ & \multicolumn{1}{r}{0.09} & \multicolumn{1}{r}{0.16} & \multicolumn{1}{r}{0.19}& \multicolumn{1}{r}{0.29} & \multicolumn{1}{r}{0.38} & \multicolumn{1}{r}{0.59}\\
      & DRFLR & \multicolumn{1}{r}{4.00} & \multicolumn{1}{r}{8.51} & \multicolumn{1}{r}{20.83}& \multicolumn{1}{r}{43.82} & \multicolumn{1}{r}{68.56} & \multicolumn{1}{r}{95.54}\\
         \hline
         
    \end{tabular}
    \caption{Running time (in seconds) of different methods. The Arrhythmia dataset only contains 452 examples, hence we only examine its performance up to $N= 250$.}
    \label{tab:time}
\end{table}

% \begin{table}[]
%     \centering
%     \begin{tabular}{l|l|r|r|r|r|r|r|}
%          Dataset& Classifier& \multicolumn{6}{c}{N} \\
%          \hline
%          & 1& \multicolumn{1}{c}{50}&\multicolumn{1}{c}{100} &\multicolumn{1}{c}{250} &\multicolumn{1}{c}{500} &\multicolumn{1}{c}{750} &\multicolumn{1}{c}{1000}
%     \end{tabular}
%     \caption{Caption}
%     \label{tab:my_label}
% \end{table}

\subsection{Solution time.} \label{subsec:soltime}
We now report the running time of different methods on 5 datasets  (Adult, Drug, COMPAS, Arrhythmia, and Synthetic) with the sample size varying from $50$ to $1000$. We set the unfairness controlling parameters $\eta=0.1$ for $\eps$-FC and $\eps$-DRFC, $\zeta=1.1$ for HFC and HDRFC, $\eta_f=\min\{\wh p_{01}, \wh p_{11}\}/2$ for DRFLR, Wasserstein radius $\rho=0.05$ for all distributionally robust models, and assume all samples are correctly labeled. All results are averaged over 10 independent trials.

% \begin{figure}[!h]
%     \centering
%     \begin{subfigure}[b]{0.45\textwidth}
%     \centering
%     \includegraphics[width=\textwidth]{time.pdf}
%     \caption{DRFSVM}
%     \end{subfigure}
%     \begin{subfigure}[b]{0.45\textwidth}
%     \centering
%     \includegraphics[width=\textwidth]{time2.pdf}
%     \caption{FSVM}
%     \end{subfigure}
%     \begin{subfigure}[b]{0.45\textwidth}
%     \centering
%     \includegraphics[width=\textwidth]{dob.pdf}
%     \caption{DOB+}
%     \end{subfigure}
%         \begin{subfigure}[b]{0.45\textwidth}
%     \centering
%     \includegraphics[width=\textwidth]{svm.pdf}
%     \caption{SVM}
%     \end{subfigure}
%     \caption{Running time of four classifiers}
%     \label{fig:time}
% \end{figure}

Table \ref{tab:time} suggests that the $\eps$-DRFC model is applicable to moderate-size problems. However, it encounters computational difficulties at large sample sizes, where it takes more than 10 minutes to solve any datasets with sample size greater than 1000. The DRFLR model involves solving an exponential cone program, which is less efficient compared to the linear-program-based methods HFC and HDRFC, and the gradient-descent-based method DOB+.
The sample size is the factor that most affects the running time, because the number of variables and constraints are proportional to the sample sizes. Compared with the $\eps$-DRFC model and DRFLR model, the HDRFC and the DOB+ methods are more efficient across all datasets. For all sample sizes, these methods can be solved in one second. Therefore, this result suggests that the HDRFC model is more suitable for large instances.

% Because the datasets have different feature size $d$, we also observe that the running times of high-dimension datasets (e.g., Arrhythmia) are generally greater than that of low-dimension datasets (e.g., the synthetic dataset).

%\yijie{To this snd, we conclude that the DRFSVM model better encourages fair outcomes with negligible accuracy loss in the out-of-sample test. Therefore, the DRFSVM model is more appropriate to moderate-size problems where fair decision-making is crucial. However, it suffers from a scalability issue as the exact fairness constraint inevitably leads to a conic mixed binary program, which is non-convex. Thus, the DOB+ method, which only involves solving a convex optimization problem, is more suitable for the case where we have tens of thousands of data points.}

%%%%%%%%%%%%%%%%%%%%%%%
\section{Concluding remarks}
\label{sec:remark}
In this paper, we developed a new principled approach to fair classification by incorporating the equality of opportunity criterion as a constraint and robustifying the resulting optimization problem using the framework of Wasserstein min-max learning. We utilize the type-$\infty$ Wasserstein ambiguity set, which generally enables a more scalable conic programming reformulation while providing the same statistical performance guarantees as the models based on the type-1 Wasserstein ambiguity sets.  In addition, our proposed model can handle problem instances with noisy, adversarial sensitive attributes and labels.

Since the original problem cannot be reformulated exactly, we propose a conservative approximation. We remark that this conservative approximation is amenable to a mixed binary linear programming reformulation. Moreover, this approximation, for the first time, enables decision makers to bound the EO unfairness measure explicitly. However,  experimental results indicate that the reformulation is not as efficiently solvable as plain-vanilla models such as SVM and logistic regression. To address this issue, we further approximate both the objective function and the unfairness measure using the hinge loss function to obtain a convex model. We find that the hinge-loss-based distributionally robust fairness-aware model plays favorably compared to the state-of-the-art method DOB+ and DRFLR in the numerical experiments. In summary, we propose a tight conservative fairness-aware classifier for moderate-size problems and an efficient, high-quality fairness-aware classifier for large instances.

% In this paper, we delineated for the first time the relationship between the hinge loss objective and the misclassification probability of the resulting classifier. Our theoretical result implies that the proposed fair classification model constitutes a conservative approximation of the actual model which seeks a classifier that minimizes the misclassification probability. 
% We remark that this exact model is also amenable to a mixed-binary linear programming reformulation. However, preliminary experimental results indicate that the reformulation is not as efficiently solvable as the plain-vanilla model that minimizes the hinge loss objective. Thus, in the future, we plan to develop fast, tailored algorithms for the reformulation model. 

\section*{Acknowledgements}
This research was supported by the National Science Foundation grant no.~$1752125$.

\bibliographystyle{abbrv}
\bibliography{references}

%%%%%%%%%%%%%%%%%%%%%%%%%%%%%%%%%%%%
\newpage
\appendix

% \section{Appendix}

\section{Auxiliary Results and Proofs}
\begin{lemma}[Compactness] \label{lemma:compact}
    The set $\mbb B(\Pnom)$ defined in~\eqref{eq:B-def} is weakly compact and convex. More specifically, there exists a convex, compact set $\mbb X \in \mc X$ defined as
    \[
        \mbb X = \mathrm{ConvexHull}\big( \{x \in \mc X: \| x - \wh x_i \| \le \rho \}_{i=1}^N \big)
    \]
    such that $\QQ(\mbb X \times \mc A \times \mc Y) = 1$ for any $\QQ \in \mbb B(\Pnom)$.
\end{lemma}
\begin{proof}[Proof of Lemma~\ref{lemma:compact}]
    Because the $\Pnom$ is an empirical measure, the ambiguity set $\mbb B(\Pnom)$ can be represented as
	    \begin{align}
	    &\mbb B(\Pnom) = \notag \\
	    &\left\{
	        \QQ \in \mc M(\mc X \times \mc A \times \mc Y): 
	        \begin{array}{l}
	            \exists \pi_i \in \mc M(\mc X \times \mc A \times \mc Y) ~\forall i \in [N] \text{ such that :} \\
	            \QQ = N^{-1} \sum_{i \in [N]} \pi_i, \\
	            \| x_i - \wh x_i\| + \kappa_{\mc A} | a_i - \wh a_i| + \kappa_{\mc Y} | y_i - \wh y_i| \leq \rho \quad \forall (x_i, a_i, y_i) \in \mathrm{supp}(\pi_i) \quad \forall i \in [N], \\
	            N^{-1} \sum_{i\in[N]} \pi_i(A = a, Y = y) = \wh p_{ay} \quad \forall (a, y) \in \mc A \times \mc Y
	        \end{array}
	    \right\}, \label{eq:B-refor}
	    \end{align}
	    where $\mathrm{supp}(\pi_i)$ denotes the support of the probability measure $\pi_i$ \cite[Page~441]{ref:aliprantis06hitchhiker}. 
	    Pick any arbitrary $\QQ^0$ and $\QQ^1$ from $\mbb B(\Pnom)$. Associated with $\QQ^j$, $j \in \{0, 1\}$ is a collection of conditional probability measures $\{\pi_i^j\} \in \mc M(\mc X \times \mc A \times \mc Y)^N$ satisfying
	    \[
	        \left\{
	            \begin{array}{ll}
	                \QQ^j = N^{-1} \sum_{i \in [N]} \pi_i^j,\\
	                \| x_i - \wh x_i\| + \kappa_{\mc A} | a_i - \wh a_i| + \kappa_{\mc Y} | y_i - \wh y_i| \leq \rho \quad &\forall (x_i, a_i, y_i) \in \mathrm{supp}(\pi_i^j) \quad \forall i \in [N], \\
	                N^{-1} \sum_{i\in[N]} \pi_i^j(A = a, Y = y) = \wh p_{ay} & \forall (a, y) \in \mc A \times \mc Y.
	            \end{array}
	        \right.
	    \]
	    Consider any convex combination $\QQ^\lambda = \lambda \QQ^1 + (1-\lambda) \QQ^0$ for $\lambda \in (0, 1)$. It is easy to verify that the measure $\pi_i^\lambda = \lambda \pi_i^1 + (1- \lambda) \pi_i^0$ for any $i \in [N]$ satisfies
	    \[
	        \left\{
	            \begin{array}{ll}
	                \QQ^\lambda = N^{-1} \sum_{i \in [N]} \pi_i^\lambda, \\
	                \| x_i - \wh x_i\| + \kappa_{\mc A} | a_i - \wh a_i| + \kappa_{\mc Y} | y_i - \wh y_i| \leq \rho \quad &\forall (x_i, a_i, y_i) \in \mathrm{supp}(\pi_i^\lambda) \quad  \forall i \in [N], \\
	                N^{-1} \sum_{i\in[N]} \pi_i^\lambda(A = a, Y = y) = \wh p_{ay} & \forall (a, y) \in \mc A \times \mc Y,
	            \end{array}
	        \right.
	    \]
	    where the middle constraint is satisfied by noticing that $\mathrm{supp}(\pi_i^\lambda) = \mathrm{supp}(\pi_i^0) \cup \mathrm{supp}(\pi_i^1)$.
	    This observation implies that $\QQ^\lambda \in \mbb B(\Pnom)$. 

    Notice that for any feasible measure $\pi_i$, we have
    \[
        \mathrm{supp}(\pi_i) \subseteq \{x \in \mc X: \| x - \wh x_i\| \le \rho \} \times \mc A \times \mc Y,
    \]
    and as a consequence, we have
    \[
        \mathrm{supp}(\QQ) \subseteq \displaystyle \bigcup_{i \in [N]} \{x \in \mc X: \| x - \wh x_i\| \le \rho \} \times \mc A \times \mc Y.
    \]  
    By definition of $\mbb X$, we have $\bigcup_{i \in [N]} \{x \in \mc X: \| x - \wh x_i\| \le \rho \} \subseteq \mbb X$. Because $\mbb X$ is a compact set, the weakly compactness of $\mbb B(\Pnom)$ follows from Prohorov's theorem. This completes the proof.
\end{proof}

The result of Lemma~\ref{lemma:compact} also extends to the ambiguity set $\mc B_\gamma(\Pnom)$ defined as in~\eqref{eq:modified_ambi}.
\begin{corollary}[Compactness] \label{corollary:compact2}
    For any $\gamma \in [0, 1]$, the set $\mc B_\gamma(\Pnom)$ defined in~\eqref{eq:modified_ambi} is weakly compact and convex. More specifically, there exists a compact set $\mbb X \in \mc X$ defined as
    \[
        \mbb X = \mathrm{ConvexHull}\big( \{x \in \mc X: \| x - \wh x_i \| \le \rho \}_{i=1}^N \big)
    \]
    such that $\QQ(\mbb X \times \mc A \times \mc Y) = 1$ for any $\QQ \in \mc B_\gamma(\Pnom)$.
\end{corollary}
The proof of Corollary~\ref{corollary:compact2} follows a similar line of argument as the proof of Lemma~\ref{lemma:compact} by noticing that $\sum_{i\in[N]} \pi_i(A = \wh a_i, Y = \wh y_i) \ge (1-\gamma) N$ is a convex constraint for $\pi_i$. 
\begin{lemma}[Reformulation of $\mbb B(\Pnom)$] \label{lemma:B-refor}
    The set $\mbb B(\Pnom)$ defined in~\eqref{eq:B-def} can be equivalently written as
    \[
    \mbb B(\Pnom) = \left\{
            \QQ \in \mc M(\mc X \times \mc A \times \mc Y) : \begin{array}{l}
                \exists \pi_i \in \mc M(\mc X \times \mc A \times \mc Y) \quad \forall i \in [N] \text{ such that :}\\
                \QQ = N^{-1} \sum_{i\in[N]}\pi_i \\
                \Wass_\infty(\pi_i, \delta_{(\wh x_i, \wh a_i, \wh y_i)}) \le \rho \\
                \QQ(A = a, Y = y) = \wh p_{ay} \quad \forall (a, y) \in \mc A \times \mc Y
            \end{array}
        \right\}.
    \]
\end{lemma}
\begin{proof}[Proof of Lemma~\ref{lemma:B-refor}]
    Notice that the condition
    \[
    \| x_i - \wh x_i\| + \kappa_{\mc A} | a_i - \wh a_i| + \kappa_{\mc Y} | y_i - \wh y_i| \leq \rho \quad \forall (x_i, a_i, y_i) \in \mathrm{supp}(\pi_i) \quad \forall i \in [N]
    \]
    is equivalent to the condition
    \[
        \Wass_{\infty}(\pi_i, \delta_{(\wh x_i, \wh a_i, \wh y_i)}) \le \rho \qquad \forall i \in [N]
    \]
    by the definition of the type-$\infty$ Wasserstein distance $\Wass_{\infty}$. Replacing latter condition into~\eqref{eq:B-refor} finishes the proof.
\end{proof}

\subsection{Proof of Section~\ref{sec:refor-prob}}\label{subsec:proof3}
\begin{proof}[Proof of Proposition~\ref{prop:eps_convergence}]
We first define the events $S=\{x: w^\top x + b \geq 0 \}$ and $S_\eps=\{x: w^\top x + b > -\eps \}$. 
As~$\eps$ tends to zero, we have
\[S=\lim_{\eps \rightarrow 0} S_\eps.
\]Moreover, we see that $S_\eps$ is non-increasing as $\eps \rightarrow 0$. Thus, for any distribution $\QQ$, we have \cite[Lemma~5]{grimmett2020probability}
\[ \QQ(S)=\lim_{\eps \rightarrow 0} \QQ(S_\eps).
\] 
Plugging this result into the $\eps$-unfairness measure yields
\begin{align*}
    \lim_{\eps \rightarrow 0} \mathds U_\eps(w, b, \QQ) &= \max \left\{ \begin{array}{l}
     \lim_{\eps \rightarrow 0} \QQ_{01}( w^\top X + b > -\eps) - \QQ_{11}(w^\top X + b \ge 0), \\
     \lim_{\eps \rightarrow 0} \QQ_{11}( w^\top X + b > -\eps) - \QQ_{01}(w^\top X + b \ge 0) 
    \end{array}
    \right\}\\
    &= \max \left\{ \begin{array}{l}
      \QQ_{01}( w^\top X + b \geq 0) - \QQ_{11}(w^\top X + b \ge 0), \\
      \QQ_{11}( w^\top X + b \geq 0) - \QQ_{01}(w^\top X + b \ge 0) 
    \end{array}
    \right\}\\
    &=\mathds U(w, b, \QQ).
\end{align*}
The proof of the objective function is the same and thus omitted.
\end{proof}

\subsection{Proof of Section~\ref{sec:cvx}}\label{subsec:proof4}
    \begin{proof}[Proof of Lemma~\ref{prop:U_h_lowerbound}]
    By definition, we find
    \begin{align*}
    \mathds H(w, b, \QQ)&=\max  \left\{ \begin{array}{l}
    \EE_{\QQ_{01}}\left[\max\{0, 1+w^\top X + b \}\right] + \EE_{\QQ_{11}}\left[\max\{0,1-w^\top X - b \}\right] -1 , \\
    \EE_{\QQ_{11}}\left[\max\{0, 1+w^\top X + b \}\right] + \EE_{\QQ_{01}}\left[\max\{0,1-w^\top X - b \}\right] -1
    \end{array}
    \right\}\\
    &\geq \max  \left\{ \begin{array}{l}
    \EE_{\QQ_{01}}\left[ 1+w^\top X + b \right] + \EE_{\QQ_{11}}\left[1-w^\top X - b \right] -1 , \\
    \EE_{\QQ_{11}}\left[1+w^\top X + b \right] + \EE_{\QQ_{01}}\left[1-w^\top X - b\right] -1
    \end{array}
    \right\}\\
    &\geq \max  \left\{ \begin{array}{l}
    \EE_{\QQ_{01}}\left[w^\top X + b \right] - \EE_{\QQ_{11}}\left[w^\top X + b \right] +1 , \\
    \EE_{\QQ_{11}}\left[w^\top X + b \right] - \EE_{\QQ_{01}}\left[w^\top X + b\right] +1
    \end{array}
    \right\}\\
    &=1 + \left|\EE_{\QQ_{01}}\left[w^\top X + b \right] - \EE_{\QQ_{11}}\left[w^\top X + b \right]\right|.
    \end{align*}
Thus, we have $\mathds H(w, b, \QQ) \ge 1$. Furthermore, it can be verified that $\EE_{\QQ_{01}}[w^\top X + b]=\EE_{\QQ_{11}}[w^\top X + b]$ is a necessary condition for $\mathds H(w, b, \QQ) = 1$, which completes the proof.
    \end{proof}

%     \begin{proof}[Proof of Proposition~\ref{prop:cvx_conservative}]
% One can verify that for any $(w,b) \in \R^{d+1}$ and $\QQ \in \mbb B(\Pnom)$
% \begin{align*}
%     \QQ(Y(w^\top X + b) \leq 0) &\leq \Sup{\QQ' \in \mbb B(\Pnom)}~\QQ'(Y(w^\top X + b) \leq 0) \\
%     &= \Sup{\QQ' \in \mbb B(\Pnom)}~\EE_{\QQ'}[\mathbbm{I}(w^\top X + b \leq 0)]\\
%     &\leq \Sup{\QQ' \in \mbb B(\Pnom)}~\EE_{\QQ'}[\max\{0, 1 - Y(w^\top X + b)\}].
% \end{align*}
% Thus, by plugging in the optimal solution $(w\opt, b\opt)$, we have 
% \begin{align*}
%   \QQ(Y((w\opt)^\top X + b\opt) \leq 0) &\leq v\opt \qquad \forall \QQ \in \mbb B(\Pnom),
% \end{align*}
% which completes the proof.
%     \end{proof}

\section{Marginal Constraints and Finite-sample Guarantees}\label{sec:marginal}

In this section, we illustrate how to handle ambiguity in the marginal distributions and obtain a generalized model with finite sample guarantees. To relax the marginal constraints in the ambiguity set~\eqref{eq:B-def}, we first construct four ambiguity sets around the empirical conditional distributions as
    \be \label{eq:B-def-marginal}
        \mbb B_{ay}(\Pnom_{ay}) = \left\{
            \QQ_{ay} \in \mc M(\mc X ) : \begin{array}{l}
                \Wass_\infty(\QQ_{ay}, \Pnom_{ay}) \le \rho_{ay} \\
            \end{array}
        \right\},
    \ee
    where $\Pnom_{ay} \Let \frac{1}{|\mc I_{ay}|}\sum_{i \in \mc I_{ay}} \delta_{\wh x_i}$ is the empirical conditional distribution. Notice that the notation $\Wass_\infty$ in the above definition of the ambiguity set is used with a slight abuse of notation: $\Wass_\infty$ in this case is a distance on $\mc M(\mc X)$, and it is no longer a distance on the \textit{joint} space $\mc X \times \mc A \times \mc Y$ as is used in the main paper.
    Next, we construct an ambiguity set for the marginal distribution $p\in \R^4$ based on the $\chi^2$-divergence by
    \be \label{eq:chi-set}
    \Delta=\left\{p \in \R_{++}^4:~ 1^\top p=1,~\sum_{(a,y)\in \mc A \times \mc Y} (p_{ay}-\wh p_{ay})^2/p_{ay} \leq \delta_p \right\}
    \ee
    % One can also employ other probability distance such as KL-divergence; however, it can be verified that the reformulation involves exponential cones, which are computationally difficult.
    Combining the two ambiguity sets, we define the following generalized ambiguity set:
        \[
    \mbb B_g(\Pnom) = \left\{
            \QQ \in \mc M(\mc X \times \mc A \times \mc Y) : \begin{array}{l}
                \exists \QQ_{ay} \in \mc M(\mathbb X) \quad \forall (a, y) \in \mc A \times \mc Y,~p \in \R^4 \text{ such that :}\\
                \QQ( \mathbb{X} \times \{a\} \times \{y\}) = p_{ay} \QQ_{ay} (\mathbb{X})   \quad \text{for all } \mathbb{X} \text{ measurable},\;\forall (a, y) \in \mc A \times \mc Y, \\
                p \in \Delta, \\
                \QQ_{ay} \in \mbb B_{ay}(\Pnom_{ay}) \quad \forall (a, y) \in \mc A \times \mc Y
            \end{array}
        \right\}.
    \]
    It can be verified that when $\Delta$ contains only the empirical marginal distribution $\wh p$, the generalized ambiguity set reduces to the ambiguity set $\mbb B(\Pnom)$ defined in~\eqref{eq:B-def}. As the conditional probability measures are supported on $\mc M(\mathbbm X )$, with a slight abuse of notation, we use the following ground metric:
    \be\label{eq:cost-marginal}
    c\big( x',  x \big) = \| x - x'\|.
    \ee
    Now, consider the $\eps$-DRFC problem with the generalized ambiguity set
    \be\label{eq:generalized-dro-prob2}
    \begin{array}{cl}
        \min & \Sup{\QQ \in \mbb B_g(\Pnom)}~\QQ( Y(w^\top X + b) < \eps) \\
        \st & w \in \R^d,~b \in \R, \\
        & \Sup{\QQ \in \mbb B_g(\Pnom)}~\mathds U_\eps(w, b, \QQ) \le \eta.
    \end{array}
    \ee
    Notice that the above optimization is similar to problem~\eqref{eq:dro-prob2}: the only difference is that the ambiguity set is now $\mbb B_g(\Pnom)$. The next theorem asserts that the generalized model is equivalent to a mixed binary second-order cone program.

\begin{theorem}[Generalized $\eps$-DRFC reformulation]\label{thm:chi-refor} 
        Suppose that the ground metric is prescribed using~\eqref{eq:cost-marginal}, then 
        the generalized $\eps$-DRFC problem~\eqref{eq:generalized-dro-prob2} is equivalent to the  mixed binary second-order cone program
    \begin{align}\label{eq:chi-refor}
    \begin{array}{cll}
    \min & \delta_p \zeta - \theta - 2 \wh p^\top r + 2 \zeta 1^\top\wh p  \\
    \st & w\in \R^d,~b \in \R,~\zeta \in \R_+,~\theta \in \R,~r \in \R^4,~s \in \R^4,\\
    &t \in \{0,1\}^N,~\lambda^{0} \in \{0, 1\}^N,~\lambda^{1} \in \{0, 1\}^N,\\
    &s_{ay} + \theta \leq \zeta, \sqrt{4r_{ay}^2 + (s_{ay} + \theta)^2} \leq 2\zeta -s_{ay}-\theta & \forall (a,y) \in \mc A \times \mc Y,\\
    &\ds \frac{1}{|\mc I_{ay}|} \sum_{i \in \mc I_{ay}} t_i  \leq s_{ay}& \forall (a,y) \in \mc A \times \mc Y,\\
    &-\wh y_i (w^\top \wh x_i + b)  + \rho_{\wh a_i \wh y_i} \| w\|_*  \le M t_i - \eps &\forall i \in [N],\\
    &\hspace{-2mm}\left.
            \begin{array}{l}
            \ds \ds \frac{1}{| \mc I_{a1} |} \sum_{i \in \mc I_{a1}} \lambda_i^a + \ds \frac{1}{| \mc I_{a'1} |} \sum_{i \in \mc I_{a'1}} \lambda_i^a -1  \le \eta,  \\
            w^\top \wh x_i + \rho_{a1} \| w \|_* + b + \eps \le M \lambda_i^a \quad \forall i \in \mc I_{a1}, \\
            -w^\top \wh x_i + \rho_{a'1} \| w \|_* - b  \le M \lambda_i^a \qquad \forall i \in \mc I_{a'1}
            \end{array} 
            \right\}  & \forall (a, a') \in \{(0, 1), (1, 0)\},
    \end{array}
    \end{align}
    where $M$ is the big-M parameter.
    % \viet{a dot is needed}
\end{theorem}
For the remainder of this section, we will provide the proof for Theorem~\ref{thm:chi-refor}. This proof relies on the following Lemma.

\begin{lemma}[$\chi^2$-divergence reformulation, Theorem $4.1$ in \cite{ben2013robust}]\label{lem:chi-refor}
    Define $\Delta$ as in \eqref{eq:chi-set}. For any $m \in \mbb N_+$ and $\varphi \in \R^m$, both the optimal value and a maximizer of the worst-case expectation problem $\sup_{p \in \Delta} \varphi^\top p $ can be obtained by solving the second-order cone program
    \[
    \begin{array}{cll}
    \sup & \varphi^\top p \\
    \st & p \in \R_+^m,~q \in \R^m_+,~1^\top p =1,~1^\top q \leq \delta_p, \\
    & \sqrt{(p_j-\wh p_j)^2+ \frac{1}{4} p_j^2+q_j^2 } \leq \frac{1}{2} p_j +q_j & \forall j \in [m].
    \end{array}
    \]
The optimal value can also be computed by solving the dual problem:
    \[
    \begin{array}{cl}
    \inf & \delta_p \zeta - \theta - 2 \wh p^\top r + 2 \zeta 1^\top \wh p  \\
    \st & \zeta \in \R_+,~\theta \in \R,~r \in \R^m,~s \in \R^m,\\
    & \varphi_j \leq s_j,~s_j + \theta \leq \zeta,~\sqrt{4r_j^2 + (s_j + \theta)^2} \leq 2\zeta -s_j-\theta \qquad \forall j \in [m].
    \end{array}
    \]
\end{lemma}

We are now ready to present the proof.
\begin{proof}[Proof of Theorem~\ref{thm:chi-refor}]
Observe that problem~\eqref{eq:generalized-dro-prob2} can be equivalently written as
    \[
        \begin{array}{cl}
            \min & \Sup{p \in \Delta}~\ds \sum_{a \in \mc A, y \in \mc Y}\Sup{\QQ_{ay} \in \mbb B_{ay}(\Pnom_{ay})}  p_{ay}  \QQ_{ay}( Y(w^\top X + b) < \eps) \\
            \st & w \in \R^d,~b \in \R, \\
            & \Sup{\QQ \in \mbb B_g(\Pnom)}~\mathds U_\eps(w, b, \QQ) \le \eta.
        \end{array}
    \]

We first derive the reformulation of the objective function. By the definition of $\Delta$ and the result of Lemma~\ref{lem:chi-refor}, we have
\begin{align*}
    &\Sup{p \in \Delta}\sum_{a \in \mc A, y \in \mc Y}\Sup{\QQ \in \mbb B_{ay}(\Pnom_{ay})}  p_{ay}  \QQ_{ay}( Y(w^\top X + b) < \eps)\\
    =&\min \left\{\delta_p \zeta - \theta - 2 \wh p^\top r + 2 \zeta 1^\top \wh p :
    \begin{array}{ll}
     \zeta \in \R_+,~\theta \in \R,~r \in \R^4,~s \in \R^4,\\
     \Sup{\QQ_{ay} \in \mbb B_{ay}(\Pnom_{ay})}   \QQ_{ay}( Y(w^\top X + b) < \eps) \leq s_{ay}& \forall (a,y) \in \mc A \times \mc Y,\\
     s_{ay} + \theta \leq \zeta, \sqrt{4r_{ay}^2 + (s_{ay} + \theta)^2} \leq 2\zeta -s_{ay}-\theta & \forall (a,y) \in \mc A \times \mc Y
    \end{array} \right\}\\
    =&\min \left\{\delta_p \zeta - \theta - 2 \wh p^\top r + 2 \zeta 1^\top \wh p :
    \begin{array}{ll}
     \zeta \in \R_+,~\theta \in \R,~r \in \R^4,~s \in \R^4,~t \in \{0,1\}^N,\\
          s_{ay} + \theta \leq \zeta, \sqrt{4r_{ay}^2 + (s_{ay} + \theta)^2} \leq 2\zeta -s_{ay}-\theta & \forall (a,y) \in \mc A \times \mc Y,\\
      \ds \frac{1}{|\mc I_{ay}|} \sum_{i \in \mc I_{ay}} t_i  \leq s_{ay}& \forall (a,y) \in \mc A \times \mc Y,\\
      -\wh y_i (w^\top \wh x_i + b)  + \rho_{\wh a_i \wh y_i} \| w\|_*  \le M t_i - \eps &\forall i \in [N]
    \end{array} \right\}.
\end{align*}

For the constraint, fixing any pair $(a, a') \in \{(0, 1), (1, 0)\}$, we have
\begin{align*}
    &\Sup{\QQ \in \mbb B_g(\Pnom)}~\QQ_{a1}( w^\top X + b > -\eps) - \QQ_{a'1}(w^\top X + b \ge 0 )\\
    =&\Sup{\QQ_{a1} \in \mbb B_{a1}(\Pnom_{a1}),\QQ_{a'1} \in \mbb B_{a'1}(\Pnom_{a'1})}~\QQ_{a1}( w^\top X + b > -\eps) - \QQ_{a'1}(w^\top X + b \ge 0 ) \\
    =&\Sup{\QQ_{a1} \in \mbb B_{a1}(\Pnom_{a1})}~\QQ_{a1}( w^\top X + b > -\eps) - \Inf{\QQ_{a'1} \in \mbb B_{a'1}(\Pnom_{a'1})} \QQ_{a'1}(w^\top X + b \ge 0 )\\
    =& \ds \frac{1}{| \mc I_{a1} |} \sum_{i \in \mc I_{a1}} \Sup{x_i: \| x_i - \wh x_i \| \le \rho_{a1}} \mbb I(w^\top x_i + b > -\eps) -  \ds \frac{1}{| \mc I_{a'1} |} \left(| \mc I_{a'1}| - \sum_{i \in \mc I_{a'1}} \Sup{x_i: \| x_i - \wh x_i \| \le \rho_{a'1}} \mbb I(w^\top x_i + b < 0) \right)  \\
    =& \ds \frac{1}{| \mc I_{a1} |} \sum_{i \in \mc I_{a1}} \Sup{x_i: \| x_i - \wh x_i \| \le \rho_{a1}} \mbb I(w^\top x_i + b > -\eps) + \ds \frac{1}{| \mc I_{a'1} |} \sum_{i \in \mc I_{a'1}} \Sup{x_i: \| x_i - \wh x_i \| \le \rho_{a'1}} \mbb I(w^\top x_i + b < 0) -1  \\
    =& \left\{
            \begin{array}{cll}
                \min & \ds \frac{1}{| \mc I_{a1} |} \sum_{i \in \mc I_{a1}} \lambda_i^a + \ds \frac{1}{| \mc I_{a'1} |} \sum_{i \in \mc I_{a'1}} \lambda_i^a -1  \\
                \st & \lambda^a \in \{0, 1\}^N \\
                & w^\top \wh x_i + \rho_{a1} \| w \|_* + b + \eps \le M \lambda_i^a & \forall i \in \mc I_{a1} \\
                & -w^\top \wh x_i + \rho_{a'1} \| w \|_* - b  \le M \lambda_i^a & \forall i \in \mc I_{a'1},
            \end{array}
        \right.
\end{align*}
   where the last equality follows from applying Lemma~\ref{lemma:indicator} twice and noticing that $\mc I_{a1} \cap \mc I_{a'1} = \emptyset$. Setting the optimal value of the above minimization problem to be less than $\eta$ completes the proof.
\end{proof}

We now investigate the finite-sample guarantee of this generalized problem. 
% \yijie{mention one for each and use max for all}
% \viet{At a high level, there should a true joint distribution $\PP_0$ of $(X, A, Y)$. There is (i) conditional distribution, (ii) marginal distribution}
% \viet{should this one be the conditional distribution of $X$ given $(A, Y) = (a, y)$ for all $(a, y)$? --- probability distribution}
\begin{theorem}[Finite-sample guarantee]\label{thm:oos-guarantee} %\viet{conditions lacking}
     Let $\PP\opt$ denotes the true joint distribution of $(X, A, Y)$. Assume that for all $(a,y) \in \mc A \times \mc Y$, the conditional distribution $\PP\opt_{ay}$ of $X \in \R^d$, with $d\geq 2$, has a density function $\tau_{ay}: \bar{\mathbb X} \rightarrow [0,\infty)$, where $\bar{\mathbb X} \subseteq \mathbb X \subseteq \R^d$ is an open, connected, and bounded set with a Lipschitz boundary, and there exists a constant $\lambda \geq 1$ such that $1/\lambda \leq \tau_{ay}(x) \leq \lambda$ for all $x \in \bar{\mathbb X}$ and $(a,y) \in \mc A \times \mc Y$. 
    Let $v\opt$ be the optimal value of \eqref{eq:generalized-dro-prob2}, and $(w\opt,b\opt) \in \R^{d+1}$ be the corresponding optimal solution. Then for any $\alpha>2$, setting $\rho_{ay}=C_1\frac{\log(| \mc I_{ay} |)^{1/d}}{| \mc I_{ay} |^{1/d}}$ and $\delta_p > \frac{k}{N}$ implies
    % \viet{what is $\PP\opt$?}\yijie{introduce $\PP\opt$ again}
    % \viet{dependence of the parameters on $\alpha$ not clear}
    \[ \textup{Prob}\left(\PP\opt(Y(w^\top X + b) \geq 0) \leq v\opt\right)  \geq 1- C_2 e^{-\frac{1}{2}\left(N\delta_p-\sqrt{2kN\delta_p-k^2} \right)} -C_3 \sum_{(a,y) \in \mc A \times \mc Y}  | \mc I_{ay} |^{-\frac{\alpha}{2}}, 
    \]
    and
    \[ \textup{Prob}\left(\mathds U(w, b, \PP\opt) \leq \eta \right) \geq 1- C_2 e^{-\frac{1}{2}\left(N\delta_p-\sqrt{2kN\delta_p-k^2} \right)} -C_3 \sum_{(a,y) \in \mc A \times \mc Y}  | \mc I_{ay} |^{-\frac{\alpha}{2}},
    \]where $C_1$ is a constant which depends on the true distribution $\PP\opt$ and $\alpha$, $C_2$ is a constant which depends on the conditional distribution $\PP\opt$, and $C_3$ is an universal constant.
\end{theorem}

The proof of Theorem~\ref{thm:oos-guarantee} relies on the following theorem, which provides the concentration inequality for the type $\infty$-Wasserstein distance.
\begin{theorem}[$\infty$-Wasserstein concentration, Theorem $1.1$ in \cite{trillos2015rate}]\label{thm:infty-concentration}
     Assume that the probability distribution of $\xi \in \R^d$ has a density function $\tau: \bar \Xi \rightarrow [0,\infty)$, where $\bar \Xi \subseteq \Xi \subseteq \R^d$ is an open, connected, and bounded set with a Lipschitz boundary, and there exists a constant $\lambda \geq 1$ such that $1/\lambda \leq \tau(\xi) \leq \lambda$ for all $\xi \in \bar \Xi$. Then, for any fixed $\alpha>2$, 
     \[\textup{Prob}\left(\Wass_\infty(\PP\opt, \Pnom)>C \left\{\begin{array}{cc}
          \frac{\log(N)^{3/4}}{N^{1/2}},& \textup{if} \ d=1 \\
          \frac{\log(N)^{1/d}}{N^{1/d}},& \textup{if} \ d\geq 2
     \end{array}\right.\right) = \mathcal O(N^{-\frac{\alpha}{2}}),
     \]
     where $C$ is a constant which depends only on $\alpha$, $\bar \Xi$, and $\lambda$.
\end{theorem}

Equipped with Theorem~\ref{thm:infty-concentration}, we are now ready to show the proof of Theorem~\ref{thm:oos-guarantee}.
\begin{proof}[Proof of Theorem~\ref{thm:oos-guarantee}]
For any $(a,y) \in \mc A \times \mc Y$, $\wh p_{ay}$ is an estimator given by
\[ \wh p_{ay}= \frac{1}{N}\sum_{i=1}^N \mathbbm{1}_{(a, y)}(\wh a_i, \wh y_i).
\]
% \viet{what is $p_{ay}$?} 
Let $p_{ay}$ denotes the mass of the true marginal distribution on $(a,y)$. It can be verified that $N \sum_{(a,y)\in \mc A \times \mc Y} (p_{ay}-\wh p_{ay})^2/p_{ay}$ asymptotically converges to $\chi^2$ with degree $k=3$ \cite{boucheron2003concentration}. To obtain an explicit and concise result, we employ the $\chi^2$ test-statistic upper tail bound (Lemma 1 in \cite{ref:Laurent2000Adaptive}), which can be written as 
% \viet{$X$ overloaded}
\[\textup{Prob}\left(Z \geq k + 2 \sqrt{kt}+ 2 t\right) \leq e^{-t},
\]where $Z \sim \chi^2_k$ and $t \geq 0$. By letting  $N\delta_p=k+\sqrt{kt}+ 2t$ and solving for $t$, %performing simple algebraic transformations, 
we further obtain
% Thus, applying the $\chi^2$ test-statistic upper tail bound (Lemma 1 in \cite{ref:Laurent2000Adaptive}) and multiplying $N$ on both sides, we have
% \yijie{need revision. Use Laurent and cite the theorem number.}
% \[\textup{Prob}\left(N \sum_{(a,y)\in \mc A \times \mc Y} (p_j-\wh p_j)^2/p_j > N \delta_p \right) \leq 1- \gamma(\frac{3}{2},\frac{N\delta_p}{2})/\Gamma(\frac{3}{2})+o(1),
% \]
% \[\textup{Prob}\left(N \sum_{(a,y)\in \mc A \times \mc Y} (p_j-\wh p_j)^2/p_j > N \delta_p \right) \leq \left(\frac{N \delta_p}{k}e^{1-N\delta_{p}/k} \right)^{k/2}+o(1),
% \]
\[\textup{Prob}\left(N \sum_{(a,y)\in \mc A \times \mc Y} (p_{ay}-\wh p_{ay})^2/p_{ay} > N \delta_p \right) \leq  C_2 e^{-\frac{1}{2}\left(N\delta_p-\sqrt{2kN\delta_p-k^2} \right)},
\] where $C_2$ is a constant that depends on the underlying distribution. %, which captures the fact that the $N \sum_{(a,y)\in \mc A \times \mc Y} (p_{ay}-\wh p_{ay})^2/p_{ay}$ has. 
We assume the conditions of Theorem~\ref{thm:infty-concentration} hold and $d \geq 2$. Then there exists $C_3>0$ such that setting $\rho_{ay}=C_1\frac{\log(| \mc I_{ay} |)^{1/d}}{| \mc I_{ay} |^{1/d}}$ implies 
\[ \textup{Prob}\left(\Wass_\infty(\PP_{ay}\opt, \Pnom_{ay})> \rho \right) \leq C_3 | \mc I_{ay} |^{-\frac{\alpha}{2}} \qquad \forall (a,y) \in \mc A \times \mc Y.
\] 
By union bound, we further obtain
\[\textup{Prob}\left(\PP_{ay}\opt \in \mbb B_{ay}(\Pnom_{ay})\ \forall (a,y) \in \mc A \times \mc Y \ \textup{and} \ p \in \Delta\right) \geq 1- C_2 e^{-\frac{1}{2}\left(N\delta_p-\sqrt{2kN\delta_p-k^2} \right)} - \sum_{(a,y) \in \mc A \times \mc Y} C_3 | \mc I_{ay} |^{-\frac{\alpha}{2}}.
\]
This result immediately leads to the statements in the theorem.
\end{proof}

Theorem~\ref{thm:oos-guarantee} requires choosing the radius $\rho$ based on the constants $C_1$ and $C_2$, and the probabilistic guarantee also depends on the universal constant $C_3$. However, these parameters depend on the properties of the underlying distribution~$\PP\opt$, which are typically unknown to decision makers. Moreover, the generalized model introduces one more tuning parameter $\delta_p$, making the cross-validation procedure more demanding. Therefore, in practice, we adopt the simpler ambiguity set \eqref{eq:B-def}.  Nonetheless, Theorem 1 describes an explicit rate for decreasing the Wasserstein radius $\rho$, which provides useful insight to determine the parameter value in practice.

\section{Relationship between SVM and CVaR}\label{subsec:cvar}

In this section, we show that the Support Vector Machine (SVM) model is exactly the Conditional Value at Risk (CVaR) approximation of the misclassification probability minimization problem~\eqref{eq:probminimize}. SVM is a linear classifier obtained by determining the parameters $(w, b) \in \R^{d+1}$ that minimizes  the empirical expected \textit{hinge loss}:
        \begin{equation}\label{eq:SVM-true}
    \min_{w \in \R^d,b \in \R}~\EE_{\hat \PP}\left[\max \left\{0, 1-Y(w^\top X + b)\right\}\right].
\end{equation}
Observe that problem~\eqref{eq:probminimize} under the empirical distribution can be equivalently written as
\begin{equation} \label{eq:chancecons}
\begin{array}{cl}
         \inf & t\\
         \st & w \in \R^d,~b \in \R,~t \in \R,\\
         & \hat \PP \left( Y(w^\top X + b)> 0\right) \geq 1- t.
\end{array} 
\end{equation}
The optimization problem above can be regarded as an extension of chance constrained programs, with the quantile $t$ being a decision variable. However, the feasible set of such a chance constraint is non-convex, which makes optimization problematic in the face of large instances. A natural way to overcome this difficulty is to replace the chance constraint with a tractable approximation. To obtain an efficient model for large sample sizes, we require the approximation to be convex. Moreover, we also like the approximation to be conservative, i.e., if a solution is feasible in the approximated problem then it is also feasible to the original problem. If these two conditions hold, we refer to the approximation as a convex conservative approximation. 

To this end, the well-known CVaR can be used to derive such an approximation.  The core idea lies in the fact that the chance constraint in \eqref{eq:chancecons} can be written as a Value at Risk (VaR) constraint. This yields the equivalent reformulation
%\begin{equation} \label{eq:VaR}
\[
\begin{array}{cl}
         \inf & t\\
         \st & w \in \R^d,~b \in \R,~t \in \R,\\
         & \textup{VaR}_{(t,\hat \PP)} (-Y(w^\top X + b)) < 0,
\end{array}
\]
where $\textup{VaR}_{(t,\hat \PP)} (Z)) \Let \inf\{\tau \in \R: \hat \PP (Z \leq \tau) \geq 1- t \}$ can be interpreted as the $1-t$ quantile of $Z$. Recall that the CVaR of a random variable $Z$ is defined as
\[ \textup{CVaR}_{(t,\hat \PP)} (Z) \Let \inf\left\{\tau \in \R: \tau + \frac{1}{t}\EE_{\wh \PP}\left[\max\{0,Z-\tau \}\right]\right\}.
\]
It can be verified that $\textup{VaR}_{(t,\hat \PP)} (Z)$ is a minimizer of the right hand side problem~\cite{nemirovski2007convex}. Thus, the relation $\textup{VaR}_{(t,\hat \PP)} (Z)) \leq \textup{CVaR}_{(t,\hat \PP)} (Z)$ holds for any distribution, and we can simply replace the %$\textup{VaR}_{(t,\QQ)} (-Y(w^\top X + b))$ with $\textup{CVaR}_{(t,\QQ)} (-Y(w^\top X + b))$
VaR term using CVaR, and obtain the following convex conservative approximation to the chance constrained program \eqref{eq:chancecons}:

% As the relation $\textup{VaR}_{(t,\hat \PP)} (Z)) \leq \textup{CVaR}_{(t,\hat \PP)} (Z)$ holds for all distributions, we can simply replace the %$\textup{VaR}_{(t,\QQ)} (-Y(w^\top X + b))$ with $\textup{CVaR}_{(t,\QQ)} (-Y(w^\top X + b))$
% VaR term using CVaR, and obtain the following convex conservative approximation to the chance constraint program
\begin{equation} \label{eq:CVaR}
\begin{array}{cl}
         \inf & t\\
         \st & w \in \R^d,~b \in \R,~t \in \R,\\
         & \textup{CVaR}_{(t,\hat \PP)} (-Y(w^\top X + b)) < 0.
\end{array}
\end{equation}

Interestingly, we find that the CVaR approximation problem \eqref{eq:CVaR} is exactly equivalently to the SVM model~\eqref{eq:SVM-true}. That is, the optimal value of \eqref{eq:CVaR} coincides with the optimal value of \eqref{eq:SVM-true}, and the two problems yield the same  optimal classifiers. This result connects the well-known SVM model with the misclassification minimization problem from the perspective of conservative approximation.

\begin{theorem}[CVaR equivalence] \label{thm:cvar_eq}
    The CVaR approximation %of the misclassification probability minimization problem 
    \eqref{eq:CVaR} is equivalent to the SVM model~\eqref{eq:SVM-true}. 
\end{theorem}

\begin{proof}[Proof of Theorem~\ref{thm:cvar_eq}]
Let $t^\star$ be the optimal value of \eqref{eq:CVaR}. When $t^\star=0$, one can verify that this condition implies the dataset is linearly separable, and the optimal value of the SVM model~\eqref{eq:SVM-true} will also be zero. Now, without loss of generality, we assume $t^\star>0$. By the definition of CVaR, we have
\begin{align*}
    % t^\star &= \inf~\left\{t: \begin{array}{l}
    %      w \in \R^d, b \in \R, t \in \R_{++}  \\
    %      \QQ\left( Y(w^\top X + b) > 0\right) \geq 1-t 
    % \end{array} \right\}\\
    % &= \inf~\left\{t: \begin{array}{l}
    %      w \in \R^d, b \in \R, t \in \R_{++}  \\
    %      \textup{VaR}_{(t,\QQ)} (-Y(w^\top X + b)) < 0
    % \end{array} \right\}\\
    t^\star &= \inf~\left\{t: \begin{array}{l}
         w \in \R^d,~b \in \R,~t \in \R_{++},  \\
         \textup{CVaR}_{(t,\hat \PP)} (-Y(w^\top X + b)) < 0
    \end{array} \right\}\\
    &= \inf~\left\{t: \begin{array}{l}
         w \in \R^d,~b \in \R,~t \in \R_{++}, \\
         \inf_{\beta \in \R} -\beta + \frac{1}{t} \EE_{\hat \PP}\left[\max\{\beta-Y(w^\top X + b),0\} \right] < 0
    \end{array} \right\}\\
    &= \inf~\left\{t: \begin{array}{l}
         w \in \R^d,~b \in \R,~\beta \in \R,~t \in \R_{++}, \\
          -\beta + \frac{1}{t} \EE_{\hat \PP}\left[\max\{\beta-Y(w^\top X + b),0\} \right] < 0.\
    \end{array} \right\}.
\end{align*}
Notice that when $\beta \leq 0$, the constraint is always infeasible; hence, we can restrict $\beta$ to be positive without changing the feasible region, which yields
\begin{align*}
   t^\star  &= \inf~\left\{t: \begin{array}{l}
         w \in \R^d,~b \in \R,~\beta \in \R_{++},~t \in \R_{++} \\
         -\beta + \frac{1}{t} \EE_{\hat \PP}\left[\max\{\beta-Y(w^\top X + b),0\} \right] < 0
    \end{array} \right\}\\
    &= \inf~\left\{t : \begin{array}{l}
         w \in \R^d,~b \in \R,~\beta \in \R_{++},~t \in \R_{++} \\
          \EE_{\hat \PP}\left[\max\{\beta-Y(w^\top X + b),0\} \right]< \beta t
    \end{array} \right\}\\
    &= \inf~\left\{t: \begin{array}{l}
         w \in \R^d,~b \in \R,~\beta \in \R_{++},~t \in \R_{++}\\
          \EE_{\hat \PP}\left[\max\left\{1-Y\left(\left(\frac{w}{\beta}\right)^\top X + \left(\frac{b}{\beta}\right)\right),0\right\} \right]< t
    \end{array} \right\}\\
    &= \inf~\left\{t: \begin{array}{l}
         w' \in \R^d,~b' \in \R,~t \in \R_{++} \\
          \EE_{\hat \PP}\left[\max\left\{1-Y\left(w'^\top X + b'\right),0\right\} \right]< t
    \end{array} \right\}\\
    &= \left\{
    \begin{array}{cl}
         \min &\EE_{\hat \PP}\left[\max\left\{1-Y\left(w'^\top X + b'\right),0\right\} \right] \\
         \st & w' \in \R^d,~b' \in \R,
    \end{array}
    \right.
\end{align*}
where the penultimate equality holds by setting $w'=w/\beta$ and $\beta'=b/\beta$. Thus, the optimal value of these two problem coincides. Furthermore, by noticing that $w' = w/\beta$ and $b' = b/\beta$, the corresponding optimal hyperplanes $w^\top X + b=0$ and $w'^\top X + b'=0$ are also the same. This completes the proof.
\end{proof}

\section{Training $\eps$-DRFC Model with General Metrics}\label{subsec:mpm_general}

The general ground metric \eqref{eq:gencost} defined in Section \ref{sec:notrust} can also be applied to the $\eps$-DRFC model~\eqref{eq:dro-prob2} derived in Section \ref{sec:prob}. %Thus, we extend model \eqref{eq:dro-prob2} to a more general case by using this ground metric~\eqref{eq:gencost}, and present its reformulation below.
We consider in this section the modified problem of~\eqref{eq:dro-prob2} that utilizes the ambiguity set \eqref{eq:modified_ambi}:
\be \label{eq:dro-prob2-modified}
    \begin{array}{cl}
        \min & \Sup{\QQ \in \mc B_\gamma(\Pnom)}~\QQ( Y(w^\top X + b) < \eps) \\
        \st & w \in \R^d,~b \in \R, \\
        & \Sup{\QQ \in \mc B_\gamma(\Pnom)}~\mathds U_\eps(w, b, \QQ) \le \eta.
    \end{array}
\ee
We now present the main result of this section, which provides the reformulation for~\eqref{eq:dro-prob2-modified}.
\begin{theorem}($\eps$-DRFC reformulation) \label{thm:exactrefor}
    Suppose that the ground metric is prescribed using~\eqref{eq:gencost}. For any $\gamma \in (0, 1)$, problem \eqref{eq:dro-prob2-modified} is equivalent to the mixed binary conic program
\begin{equation} \label{eq:exact-refor-general}
\begin{array}{cll}
\inf& \ds \frac{1}{N} \sum_{i \in [N]} \nu_i + \sum_{(\bar a, \bar y) \in \mc A \times \mc Y} \hat p_{\bar a \bar y} \mu_{\bar a \bar y}-\theta(1-\gamma) & \\
 \st& \nu \in \R^N ,\; \theta \in \R_+,\; \mu \in \R^{2\times 2},\; \tau \in \{0, 1\}^N,  \\
  &\nu^a \in \R^N ,\; \theta^a \in \R_+,\; \mu^a \in \R^{2\times 2},\lambda_a^{a} \in \{0, 1\}^N, \; \lambda_{a'}^{a} \in \{0, 1\}^N \qquad \forall (a, a') \in \{(0, 1), (1, 0)\}, \\
&    \hspace{-2mm}\left.\begin{array}{l}
   \textup{If} \ \kappa_{\mc A} |a - \wh a_i| + \kappa_{\mc Y}| y - \wh y_i| \leq \rho:\\
    \quad  \tau_i^{a} \in \{0,1\},\\
    \quad \tau_i^a \leq \mu_{ay} - \theta \mathbbm{1}_{(\hat a_i, \hat y_i)} (a,  y) + \nu_i,\\
    \quad -\wh y_i (w^\top \wh x_i + b)  +(\rho - \kappa_{\mc A}|a - \wh a_i| - \kappa_{\mc Y}|y - \wh y_i|)\| w\|_*  \le M \tau_i^a - \eps
  \end{array}
  \right\} \forall i \in [N] \quad \forall(a, y) \in \mc A \times \mc Y,\\
 &\hspace{-2mm}\left.
\begin{array}{l}
   \textup{If} \ \kappa_{\mc A} |a - \wh a_i| + \kappa_{\mc Y}| 1 - \wh y_i| \leq \rho:\\
    \quad \wh p_{a1}^{-1} \lambda_{ai}^{a} \leq \mu_{a,1}^a - \theta^a \mathbbm{1}_{(\hat a_i, \hat y_i)} (a, 1) + \nu_i^a,  \\
    \quad w^\top \wh x_i + (\rho- \kappa_{\mc A} |a - \wh a_i| - \kappa_{\mc Y}| 1 - \wh y_i|) \|w\|_* + b + \eps \leq M \lambda_{ai}^{a}\\
    \textup{If} \ \kappa_{\mc A} |a' - \wh a_i| + \kappa_{\mc Y}| 1 - \wh y_i| \leq \rho:\\
    \quad \wh p_{a'1}^{-1}( \lambda_{a'i}^{a}-1) \leq \mu_{a'1}^a - \theta^a \mathbbm{1}_{(\hat a_i, \hat y_i)} (a', 1) + \nu_i^a,  \\
    \quad -w^\top \wh x_i + (\rho- \kappa_{\mc A} |a' - \wh a_i| - \kappa_{\mc Y}| 1 - \wh y_i|) \|w\|_* - b  \leq M \lambda_{a'i}^{a}\\
    \textup{If} \ \kappa_{\mc A} |a - \wh a_i| + \kappa_{\mc Y}| -1 - \wh y_i| \leq \rho:\\
    \quad 0 \leq \mu_{a,-1}^a - \theta^a \mathbbm{1}_{(\hat a_i, \hat y_i)} (a, -1) + \nu_i^a,\\
    \textup{If} \ \kappa_{\mc A} |a' - \wh a_i| + \kappa_{\mc Y}| -1 - \wh y_i| \leq \rho:\\
    \quad 0 \leq \mu_{a',-1}^a - \theta^a \mathbbm{1}_{(\hat a_i, \hat y_i)} (a', -1) + \nu_i^a,\\
\end{array}
\right\} \quad \forall (a, a') \in \{(0, 1), (1, 0)\} \quad \forall i \in [N], \\
& \ds\frac{1}{N} \sum_{i \in [N]} \nu_i^a + \sum_{(\bar a, \bar y) \in \mc A \times \mc Y} \hat p_{ay} \mu_{\bar a \bar y}^a-\theta^a(1-\gamma) \leq \eta
\quad \forall a \in \mc A,\\
\end{array}
\end{equation}
where $M$ is the big-M constant.
\end{theorem}
\noindent The reformulation~\eqref{eq:exact-refor-general} involves $8N$ binary variables. However, because the constraints of problem~\eqref{eq:exact-refor-general} are contingent, the empirical number of binary variables is smaller than $8N$. Problem~\eqref{eq:exact-refor-general} is a linear mixed binary optimization problem  if $\| \cdot \|$ is either a 1-norm or an $\infty$-norm on $\R^d$. If $\| \cdot \|$ is the Euclidean norm, problem~\eqref{eq:refor-probtrust} becomes a mixed binary second-order cone optimization problem. Both types of problems can be solved using off-the-shelf solvers such as MOSEK~\cite{mosek}.
We then present the proof of Theorem~\ref{thm:exactrefor}.
\begin{proof}[Proof of Theorem~\ref{thm:exactrefor}]
Notice that the objective function can be written in the form of
\[ 
        \Sup{\QQ \in \mc B_\gamma(\Pnom)}    \EE_{\QQ} [\phi(X,A,Y)],
\]
where $\phi(X,A,Y)=\mathbbm{I} \left(Y(w^\top X + b) < \eps \right)$ is an indicator function. By Lemma \ref{lem:strong_dual}, we have
\begin{align}\label{eq:obj_dual}
&\nonumber\Sup{\QQ \in \mc B_\gamma(\Pnom)}    \EE_{\QQ} [\phi(X,A,Y)]\\
\nonumber=&\left\{
\begin{array}{cll}
 \inf& \ds \frac{1}{N} \sum_{i \in [N]} \nu_i + \sum_{(\bar a, \bar y) \in \mc A \times \mc Y} \hat p_{\bar a \bar y} \mu_{\bar a \bar y}-\theta(1-\gamma) & \\
 \st& \nu \in \R^N,~\theta \in \R_+,~\mu \in \R^{2\times 2},\\
 &\Sup{x: \| x - \wh x_i\| \le \rho - \kappa_{\mc A}|a - \wh a_i| - \kappa_{\mc Y}|y - \wh y_i|} \mathbbm{I} \left(Y(w^\top X + b) < \eps \right) \leq \mu_{ay} - \theta \mathbbm{1}_{(\hat a_i, \hat y_i)} (a,  y) + \nu_i \\
 &\hspace{10cm}\forall i \in [N] \quad \forall (a, y) \in \mc A \times \mc Y.
\end{array}
 \right.
\end{align}
Based on Lemma \ref{lemma:indicator}, the constraint in the above infimum problem is equivalent to
\[
\left.\begin{array}{l}
   \textup{If} \ \kappa_{\mc A} |a - \wh a_i| + \kappa_{\mc Y}| y - \wh y_i| \leq \rho:\\
    \quad \tau_i^{a} \in \{0,1\},\\
    \quad \tau_i^a \leq \mu_{ay} - \theta \mathbbm{1}_{(\hat a_i, \hat y_i)} (a,  y) + \nu_i,\\
    \quad -\wh y_i (w^\top \wh x_i + b)  +(\rho - \kappa_{\mc A}|a - \wh a_i| - \kappa_{\mc Y}|y - \wh y_i|)\| w\|_*  \le M \tau_i^a - \eps
  \end{array}
  \right\} \forall i \in [N] \quad \forall (a, y) \in \mc A \times \mc Y.
\]

Next, we show the derivation for constraints. Recall that the worst-case unfairness measure can be written as
\begin{align*}
\Sup{\QQ \in \mc B_\gamma(\Pnom)}~\mathds U_\eps(w, b, \QQ) 
&= \max \left\{ \begin{array}{l}
\Sup{\QQ \in \mc B_\gamma(\Pnom)}~\QQ_{01}( w^\top X + b > -\eps) - \QQ_{11}(w^\top X + b \ge 0), \\
\Sup{\QQ \in \mc B_\gamma(\Pnom)}~\QQ_{11}( w^\top X + b > -\eps) - \QQ_{01}(w^\top X + b \ge 0) 
\end{array}
\right\}.
\end{align*}

Consider a fixed pair of $(a, a') \in \{ (0, 1), (1, 0)\}$. Employing the result of Lemma \ref{lem:strong_dual} yields
% To employ the result of Lemma \ref{lem:strong_dual}, we write the constraints in the form of \eqref{eq:label_primal}
\begin{align*}
&\Sup{\QQ \in \mc B_\gamma(\Pnom)}~\QQ( w^\top X + b > -\eps | A = a, Y = 1) - \QQ(w^\top X + b \ge 0 | A = a', Y = 1) \\
    =& \Sup{\QQ \in \mc B_\gamma(\Pnom)}~ \EE_{\QQ}[\wh p_{a1}^{-1} \mathbbm{I}(w^\top X + b > -\eps) \mathbbm{1}_{(a,1)}(A, Y) - \wh p_{a'1}^{-1} \mbb I(w^\top X + b \ge 0) \mathbbm{1}_{(a',1)}(A, Y)]\\
    =&\left\{
\begin{array}{cl}
 \inf& \ds\frac{1}{N} \sum_{i \in [N]} \nu_i^a + \sum_{(\bar a, \bar y) \in \mc A \times \mc Y} \hat p_{ay} \mu_{\bar a \bar y}^a-\theta^a(1-\gamma)  \\
 \st& \nu^a \in \R^N ,\; \theta^a \in \R_+,\; \mu^a \in \R^{2\times 2},\\
 &\Sup{\forall x_i \in \mc X: \| x_i - \wh x_i\| \le \rho - \kappa_{\mc A}| \bar a_i - \wh a_i| - \kappa_{\mc Y}| \bar y_i - \wh y_i|}\phi_a(x_i, \bar a_i, \bar y_i) \leq \mu_{\bar a_i \bar y_i}^a - \theta^a \mathbbm{1}_{(\hat a_i, \hat y_i)} (\bar a_i, \bar y_i) + \nu_i^a \\
 &\hspace{10cm} \forall i \in [N]  \quad \forall (\bar a_i, \bar y_i) \in \mc A \times \mc Y,
\end{array}
 \right.
\end{align*}
where the second equation relies on the result of Lemma \ref{lem:strong_dual} by defining 
\[
\phi_a(X,A,Y)=\wh p_{a1}^{-1} \mathbbm{I}(w^\top X + b > -\eps) \mathbbm{1}_{(a,1)}(A, Y) - \wh p_{a'1}^{-1} \mbb I(w^\top X + b \ge 0) \mathbbm{1}_{(a',1)}(A, Y).\] 
Fix any $i \in [N]$, we now iterate over $(\bar a_i, \bar y_i)$.
\begin{enumerate}
    \item Case 1: $(\bar a_i, \bar y_i) = (a, 1)$. There is an active constraint if $\kappa_{\mc A} |a - \wh a_i| + \kappa_{\mc Y}| 1 - \wh y_i| \leq \rho$, and the constraint is equivalent to
\begin{align*}
    &\Sup{\forall x_i \in \mc X : \| x_i - \wh x_i\| \leq \rho- \kappa_{\mc A} |a - \wh a_i| - \kappa_{\mc Y}| 1 - \wh y_i|}\wh p_{a1}^{-1} \mathbbm{I}(w^\top x_i + b > -\eps) \mathbbm{1}_{(a,1)}(a, 1) - \wh p_{11}^{-1} \mbb I(w^\top x_i + b \ge 0) \mathbbm{1}_{(a',1)}(a, 1)\\
    &\hspace{11cm} \leq \mu_{a1}^a - \theta^a \mathbbm{1}_{(\hat a_i, \hat y_i)} (a, 1) + \nu_i^a    \\
    \Longleftrightarrow & \Sup{\forall x_i \in \mc X: \| x_i - \wh x_i\| \leq \rho- \kappa_{\mc A} |a_ - \wh a_i| - \kappa_{\mc Y}| 1 - \wh y_i| }  \wh p_{a1}^{-1} \mathbbm{I}(w^\top x_i + b > -\eps)  \leq \mu_{a1}^a - \theta^a \mathbbm{1}_{(\hat a_i, \hat y_i)} (a, 1) + \nu_i^a \\
    \Longleftrightarrow& \left\{\begin{array}{l}
         \lambda_{ai}^{a} \in \{0,1\},\\
         \wh p_{a1}^{-1} \lambda_{ai}^{a} \leq \mu_{a1}^a - \theta^a \mathbbm{1}_{(\hat a_i, \hat y_i)} (a, 1) + \nu_i^a,  \\
         w^\top \wh x_i + (\rho- \kappa_{\mc A} |a - \wh a_i| - \kappa_{\mc Y}| 1 - \wh y_i|) \|w\|_* + b + \eps \leq M \lambda_{ai}^{a},
    \end{array}
    \right.
\end{align*}
where the last equation follows from the result of Lemma \ref{lemma:indicator}. 
    \item Case 2: $(\bar a_i,\bar  y_i)= (a', 1)$. There is an active constraint if $\kappa_{\mc A} |a' - \wh a_i| + \kappa_{\mc Y}| 1 - \wh y_i| \leq \rho$, and the constraint is equivalent to
\begin{align*}
    &\Sup{ \forall x_i \in \mc X: \| x_i - \wh x_i\| \leq \rho- \kappa_{\mc A} |a' - \wh a_i| - \kappa_{\mc Y}| 1 - \wh y_i| }\wh p_{a1}^{-1} \mathbbm{I}(w^\top x_i + b > -\eps) \mathbbm{1}_{(a,1)}(a', 1) - \wh p_{a'1}^{-1} \mbb I(w^\top x_i + b \ge 0) \mathbbm{1}_{(a',1)}(a', 1)  \\
    &\hspace{11cm} \leq \mu_{a'1}^a - \theta^a \mathbbm{1}_{(\hat a_i, \hat y_i)} (a', 1) + \nu_i^a \\
    \Longleftrightarrow& \Sup{\forall x_i \in \mc X:\| x_i - \wh x_i\| \leq \rho- \kappa_{\mc A} |a' - \wh a_i| - \kappa_{\mc Y}| y_i - \wh y_i| } - \wh p_{a'1}^{-1} \mbb I(w^\top x_i + b \ge 0)  \leq \mu_{a'1}^a - \theta^a \mathbbm{1}_{(\hat a_i, \hat y_i)} (a', 1) + \nu_i^a\\
    \Longleftrightarrow& - \wh p_{a'1}^{-1} \Inf{\| x_i - \wh x_i\| \leq \rho- \kappa_{\mc A} |a' - \wh a_i| - \kappa_{\mc Y}| y_i - \wh y_i| }  \mbb I(w^\top x_i + b \ge 0) \leq \mu_{a'1}^a - \theta^a \mathbbm{1}_{(\hat a_i, \hat y_i)} (a', 1) + \nu_i^a\\
    \Longleftrightarrow& - \wh p_{a'1}^{-1}\left(1- \Sup{\| x_i - \wh x_i\| \leq \rho- \kappa_{\mc A} |a' - \wh a_i| - \kappa_{\mc Y}| y_i - \wh y_i| }  \mbb I(w^\top x_i + b < 0) \right) \leq \mu_{a'1}^a - \theta^a \mathbbm{1}_{(\hat a', \hat y_i)} (a', 1) + \nu_i^a\\
    % \Longleftrightarrow&  \wh p_{a'1}^{-1}\left( \Sup{\| x_i - \wh x_i\| \leq \rho- \kappa_{\mc A} |a' - \wh a_i| - \kappa_{\mc Y}| y_i - \wh y_i| }  \mbb I(w^\top x_i + b < 0) - 1 \right) \leq \mu_{a'1}^a - \theta^a \mathbbm{1}_{(\hat a', \hat y_i)} (a', 1) + \nu_i^a\\
    \Longleftrightarrow& \left\{\begin{array}{l}
    \lambda_{a'i}^{a} \in \{0,1\},\\
     \wh p_{a'1}^{-1}( \lambda_{a'i}^{a}-1) \leq \mu_{a'1}^a - \theta^a \mathbbm{1}_{(\hat a_i, \hat y_i)} (a', 1) + \nu_i^a,  \\
     -w^\top \wh x_i + (\rho- \kappa_{\mc A} |a' - \wh a_i| - \kappa_{\mc Y}| 1 - \wh y_i|) \|w\|_* - b  \leq M \lambda_{a'i}^{a}.
    \end{array}
    \right.
\end{align*}

    \item Case 3: $(\bar a_i, \bar y_i)= (a, -1)$. There is an active constraint if $\kappa_{\mc A} |a - \wh a_i| + \kappa_{\mc Y}| -1 - \wh y_i| \leq \rho$, the constraint is equivalent to
\begin{align*}
        &\Sup{\forall x_i \in \mc X : \| x_i - \wh x_i\| \leq \rho- \kappa_{\mc A} |a - \wh a_i| - \kappa_{\mc Y}| -1 - \wh y_i|}\wh p_{a1}^{-1} \mathbbm{I}(w^\top x_i + b > -\eps) \mathbbm{1}_{(a,1)}(a, -1) - \wh p_{a'1}^{-1} \mbb I(w^\top x_i + b \ge 0) \mathbbm{1}_{(a',1)}(a, -1) \\
    &\hspace{11cm} \leq \mu_{a,-1}^a - \theta^a \mathbbm{1}_{(\hat a_i, \hat y_i)} (a, -1) + \nu_i^a    \\
    \Longleftrightarrow& 0 \leq \mu_{a,-1}^a - \theta^a \mathbbm{1}_{(\hat a_i, \hat y_i)} (a, -1) + \nu_i^a.\\
\end{align*}

    \item Case 4: $(\bar a_i, \bar y_i)= (a', -1)$. There is an active constraint if $\kappa_{\mc A} |a' - \wh a_i| + \kappa_{\mc Y}| -1 - \wh y_i| \leq \rho$, the constraint is equivalent to
\begin{align*}
        &\Sup{\forall x_i \in \mc X: \| x_i - \wh x_i\| \leq \rho- \kappa_{\mc A} |a' - \wh a_i| - \kappa_{\mc Y}| -1 - \wh y_i|}\wh p_{a1}^{-1} \mathbbm{I}(w^\top x_i + b > -\eps) \mathbbm{1}_{(a,1)}(a', -1) - \wh p_{a'1}^{-1} \mbb I(w^\top x_i + b \ge 0) \mathbbm{1}_{(a',1)}(a', -1)  \\
    &\hspace{11cm}   \leq \mu_{a',-1}^a - \theta \mathbbm{1}_{(\hat a_i, \hat y_i)} (a', -1) + \nu_i \\
    \Longleftrightarrow& 0 \leq \mu_{a',-1}^a - \theta^a \mathbbm{1}_{(\hat a_i, \hat y_i)} (a', -1) + \nu_i^a.\\
\end{align*}   

\end{enumerate}

Notice that at least one of the above four conditions will be satisfied, because when $\bar a_i=\wh a_i$ and $\bar y_i=\wh y_i$, we have 
\[  \kappa_{\mc A} |a_i - \wh a_i| + \kappa_{\mc Y}| y_i - \wh y_i|=0\leq \rho
\] for any $\rho \geq 0$. Combining all four cases leads to the second set of constraints.

The last constraint in the reformulation is obtained by setting the optimal value of the dual problem to be less than $\eta$ for each value of $a \in \mc A$. This completes the proof.
\end{proof}

\end{document}